%% file: TopicModeling-arxiv.tex
\documentclass{article}

\usepackage{graphicx} 
\usepackage{subfigure} 
\usepackage{tabularx,ragged2e,booktabs,array}
\def\QED{\mbox{\rule[0pt]{1.5ex}{1.5ex}}}
\usepackage{natbib}

\usepackage{algorithm}
\usepackage{algorithmic}

\usepackage{hyperref}

\usepackage{bm}

\usepackage[accepted]{icml2013}
\usepackage{amsmath,amsfonts,amssymb,amsthm,mathrsfs,bm}
\newtheorem{theorem}{Theorem}
\newtheorem{lemma}{Lemma}
\newtheorem{proposition}{Proposition}
\newtheorem{corollary}{Corollary}
\DeclareMathOperator{\E}{\mathbb{E}}
\DeclareMathOperator{\diag}{diag}
\DeclareMathOperator*{\argmin}{arg\,min}
\DeclareMathOperator*{\argmax}{arg\,max}
\allowdisplaybreaks[3]

\icmltitlerunning{Topic Discovery through Data Dependent and Random Projections}

\def\rR{{\mathbb R}}

\def\pP{{\mathbb P}}
\DeclareMathOperator{\supp}{supp}

\newtheorem{defn}{Definition}

\begin{document} 

\twocolumn[
\icmltitle{Topic Discovery through Data Dependent and Random Projections}

\icmlauthor{Weicong Ding}{dingwc@bu.edu}
\icmlauthor{Mohammad H. Rohban}{mhrohban@bu.edu}
\icmlauthor{Prakash Ishwar}{pi@bu.edu}
\icmlauthor{Venkatesh Saligrama}{srv@bu.edu}
\icmladdress{Electrical and Computer Engineering Department, Boston University}

\icmlkeywords{topic modeling, separability, extreme points}

\vskip 0.3in
]

\begin{abstract} 
We present algorithms for topic modeling based on the geometry of cross-document word-frequency patterns. This perspective gains significance under the so called separability condition. This is a condition on existence of novel-words that are unique to each topic. We present a suite of highly efficient algorithms based on data-dependent and random projections of word-frequency patterns to identify novel words and associated topics. We will also discuss the statistical guarantees of the data-dependent projections method based on two mild assumptions on the prior density of topic document matrix. Our key insight here is that the maximum and minimum values of cross-document frequency patterns projected along any direction are associated with novel words. While our sample complexity bounds for topic recovery are similar to the state-of-art, the computational complexity of our random projection scheme scales linearly with the number of documents and the number of words per document. We present several experiments on synthetic and real-world datasets to demonstrate qualitative and quantitative merits of our scheme.
\end{abstract} 
\section{Introduction}


We consider a corpus of $M$ documents composed of words chosen from
a vocabulary of $W$ distinct words indexed by $w = 1,\ldots,W$. We adopt the
classic ``bags of words'' modeling paradigm widely-used in
probabilistic topic modeling \cite{Blei2012Review:ref}.
Each document is modeled as being generated by $N$ independent and
identically distributed (iid) drawings of words from an unknown
$W\times 1$ document word-distribution vector. Each document
word-distribution vector is itself modeled as an unknown {\it
probabilistic mixture} of $K < \min(M,W)$ unknown $W \times 1$
latent topic word-distribution vectors that are {\it shared} among the
$M$ documents in the corpus. Documents are generated independently.
For future reference, we adopt the following notation.
We denote by ${\bm \beta}$ the unknown $W \times K$ topic-matrix
whose columns are the $K$ latent topic word-distribution vectors. ${\bm \theta}$
denotes the $K \times M$ weight-matrix whose $M$ columns are the
mixing weights over $K$ topics for the $M$ documents. These columns are assumed to be iid samples from a prior distribution. Each column
of the $W \times M$ matrix ${\mathbf A} = {\bm \beta}{\bm \theta}$ corresponds to a document 
word-distribution vector. ${\mathbf X}$ denotes the observed $W \times M$
word-by-document matrix realization. The $M$ columns of ${\mathbf X}$ are the {\it empirical}
word-frequency vectors of the $M$ documents. Our goal is to
estimate the latent topic word-distribution vectors (${\bm \beta}$) from the empirical
word-frequency vectors of all documents (${\mathbf X}$).

A fundamental challenge here is that word-by-document distributions (${\mathbf A}$) are unknown and only a realization is  available through sampled word frequencies in each document. Another challenge is that even when these distributions are exactly known, the decomposition into the product of topic-matrix, ${\bm \beta}$, and topic-document distributions, ${\bm \theta}$, which is known as {\it Nonnegative Matrix Factorization (NMF)}, has been shown to be an $\mathcal{NP}$-hard problem in general.
%
%
In this paper, we develop computationally efficient algorithms with provable guarantees for estimating ${\bm \beta}$ for topic matrices satisfying the \emph{separability condition}~\cite{Donhunique:ref,ARORA:ref}. 
\begin{defn} (Separability) 
\label{def:separable}
A topic matrix ${\bm \beta}\in\rR^{W\times K}$ is separable if for each topic $k$, there is some word $i$ such that ${\beta}_{i,k}>0$ and ${\beta}_{i,l}=0$, $\forall l\neq k$.
\end{defn}
The condition suggests the existence of novel words that are unique to each topic. Our algorithm has three main steps. In the first step, we identify novel words by means of data dependent or random projections. A key insight here is that when each word is associated with a vector consisting of its occurrences across all documents, the novel words correspond to extreme points of the convex hull of these vectors. A highlight of our approach is the identification of novel words based on data-dependent and random projections. Our idea is that whenever a convex object is projected along a random direction, the maximum and minimum values in the projected direction correspond to extreme points of the convex object. While our method identifies novel words with negligible false and miss detections, evidently multiple novel words associated with the same topic can be an issue. To account for this issue, we apply a distance based clustering algorithm to cluster novel words belonging to the same topic. 
Our final step involves linear regression to estimate topic word frequencies using novel words. 

We show that our scheme has a similar sample complexity to that of state-of-art such as \cite{Arora2:ref}. On the other hand, the computational complexity of our scheme can scale as small as $\mathcal{O}(\sqrt{M}W + MN)$ for a corpora containing $M$ documents, with an average of $N$ words per document from a vocabulary containing $W$ words. We then present a set of experiments on synthetic and real-world datasets. The results demonstrates  qualitative and quantitative superiority of our scheme in comparison to other state-of-art schemes. 


%

\section{Related Work}
%
%
%
%
The literature on topic modeling and discovery is extensive. One direction of work is based on solving a nonnegative matrix factorization (NMF) problem. To address the scenario where only the realization $\mathbf{X}$ is known and not ${\mathbf A}$, several papers \cite{nmfLS:ref,Donhunique:ref,NMFbook:ref,recht2012factoring} attempt to minimize a regularized cost function. 
Nevertheless, this joint optimization is non-convex and suboptimal strategies 
have been used in this context. Unfortunately, when $N \ll W$ which is often the case, many words do not appear in $\mathbf{X}$ and such methods often fail in these cases. 

Latent Dirichlet Allocation (LDA) ~\cite{LDA:ref,Blei2012Review:ref} is a statistical approach to topic modeling. In this approach, the columns of ${\bm \theta}$ are modeled as iid random drawings from some prior distributions such as Dirichlet. The goal is to compute MAP (maximum aposteriori probability) estimates for the topic matrix. This setup is inherently non-convex and MAP estimates are computed using variational Bayes approximations of the posterior distribution, Gibbs sampling or expectation propagation.

A number of methods with provable guarantees have also been proposed. \cite{Anan12} describe a novel method of moments approach.  While their algorithm does not impose structural assumption on topic matrix ${\bm \beta}$, they require Dirichlet priors for ${\bm \theta}$ matrix. One issue is that such priors do not permit certain classes of correlated topics~\cite{Blei07:ref, Li07:ref}. Also their algorithm is not agnostic since it uses parameters of the Dirichlet prior. Furthermore, the algorithm suggested involves finding empirical moments and singular decompositions which can be cumbersome for large matrices. 

Our work is closely related to recent work of \cite{ARORA:ref} and \cite{Arora2:ref} with some important differences. In their work, they describe methods with provable guarantees when the topic matrix satisfies the separability condition. 
Their algorithm discovers novel words from empirical {\bf word co-occurrence} patterns and then in the second step the topic matrix is estimated. Their key insight is that when each word, $j$, is associated with a $W$ dimensional vector\footnote{$k$th component is probability of occurrence of word $j$ and word $k$ in the same document in the entire corpus} the novel words correspond to extreme points of the convex hull of these vectors. \cite{Arora2:ref} presents combinatorial algorithms to recover novel words with computational complexity scaling as $\mathcal{O}(MN^2 + W^2 + WK/\epsilon^2)$, where $\epsilon$ is the element wise tolerable error of the topic matrix ${\bm \beta}$. An important computational remark is that $\epsilon$ often scales with $W$, i.e. probability values in ${\bm \beta}$ get small when $W$ is increased, hence one needs smaller $\epsilon$ to safely estimate ${\bm \beta}$ when $W$ is too large. The other issue with their method is that empirical estimates of joint probabilities in the word-word co-occurrence matrix can be unreliable, especially when $M$ is not large enough. Finally, their novel word detection algorithm requires linear independence of the extreme points of the convex hull. This can be a serious problem in some datasets where word co-occurrences lie on a low dimensional manifold.  

{\bf Major Differences:} Our work also assumes separability and existence of novel words. We associate each word with a $M$-dimensional vector consisting of the word's frequency of occurrence in the $M$-documents rather than word co-occurrences as in \cite{ARORA:ref, Arora2:ref}. We also show that extreme points of the convex hull of these cross-document frequency patterns are associated with novel words. While these differences appear technical, it has important consequences. In several experiments our approach appears to significantly outperform \cite{Arora2:ref} and mirror performance of more conventional methods such as LDA \cite{Griffiths:ref}. Furthermore, our approach can deal with degenerate cases found in some image datasets where the data vectors can lie on a lower dimensional manifold than the number of topics. At a conceptual level our approach appears to hinge on distinct cross-document support patterns of novel words belonging to different topics. This is typically robust to sampling fluctuations when support patterns are distinct in comparison to word co-occurrences statistics of the corpora. Our approach also differs algorithmically. We develop novel algorithms based on data-dependent and random projections to find extreme points efficiently with computational complexity scaling as $\mathcal{O}(MN + \sqrt{M}W)$ for the random scheme.
{\bf Organization:} We illustrate the motivating Topic Geometry in Section \ref{sec:Geometry}. We then present our three-step algorithm in Section \ref{sec:algorithm} with intuitions and computational complexity. Statistical correctness of each step of proposed approach are summarized in Section \ref{sec:theory}. We address practical issues in Section \ref{sec:experiment}.

\section{Topic Geometry}
\label{sec:Geometry}
%
Recall that ${\mathbf X}$ and ${\mathbf A}$ respectively denote the $W \times M$ empirical and actual document  word distribution matrices, and ${\mathbf A} = {\bm \beta}{\bm \theta}$, where ${\bm \beta}$ is the latent topic word distribution matrix and ${\bm \theta}$ is the underlying weight matrix.
Let $\widetilde{\mathbf A}$,  $\widetilde{\bm \theta}$ and $\widetilde{\mathbf X}$ denote the ${\mathbf A}$, ${\bm \theta}$ and ${\mathbf X}$ matrices after $\ell_1$ row normalization. 
We set \hbox{$\widetilde{\bm {\beta}}=\diag(\mathbf{A}\mathbf{1})^{-1}{\bm \beta} \diag(\bm{\theta}\mathbf{1})$}, so that \hbox{$\widetilde{\mathbf{A}}=\widetilde{\bm{\beta}}\widetilde{\bm{\theta}}$}.
Let ${\mathbf X}_i$ and ${\mathbf A}_i$ respectively denote the $i-{\text{th}}$ row of ${\mathbf X}$ and ${\mathbf A}$ representing the cross-document patterns of word $i$. We assume that $\bm{\beta}$ is \textit{separable} (Def. \ref{def:separable}). 
Let $\mathcal{C}_k$ be the set of novel words of topic $k$ and let $\mathcal{C}_0$ be the set of non-novel words. 

The geometric intuition underlying our approach is formulated in the following proposition :
\begin{proposition} 
\label{GeomPic}
Let $\bm{\beta}$ be separable. Then for all novel words $i \in \mathcal{C}_j$, $\widetilde{\mathbf A}_i = \widetilde{\bm \theta}_j$ and for all non-novel words $i\in\mathcal{C}_0$, $\widetilde{\mathbf{A}}_i$ is a convex combination of $\widetilde{\bm \theta}_j$'s, for $j=1,\ldots,K$.
\end{proposition}
%

\vglue -1ex
\noindent
{\bf Proof:}
Note that for all $i$, 
\begin{equation*}
\sum\limits_{k=1}^{K}\widetilde{\beta}_{ik}=1
\end{equation*}
and for all $i \in \mathcal{C}_j$, $\widetilde{\beta}_{ij} = 1$. Moreover, we have 
\begin{equation*}
\widetilde{\mathbf A}_i = \sum\limits_{k = 1}^{K} \widetilde{\beta}_{ik} \widetilde{\bm \theta}_k
\end{equation*}
Hence $\widetilde{\mathbf A}_i=\widetilde{\bm \theta}_j$ for $i\in\mathcal{C}_j$. In addition, $\widetilde{\mathbf A}_i = \sum\limits_{k=1}^{K}\widetilde{\bm \beta}_{ik} \widetilde{\bm \theta}_k$ for $i\in\mathcal{C}_0$. 
\QED

Fig.~\ref{fig:extreme} illustrates this geometry. Without loss of generality, we could assume that novel word vectors $\widetilde{\bm \theta}_i$ are not in the convex hull of the other rows of $\widetilde{\bm\theta}$. 
Hence, The problem of identifying novel words reduces to finding extreme points of all $\widetilde{\mathbf{A}}_i$'s.

\begin{figure}[htb]
\begin{minipage}[b]{1.0\linewidth}
 \centering
  \centerline{\includegraphics[width=8.0cm]{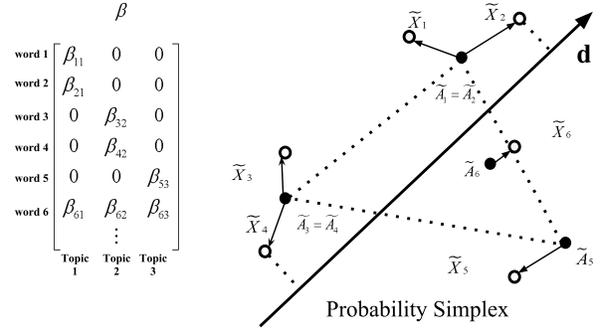}}
\end{minipage}
\caption{A separable topic matrix and the underlying geometric
  structure. Solid circles represent rows of $\widetilde{\mathbf A}$, empty
  circles represent rows of $\widetilde{\mathbf X}$. Projections of $\widetilde{\mathbf X}_i$'s along a direction ${\mathbf d}$ can be used to identify novel words. }
\label{fig:extreme}
\end{figure}
Furthermore, retrieving topic matrix $\bm{\beta}$ is straightforward given all $K$ distinct novel words :
\begin{proposition} 
\label{Ret}
If the matrix ${\mathbf A}$ and $K$ distinct novel words $\{ i_1, \ldots, i_K \}$ are given, then ${\bm \beta}$ can be calculated using $W$ linear regressions.
\end{proposition}
%
\vglue -2ex
\noindent
{\bf Proof:}
By Proposition \ref{GeomPic}, we have $\widetilde{\bm \theta} = ({\mathbf A}_{i_1}^{\top}, \ldots, {\mathbf A}_{i_K}^{\top})^{\top}$. Next $\widetilde{\mathbf A}_i = \widetilde{\bm \beta}_{i} \widetilde{\bm \theta}$. So $\widetilde{\bm \beta}_{i}$ can be computed by solving a linear system of equations. 
Specifically, if we let ${\bm \beta}^{\prime}=\diag(\mathbf{A}\mathbf{1}) \widetilde{\bm {\beta}} = {\bm \beta} \diag(\bm{\theta}\mathbf{1})^{-1}$, ${\bm \beta}$ can be obtained by column normalizing ${\bm \beta}^{\prime}$. 
\QED

%
Proposition \ref{GeomPic} and \ref{Ret} validate the approach to estimate ${\bm \beta}$ via identifying novel words given access to ${\mathbf A}$. 
However, only ${\mathbf X}$, a realization of ${\mathbf A}$, is available in the real problem which is not close to ${\mathbf A}$ in typical settings of interest ($N \ll W$). However, even when the number of samples per document ($N$) is limited, if we collect enough documents ($M \rightarrow \infty$), the proposed algorithm could still asymptotically estimate $\bm \beta$ with arbitrary precision, as we will discuss in the following sections.

\section{Proposed Algorithm}
\label{sec:algorithm}
%
The geometric intuition mentioned in Propositions \ref{GeomPic} and \ref{Ret} motivates the following three-step approach for topic discovery :
%

%
\noindent
{\bf (1) Novel Word Detection:} Given the empirical word-by-document matrix $\mathbf X$, extract the set of all novel words $\mathcal{I}$.  We present variants of projection-based algorithms in Sec. \ref{sec:Projection}.

\noindent
{\bf (2) Novel Word Clustering:} Given a set of novel words $\mathcal{I}$ with $|\mathcal{I}|\geq K$ , cluster them into $K$ groups corresponding to $K$ topics. Pick a representative for each group. We adopt a distance based clustering algorithm. (Sec. \ref{ClustSec}).

\noindent
{\bf (3)
Topic Estimation:} Estimate topic matrix as suggested in Proposition~\ref{Ret} by constrained linear regression. (Section \ref{EstSec}).
%
\subsection{Novel Word Detection}
\label{sec:Projection}
Fig. ~\ref{fig:extreme} illustrates the key insight to identify novel words as extreme points of some convex body. When we project every point of a convex body onto some direction $\mathbf{d}$, the maximum and minimum correspond to extreme points of the convex object. Our proposed approaches, data dependent and random projection, both exploit this fact. They only differ in the choice of projected directions.

\noindent
{\bf A. Data Dependent Projections (DDP) } \\
To simplify our analysis, we randomly split each document into two subsets, and obtain two statistically independent document collections $\mathbf{X}$ and $\mathbf{X}^{\prime}$, both distributed as $\mathbf{A}$, and then row normalize as  $\widetilde{\mathbf X}$ and $\widetilde{\mathbf X}^{\prime}$. For some threshold, $d$, to be specified later, and for each word $i$, we consider the set, $J_i$, of all other words that are sufficiently different from word $i$ in the following sense:
\begin{equation} \label{eq:ji}
J_i = \{ j \mid M (\widetilde{\mathbf X}_i - \widetilde{\mathbf X}_j)(\widetilde{\mathbf X}_i^{\prime} - \widetilde{\mathbf X}_j^{\prime})^{\top} \geq d/2 \}
\end{equation}
%
We then declare word $i$ as a novel word if all words $j \in J_i$ are uniformly uncorrelated to word $i$ with some margin, $\gamma/2$ to be specified later.  
\begin{equation} \label{eq:corr}
M \langle \widetilde{\mathbf X}_i,  \widetilde{\mathbf X}_i^{\prime} \rangle \geq  M \langle \widetilde{\mathbf X}_i,  \widetilde{\mathbf X}_j^{\prime} \rangle +\gamma/2,\,\, \forall j \in J_i
\end{equation} 

The correctness of DDP Algorithm is established by the following Proposition and will be further discussed in section \ref{TheoSec}. The proof is given in the Supplementary section.
\begin{proposition} \label{AWCons}
Suppose conditions $P1$ and $P2$ (will be defined in section \ref{TheoSec}) on prior distribution of ${\bm \theta}$ hold. Then, there exists two positive constants $d$ and $\gamma$ such that if $i$ is a novel word, for all $j \in J_i$, $M \langle \widetilde{\mathbf X}_i,  \widetilde{\mathbf X}_i^{\prime} \rangle - M \langle \widetilde{\mathbf X}_i,  \widetilde{\mathbf X}_j^{\prime} \rangle \geq \gamma/2$ with high probability (converging to one as $M \rightarrow \infty$). In addition, if $i$ is a non-novel word, there exists some $j \in J_i$ such that $M \langle \widetilde{\mathbf X}_i,  \widetilde{\mathbf X}_i^{\prime} \rangle - M \langle \widetilde{\mathbf X}_i,  \widetilde{\mathbf X}_j^{\prime} \rangle \leq \gamma/2$ with high probability.
\end{proposition}
\begin{algorithm}[h]
\caption{Novel Word Detection - DDP}
\label{AWAlg}
\begin{algorithmic}[1]
\STATE {\bf Input} $\widetilde{\mathbf X}, \widetilde{\mathbf X}^{\prime}, d, \gamma, K$
\STATE {\bf Output}: The indices of the novel words $\mathcal{I}$
\STATE ${\mathbf C} \leftarrow M ~ \widetilde{\mathbf X}^{\prime} \widetilde{\mathbf X}^{\top}$
\STATE $\mathcal{I} \leftarrow \emptyset $
\FORALL {$ 1 \leq i \leq W$}
	\STATE $J_i \leftarrow$ All indices $j \neq i : C_{i,i} - 2C_{i,j} + C_{j,j} \geq \frac{d}{2}$
	\IF {$\forall j \in J_i ~:~ C_{i,i} - C_{i, j} \geq \gamma/2 $}
		\STATE $\mathcal{I} \leftarrow \mathcal{I} \cup \{ i \}$
	\ENDIF
\ENDFOR
\end{algorithmic}
\end{algorithm}

The algorithm is elaborated in Algorithm \ref{AWAlg}. The running time of the algorithm is summarized in the following proposition. Detailed justification is provided in the Supplementary section.
\begin{proposition}
The running time of Algorithm \ref{AWAlg} is $\mathcal{O}(MN^2 + W^2)$.
\end{proposition}
\noindent
{\it Proof Sketch.}
Note that ${\mathbf X}$ is sparse since $N\ll W$. Hence by exploiting the sparsity ${\mathbf C} = M {\mathbf X} {\mathbf X}^{\prime \top}$ can be computed in $\mathcal{O}(MN^2 + W)$ time. 
For each word $i$, finding $J_i$ and calculating $C_{i,i} - C_{i,j} \geq \gamma/2$ cost  $\mathcal{O}(W^2)$ time in the worst case. 
\QED

\noindent
{\bf B. Random Projections (RP)}\\
DDP uses $W$ different directions to find all the extreme points. Here we use random directions instead. This significantly reduces the time complexity by decreasing the number of required projections. 

The Random Projection Algorithm (RP)  uses roughly $P=\mathcal{O}(K)$ random directions drawn uniformly iid over the unit sphere.  For each direction ${\mathbf d}$, we project all  $\widetilde{\mathbf X}_i$'s onto it and choose the maximum and minimum. 
\begin{algorithm}
\caption{Novel Word Detection - RP}
\label{RP}
\begin{algorithmic}[1]
\STATE {\bf Input} $\widetilde{\mathbf X}, P$
\STATE {\bf Output} : The indices of the novel words $\mathcal{I}$
\STATE $\mathcal{I} \leftarrow \emptyset $
\FORALL {$ 1 \leq j \leq P$}
	\STATE Generate ${\mathbf d} \sim$ Uniform(unit-sphere in $\rR^{M}$)
	\STATE $i_{max}=\argmax \widetilde{\mathbf X}_i{\mathbf d},i_{min}=\argmax \widetilde{\mathbf X}_i{\mathbf d}$
	\STATE $\mathcal{I} \leftarrow \mathcal{I} \cup \{ i_{max},i_{min} \}$
\ENDFOR
\end{algorithmic}
\end{algorithm}
%
Note that $\widetilde{\mathbf X}_i {\mathbf d}$ will converge to $\widetilde{\mathbf A}_i {\mathbf d}$ conditioned on ${\mathbf d}$ and ${\bm \theta}$ as $M$ increases. Moreover, only for the extreme points $i$, $\widetilde{\mathbf A}_i {\mathbf d}$ can be the maximum or minimum projection value. This provides intuition of consistency for RP. Since the directions are independent, we expect to find all the novel words using $P=\mathcal{O}(K)$ number of random projections. 


\noindent
{\bf C. Random Projections with Binning}

Another alternative to RP is a Binning algorithm which is computationally more efficient. Here the corpus is split into $\sqrt{M}$ equal sized bins. For each bin $j$ a random direction ${\mathbf d}^{(j)}$ is chosen and the word with the maximum projection along ${\mathbf d}^{(j)}$ is chosen as a winner. Then, we find the number of wins for each word $i$. We then divide these winning frequencies by $\sqrt{M}$ as an estimate for $p_i\triangleq \Pr(\forall j \neq i : \widetilde{\mathbf A}_i {\mathbf d} \geq \widetilde{\mathbf A}_j {\mathbf d})$. $p_i$ can be shown to be zero for all non-novel words. For non-degenerate prior over
${\bm \theta}$, these probabilities converge to strictly positive values for novel words. Hence, estimating $p_i$'s helps in identifying novel words. We then choose the indices of $\mathcal{O}(K)$ largest $p_i$ values as novel words. The Binning algorithm is outlined in Algorithm \ref{AWRAlg}.
\begin{algorithm}[h]
\caption{Novel Word Detection - Binning}
\label{AWRAlg}
\begin{algorithmic}[1]
\STATE {\bf Input} : $\widetilde{\mathbf X}$, $\widetilde{\mathbf X}^{\prime}$, $d$, K
\STATE {\bf Output} : The indices of the novel words $\mathcal{I}$
\STATE Split documents in ${\mathbf X}$ into $\sqrt{M}$ equal sized groups of documents ${\mathbf X}^{(1)}, \ldots, {\mathbf X}^{(\sqrt{M})}$ and normalize each one separately to obtain $\widetilde{\mathbf X}^{(1)}, \ldots, \widetilde{\mathbf X}^{(\sqrt{M})}$ as well.
\FORALL {$ 1 \leq j \leq \sqrt{M} $}
	\STATE ${\mathbf d}^{(j)} \leftarrow$ a sample from $U(\mathcal{S}^{\sqrt{M} - 1})$
	\STATE $l \leftarrow \argmax\limits_{1 \leq i \leq W} \widetilde{\mathbf X}^{(j)}_i {\mathbf d}^{(j)}$ 
	\STATE $\hat{p}_l^{(j)} \leftarrow \hat{p}_l^{(j)} + 1$
\ENDFOR
\FORALL {$ 1 \leq i \leq W $}
	\STATE $\hat{p}_i \leftarrow \frac{1}{\sqrt{M}} \sum_{j = 1}^{\sqrt{M}} \hat{p}_i^{(j)}$
\ENDFOR
\STATE $k \leftarrow 0$, $\mathcal{I} \leftarrow \emptyset$ and $i \leftarrow 1$
\REPEAT
\STATE $j \leftarrow$ the index of the $i^{\text{th}}$ largest value of $(\hat{p}_1, \ldots, \hat{p}_W)$
\IF {${\mathcal I} = \emptyset \text{ or } \forall l \in \mathcal{I} : M (\widetilde{\mathbf X}_j - \widetilde{\mathbf X}_l) (\widetilde{\mathbf X}_j^{\prime} - \widetilde{\mathbf X}_l^{\prime}) \geq d/2 $ }
\STATE $\mathcal{I} \leftarrow \mathcal{I} \cup \{j\}$
\STATE $k \leftarrow k + 1$
\ENDIF
\STATE $i \leftarrow i + 1$
\UNTIL{$k = K$ }
\end{algorithmic}
\end{algorithm}

In contrast with DDP, the RP algorithm is completely agnostic and parameter-free. This means that it requires no parameters like $d$ and $\gamma$ to find the novel words. Moreover, it significantly reduces the computational complexity : 
\begin{proposition}
The running times of the RP and Binning algorithms are $\mathcal{O}(MNK + WK)$ and $\mathcal{O}(MN + \sqrt{M}W)$, respectively.
\end{proposition}
\vspace*{-0.5cm}
\begin{proof}
We will sketch the proof and provide a more detailed justification in the Supplementary section. Note that the number of operations needed to find the projections is 
$\mathcal{O}(MN + W)$ in Binning and $\mathcal{O}(MNK + W)$ in RP. 
In addition, finding the the maximum takes $\mathcal{O}(WK)$ for RP and $\mathcal{O}(\sqrt{M}W)$ for Binning. In sum, it takes $\mathcal{O}(MNK + WK)$ for RP and $\mathcal{O}(MN + \sqrt{M}W)$ for Binning to find all the novel words.
\end{proof}

\subsection{Novel Word Clustering} \label{ClustSec}
Since there may be multiple novel words for a single topic, our DDP or RP algorithm can extract multiple novel words for each topic. This necessitates clustering to group the copies. We can show that our clustering scheme is consistent if we assume that ${\mathbf R} = \frac{1}{M} \E({\bm \theta}{\bm \theta}^{\top})$ is positive definite:
%
\begin{proposition} \label{ClustProp}
Let $C_{i,j} \triangleq M \widetilde{\mathbf X}_i \widetilde{\mathbf X}_j^{\prime \top}$, and $D_{i,j} \triangleq C_{i,i} - 2C_{i,j} + C_{j,j}$. If ${\mathbf R}$ is positive definite, then $D_{i,j}$ converges 
to zero in probability whenever $i$ and $j$ are novel words of the same topic as $M \rightarrow \infty$. Moreover, if $i$ and $j$ are novel words of different types, it converges in probability
to some strictly positive value greater than some constant $d$ .%
\end{proposition}
The proof is presented in the Supplementary section.
\begin{algorithm}
\caption{Novel Word Clustering}
\label{ACAlg}
\begin{algorithmic}[1]
\STATE {\bf Input} : $\mathcal{I}$, $\widetilde{\mathbf X}$, $\widetilde{\mathbf X}^{\prime}$, $d$, $K$
\STATE {\bf Output} : $\mathcal{J}$ which is a set of $K$ novel words of distinct topics
\STATE ${\mathbf C} \leftarrow M ~ \widetilde{\mathbf X}^{\prime} \widetilde{\mathbf X}^{\top}$
\STATE ${\mathbf B} \leftarrow \text{a } |\mathcal{I}| \times |\mathcal{I}| \text{ zero matrix}$
\FORALL {$ i, j \in \mathcal{I}, ~ i \neq j $}
	\IF {$C_{i,i} - 2C_{i, j} + C_{j,j} \leq d/2 $}
		\STATE $B_{i,j} \leftarrow 1$
	\ENDIF
\ENDFOR
\STATE $\mathcal{J} \leftarrow \emptyset$
\FORALL {$1 \leq j \leq K$}
	\STATE $c \leftarrow$ one of the indices of the $j^{\text{th}}$ connected component vertices in ${\mathbf B}$
	\STATE $\mathcal{J} \leftarrow \mathcal{J} \cup \{ c \}$
\ENDFOR
\end{algorithmic}
\end{algorithm}
As the Proposition \ref{ClustProp} suggests, we construct a binary graph with its vertices correspond to the novel words. An edge between word $i$ and $j$ is established if $D_{i,j} \leq d/2$. Then, the clustering reduces to finding $K$ connected components. The procedure is described in Algorithm \ref{ACAlg}. 

In Algorithm \ref{ACAlg}, we simply choose any word of a cluster as the representative for each topic. This is simply for theoretical analysis. However, we could set the representative to be the average of data points in each cluster, which is more noise resilient. 


\subsection{Topic Matrix Estimation} \label{EstSec}
Given $K$ novel words of different topics ($\mathcal{J}$), we could directly estimate (${\bm \beta}$) as in Proposition \ref{Ret}. This is described in Algorithm \ref{TEAlg}. We note that this part of the algorithm is similar to some other topic modeling approaches, which exploit separability. Consistency of this step is also validated in \cite{ARORA:ref}. In fact, one may use the convergence of extremum estimators \cite{Tak:ref} to show the consistency of this step. 

\begin{algorithm}
\caption{Topic Matrix Estimation}
\label{TEAlg}
\begin{algorithmic}[1]
\STATE {\bf Input}: $\mathcal{J}=\{j_1,\ldots,j_K\}$, ${\mathbf X}$, ${\mathbf X}^{\prime}$
\STATE {\bf Output}: $\widehat{{\bm \beta}}$, which is the estimation of ${\bm \beta}$ matrix
	\STATE ${\mathbf Y}=(\widetilde{\mathbf X}_{j_1}^{\top}, \ldots, \widetilde{\mathbf X}_{j_K}^{\top})^{\top}$,${\mathbf Y^{\prime}}=(\widetilde{\mathbf X}_{j_1}^{{\prime}\top}, \ldots, \widetilde{\mathbf X}_{j_K}^{{\prime}\top})^{\top}$
\FORALL {$1 \leq i \leq W$}
	\STATE $\widehat{\bm \beta}_i \leftarrow (\frac{1}{M} {\mathbf X}_i {\mathbf 1}) \argmin\limits_{b_j \geq 0, \sum_{j=1}^{K} b_j = 1} M (\widetilde{\mathbf X}_i - {\mathbf b} {\mathbf Y}) (\widetilde{\mathbf X}^{\prime}_i - {\mathbf b} {\mathbf Y}^{\prime})^{\top} $
\ENDFOR 
\STATE column normalize $\widehat{\bm \beta}$ 
\end{algorithmic}
\end{algorithm}

\section{Statistical Complexity Analysis} \label{TheoSec}
\label{sec:theory}
In this section, we describe the sample complexity bound for each step of our algorithm. 
Specifically, we provide guarantees for DDP algorithm under some mild assumptions on the distribution over ${\bm \theta}$. 
The analysis of the random projection algorithm is much more involved and requires elaborate arguments. We will omit it in this paper. 

We require following technical assumptions on the correlation matrix ${\mathbf R}$ and the mean vector ${\mathbf a}$ of ${\bm \theta}$ :

%
\noindent
$({P1})$ ${\mathbf R}$ is positive definite with its minimum eigenvalue being lower bounded by $\lambda_{\wedge} >0$. In addition, $\forall i, {a}_i \geq a_{\wedge}>0$. 

%
\noindent
$({P2})$ There exists a positive value $\zeta$ such that for $i\neq j$, $R_{i,i}/(a_i a_i) - R_{i,j}/(a_i a_j) \geq \zeta$.

The second condition captures the following intuition :  if two novel words are from different topics, they must appear in a substantial number of distinct documents. Note that for two novel words $i$ and $j$ of different topics, $M \widetilde{\mathbf A}_i ( \widetilde{\mathbf A}_i - \widetilde{\mathbf A}_j )^{\top} \xrightarrow{p} R_{i,i}/(a_i a_i) - R_{i,j}/(a_i a_j) $. Hence, this requirement means that $M ( \widetilde{\mathbf A}_i - \widetilde{\mathbf A}_j)$ should be fairly distant from the origin, which implies that the number of documents these two words co-occur in, with similar probabilities, should be small. This is a reasonable assumption, since otherwise we would rather group two related topics into one. In fact, we show in the Supplementary section (Section \ref{P1P2A}) that both conditions hold for the Dirichlet distribution, which is a traditional choice for the prior distribution in topic modeling. Moreover, we have tested the validity of these assumptions numerically for the logistic normal distribution (with non-degenerate covariance matrices), which is used in Correlated Topic Modeling (CTM) \cite{Blei07:ref}.

\subsection{Novel Word Detection Consistency}

In this section, we provide analysis only for the DDP Algorithm. The sample complexity analysis of the randomized projection algorithms is however more involved and is the subject of the ongoing research.
Suppose $P1$ and $P2$ hold. Denote $\beta_{\wedge}$ and $\lambda_{\wedge}$ to be positive lower bounds on non-zero elements of ${\bm \beta}$ and minimum eigenvalue of ${\mathbf R}$, respectively. We have: 
\begin{theorem} \label{Main1}
For parameter choices $d = {\lambda_{\wedge} \beta_{\wedge}^2}$ and $\gamma = \zeta {a_{\wedge} \beta_{\wedge}}$ the DDP algorithm is consistent as \hbox{$M \rightarrow \infty$.} Specifically, true novel and non-novel words are asymptotically declared as novel and non-novel, respectively. 
Furthermore, for 
\begin{equation*}
M \geq \frac{C_1 \left(\log W + \log \left( \frac{1}{\delta_1} \right)\right)}{\beta_{\wedge}^2 \eta^8 \min(\lambda_{\wedge}^2 \beta_{\wedge}^2, \zeta^2 a_{\wedge}^2) }
\end{equation*}
where $C_1$ is a constant, Algorithm \ref{AWAlg} finds all novel words
without any outlier with probability at least $1 - \delta_1$, where $\eta = \min\limits_{1 \leq i \leq W} {\bm \beta}_i {\mathbf a}$.  \\
\end{theorem}

\noindent
{\it Proof Sketch. }
The detailed justification is provided in the Supplementary section. The main idea of the proof is a sequence of statements :
\begin{itemize}
\item Given $P1$, for a novel word $i$, $J_i$ defined in the Algorithm \ref{AWAlg} is a subset of $J_i^{*}$ asymptotically with high probability, where $J_i^{*} = \{ j : \supp({\bm \beta}_j) \neq \supp({\bm \beta}_i) \}$. Moreover $J_i$ is a superset of $J^{*}_i$ with high probability for a non-novel word with $J_i^{*} = \{ j : |\supp({\bm \beta}_j)| = 1 \}$.
\item Given $P2$, for a novel word $i$, $C_{i,i} - C_{i,j}$ converges to a strictly positive value  greater than $\gamma$ for $j \in J_i^{*}$, and if $i$ is non-novel, $\exists j \in J_{i}^{*}$ such that $C_{i,i} - C_{i,j}$ converges to a non-positive value.
\end{itemize}
These statements imply Proposition \ref{AWCons}, which proves the consistency of the DDP Algorithm.
\QED

The term $\eta^{-8}$ seems to be the dominating factor in the sample complexity bound. Basically, \hbox{$\eta = \min\limits_{1 \leq i \leq W} \frac{1}{M} \E({\mathbf X}_i {\mathbf 1})$} represents the minimum proportion of documents that a word would appear in. This is not surprising as the rate of convergence of $C_{i,j} = M \langle \widetilde{\mathbf X}_i, \widetilde{\mathbf X}_j^{\prime} \rangle$ is dependent on the values of $\frac{1}{M} \E({\mathbf X}_i {\mathbf 1})$ and $\frac{1}{M} \E({\mathbf X}_j {\mathbf 1})$. As these values are decreased, $C_{i,j}$ converges to a larger value and the convergence get slower.
In another view, given that the number of words per document $N$ is bounded, in order to have $C_{i,j}$ converge, a large number of documents is needed to observe all the words sufficiently. It is remarkable that a similar term $p^{-6}$ would also  arise in the sample complexity bound of \cite{ARORA:ref}, where $p$ is the minimum non-zero element of diagonal part of ${\bm \beta}$. It may be noted that although it seems that the sample complexity bound scales logarithmically with $W$, $\eta$ and $p$ would be decreased typically as $W$ increases.

\subsection{Novel Word Clustering Consistency} 
We similarly prove the consistency and sample complexity of the novel word clustering algorithm : 
\begin{theorem} \label{Main2}
For $d = {\lambda_{\wedge} \beta_{\wedge}^2}$, given all true novel words as the input, the clustering algorithm, \hbox{Algorithm \ref{ACAlg}} (ClusterNovelWords) asymptotically (as $M\rightarrow \infty$ recovers $K$ novel word indices of different types, namely, the support of the corresponding ${\bm \beta}$ rows are different for any two retrieved indices. Furthermore, if 
\begin{equation*}
M \geq \frac{C_2 \left( \log W + \log \left( \frac{1}{\delta_2} \right)\right) } {\eta^8 \lambda_{\wedge}^2 \beta_{\wedge}^4}
\end{equation*}
then Algorithm \ref{ACAlg} clusters all novel words correctly with probability at least $1 - \delta_2$. \\
\end{theorem}

\noindent
{\it Proof Sketch. }
More detailed analysis is provided in the Supplementary section. We can show that $C_{i,i} - 2C_{i,j} + C_{j,j}$ converges to a strictly positive value $d$ if $i$ and $j$ are novel words of different topics. Moreover, it converges to zero if they are novel words of the same topic. Hence all novel words of the same topic are connected in the graph with high probability asymptotically. Moreover, there would not be an edge between the novel words of different topics with high probability. Therefore, the connected components of the graph corresponds to the true clusters asymptotically. The detailed discussion of the convergence rate is provided in the Supplementary section.
\QED

It is noticeable that the sample complexity of the clustering is similar to that of the novel word detection. This means that the hardness of novel word detection and distance based clustering using the proposed algorithms are almost the same. 

\subsection{Topic Estimation Consistency} 

Finally, we show that the topic estimation by regression is also consistent. 
\begin{theorem}\label{Main3}
Suppose that Algorithm \ref{TEAlg} outputs $\widehat{\bm \beta}$ given the indices of $K$ distinct novel words. Then, $\widehat{\bm \beta} \xrightarrow{p} {\bm \beta}$. Specifically, if 
\begin{equation*}
M \geq \frac{C_3 W^4(\log(W) + \log(K) + \log(1/\delta_3))}{\lambda_{\wedge}^2 \eta^8 \epsilon^4 a_{\wedge}^8}
\end{equation*}
then for all $i$ and $j$, $\widehat{\beta}_{i,j}$ will be $\epsilon$ close to $\beta_{i,j}$ with probability at least $1 - \delta_3$, with $\epsilon < 1$, $C_3$ being a constant, $a_{\wedge} = \min_{i} a_i$ and $\eta = \min\limits_{1 \leq i \leq W} {\bm \beta}_i {\mathbf a}$.\\
\end{theorem}

\noindent
{\it Proof Sketch.}
We will provide a detailed analysis in the Supplementary section. To prove the consistency of the regression algorithm, we will use a consistency result for the {\it extremum estimators} : If we assume $Q_{M}({\bm \beta})$ to be a stochastic objective function which is minimized at $\widehat{\bm \beta}$ under the constraint ${\bm \beta} \in \Theta$ (for a compact $\Theta$), and ${Q}_M({\bm \beta})$ converges uniformly to $\bar{Q}({\bm \beta})$, which in turn is minimized uniquely in ${\bm \beta}^{*}$, then $\widehat{\bm \beta} \xrightarrow{p} {\bm \beta}^{*}$ \cite{Tak:ref}. In our setting, we may take $Q_M$ to be the objective function in Algorithm \ref{TEAlg}. Then, $Q_M({\mathbf b}) \xrightarrow{p} \bar{Q}({\mathbf b}) = {\mathbf b} {\mathbf D} {\mathbf R} {\mathbf D} {\mathbf b}^{\top} - 2 {\mathbf b} {\mathbf D} {\mathbf R} \frac{{\bm \beta}_i^{\top}}{{\bm \beta}_i {\mathbf a}} + \frac{{\bm \beta}_i}{{\bm \beta}_i {\mathbf a}} {\mathbf R} \frac{{\bm \beta}_i^{\top}}{{\bm \beta}_i {\mathbf a}}$, where ${\mathbf D} = \diag({\mathbf a})^{-1}$. Note that if ${\mathbf R}$ is positive definite, $\bar{Q}$ is uniquely minimized at ${\mathbf b}^{*} = \frac{{\bm \beta}_i}{{\bm \beta}_i {\mathbf a}} {\mathbf D}^{-1} $, which satisfies the conditions of the optimization. Moreover, ${Q}_M$ converges to $Q$ uniformly as a result of Lipschitz continuity of $Q_M$. Therefore, according to Slutsky's theorem, $(\frac{1}{M}{\mathbf X}_i{\mathbf 1}) {\mathbf b}^{*} = \widehat{\bm \beta}_i$ converges to ${\bm \beta}_i {\mathbf D}^{-1} $, and hence the column normalization of $\widehat{\bm \beta}$ converges to ${\bm \beta}$. We will provide a more detailed analysis of this part in the Supplementary section. 
\QED

In sum, consider the approach outlined at the beginning of section \ref{sec:algorithm} based on data-dependent projections method, and assume that $\widehat{\bm{\beta}}$ is the output. Then,
\begin{theorem}\label{Main}
The output of the topic modeling algorithm $\widehat{\bm{\beta}}$ converges in probability to ${\bm \beta}$ element-wise. 
To be precise, if 
\begin{equation*}
M\geq\max\left\{\frac{C_2^{\prime} W^4\log\frac{WK}{\delta}}{\lambda_{\wedge}^2 \eta^8 \epsilon^4 a_{\wedge}^8}, \frac{C_1^{\prime} \log\frac{W}{\delta} }{\beta_{\wedge}^2 \eta^8 \min(\lambda_{\wedge}^2 \beta_{\wedge}^2, \zeta^2 a_{\wedge}^2) } \right\}
\end{equation*}
then with probability at least $1-3\delta$, for all $i$ and $k$, $\widehat{\beta}_{i,k} $ will be $\epsilon$ close to $\beta_{i,k}$, with $\epsilon < 1$, $C_1^{\prime}$ and $C_2^{\prime}$ being two constants.
\end{theorem}
The proof is a combination of Theorems \ref{Main1}, \ref{Main2} and \ref{Main3}.

\section{Experimental Results}
\label{sec:experiment}
\subsection{Practical Considerations} \label{PCon}
DDP algorithm requires two parameters $\gamma$ and $d$. In practice, we can apply DDP without knowing them adaptively and agnostically. 
Note that $d$ is for the construction of $J_i$. We can otherwise construct $J_i$ by finding $r < W$ words that are maximally distant from $i$ in the sense of Eq.~\ref{eq:ji}. 
To bypass $\gamma$, we can rank the values of $\min_{j \in J_i} M \langle \widetilde{\mathbf X}_i,  \widetilde{\mathbf X}_i^{\prime} \rangle - M \langle \widetilde{\mathbf X}_i,  \widetilde{\mathbf X}_j^{\prime} \rangle$ across all $i$ and declare the topmost $s$ values as the novel words.

The clustering algorithm also requires parameter $d$. Note that $d$ is just for thresholding a $0-1$ weighted graph. In practice, we could avoid hard thresholding by using $\exp(-(C_{i,i}-2C_{i,j}+C_{j,j}))$ as weights for the graph and apply spectral clustering. To point out, typically the size of $\mathcal{I}$ in Algorithm \ref{ACAlg} is of the same order as $K$. Hence the spectral clustering is on a relative small graph which typically adds $\mathcal{O}(K^3)$ computational complexity. 

\noindent
{\bf Implementation Details:} We choose the parameters of the DDP and RP in the following way. 
For DDP in all datasets except the Donoho image corpus, we use the agnostic algorithm discussed in section \ref{PCon} with $r = W/2$. Moreover, we take $s = 10 \times K$. For the image dataset, we used $d = 1$ and $\gamma = 3$. For RP, we set the number of projections $P \approx 50 \times K$ in all datasets to obtain the results.

\subsection{Synthetic Dataset}
\label{sec:synthetic}
%
%
%
\input{figure2}
%
%
%
In this section, we validate our algorithm on synthetic examples.
We generate a $W \times K$ separable topic matrix ${\bm \beta}$ with $W_1/K
> 1$ novel words per topic as follows: first, iid $1\times K$
row-vectors corresponding to non-novel words are generated uniformly
on the probability simplex. Then, $W_1$ iid $\mathrm{Uniform}[0,1]$
values are generated for the nonzero entries in the rows of novel
words. The resulting matrix is then column-normalized to get one
realization of ${\bm \beta}$. Let $\rho := W_1/W$. Next, $M$ iid $K\times 1$
column-vectors are generated for the $\theta$ matrix according to a
Dirichlet prior $c\prod\limits_{i=1}^{K} \theta_i^{\alpha_i
  -1}$. Following \cite{Griffiths:ref}, we set $\alpha_i = 0.1$ for
all $i$. Finally, we obtain ${\mathbf X}$ by generating $N$ iid words for each
document.

For different settings of $W$, $\rho$, $K$, $M$ and $N$, we calculate
the $\ell_1$ distance of the estimated topic matrix to the ground truth after finding the best matching between two sets of topics. For each setting we average the error over $50$ random samples. For RP \& DDP we set parameters as discussed in the implementation details. 


We compare the DDP and RP against the Gibbs sampling approach \cite{Griffiths:ref} (Gibbs), a state-of-art NMF-based algorithm \cite{betaDivergence:ref} (NMF) and the most recent practical provable algorithm in \cite{Arora2:ref} (RecL2).
The NMF algorithm is chosen because it compensates for the type of noise in our topic model. Fig. ~\ref{fig:synthetic} depicts the estimation error as a function of the number of documents $M$ (Upper) and the number of words/document $N$ (bottom). RP and DDP have similar performance and are uniformly better than comparable techniques. Gibbs performs relatively poor in the first setting and NMF in the second.  RecL2 perform worse in all the settings. Note that $M$ is relatively small ($\leq 1,000$) compared to $W=500$. DDP/RP outperform other methods with fairly small sample size. Meanwhile, as is also observed in \cite{Arora2:ref}, RecL2 has a poor performance with small $M$.
%
%
\vspace{-2ex}
\subsection{Swimmer Image Dataset}
%
%
\input{figure3}
%
%
\input{figure5}
\input{figure4}
%
%
%
%
%
%
In this section we apply our algorithm to the synthetic
\textit{swimmer} image dataset introduced in
\cite{Donhunique:ref}. There are $M = 256$ binary images, each with
$W=32\times32=1024$ pixels. Each image represents a swimmer composed
of four limbs, each of which can be in one of $4$ distinct positions,
and a torso.
We interpret pixel positions $(i,j)$ as words. Each image is interpreted as a document composed of pixel positions with non-zero values. Since each position of a limb features some unique pixels in the image, the topic matrix $\bm \beta$ satisfies the separability assumption
with $K = 16$ ``ground truth'' topics that correspond to $16$ {\it
  single} limb positions.

Following the setting of \cite{betaDivergence:ref}, we set body pixel
values to 10 and background pixel values to 1. We then take each
``clean'' image, suitably normalized, as an underlying distribution
across pixels and generate a ``noisy'' document of $N = 200$ iid
``words'' according to the topic model. Examples are shown in
Fig.~\ref{fig:sample_swimmer}. We then apply RP and DDP algorithms to the
``noisy'' dataset and compare against Gibbs \cite{Griffiths:ref}, NMF \cite{betaDivergence:ref}, and RecL2 \cite{Arora2:ref}.
Results are shown in Figs. \ref{fig:badimages} and \ref{fig:noisy_swimmer}. We set the parameters as discussed in the implementation details.

This dataset is a good validation test for different algorithms since
the ground truth topics are known and unique. As we see in
Fig.~\ref{fig:badimages}, both Gibbs and NMF produce topics that do not
correspond to any {\it pure} left/right arm/leg positions. Indeed,
many of them are composed of multiple limbs. Nevertheless, as shown in Fig.~\ref{fig:noisy_swimmer}, no
such errors are realized in RP and DDP and our topic-estimates are
closer to the ground truth images.
In the meantime, RecL2 algorithm failed to work even with the clean data.
Although it also extracts extreme points of a convex body, the algorithm additionally requires these points to be linearly independent. It is possible that extreme points of a convex body are linearly dependent (for example, a 2-D square on a 3-D simplex). This is exactly the case in the {\it swimmer} dataset.
As we see in the last row in Fig.~\ref{fig:noisy_swimmer}, RecL2 produces only a few topics close to ground truth. Its extracted topics for the noisy images are shown in Fig.~\ref{fig:badimages}. Results of RecL2 on noisy images are no close to ground truth as shown in Fig.~\ref{fig:badimages}.



\vspace{-1pt}

\subsection{Real World Text Corpora}

%
\input{table1}
\input{table2}
%
%
%


In this section, we apply our algorithm on two different real world text corpora from \cite{Frank+Asuncion:2010}.
The smaller corpus is NIPS proceedings dataset with $M=1,700$ documents, a vocabulary of $W=14,036$ words
and an average of $N \approx 900$ words in each document.
Another is a large corpus New York {\it (NY)} Times articles dataset, with $M=300,000$, $W=102,660$, and $N
\approx 300$. The vocabulary is obtained by deleting a standard
``stop'' word list used in computational linguistics, including
numbers, individual characters, and some common English words such as ``the''.

In order to compare with the practical algorithm in \cite{Arora2:ref},
we followed the same pruning in their experiment setting to shrink the vocabulary size to $W=2,500$ for NIPS and  $W=15,000$ for NY Times. Following typical settings in \cite{Blei2012Review:ref}  and \cite{Arora2:ref}, we set $K=40$ for NIPS and $K=100$ for {\it NY} Times. We set our parameters as discussed in implementation details.

We compare DDP and RP algorithms against RecL2 \cite{Arora2:ref} and a practically widely successful algorithm \cite{Griffiths:ref}(Gibbs).
Table~\ref{topicword:table} and \ref{topicword:table2}\footnote{the zzz prefix annotates the named entity.} depicts typical topics extracted by the different methods.
For each topic, we show its most frequent words, listed in descending order of the estimated probabilities.
Two topics extracted by different algorithms are grouped if they are close in $\ell_1$ distance.

Different algorithms extract some fraction of similar topics which are easy to recognize.
Table~\ref{topicword:table} indicates most of the topics extracted by RP and DDP are similar and are comparable with that of Gibbs.
We observe that the recognizable themes formed with DDP or RP topics are more abundant than that by RecL2. For example, topic on ``chip design'' as shown in the first panel in Table~\ref{topicword:table} is not extracted by RecL2, and topics in Table~\ref{topicword:table2} on ``weather'' and ``emotions'' are missing in RecL2. Meanwhile, RecL2 method produces some obscure topics. For example, in the last panel of Table~\ref{topicword:table}, RecL2 contains more than one theme, and in the last panel of Table~\ref{topicword:table2} RecL2 produce some unfathomable combination of words. More details about the topics extracted are given in the Supplementary section.


\begin{table*}[!htb]
\centering
\caption{Comparison of Approaches. Recover from \cite{ARORA:ref}; RecL2 from \cite{Arora2:ref}; ECA from \cite{Anan12}; Gibbs from \cite{Griffiths:ref}; NMF from \cite{nmfLS:ref}. $Time_W(L.P)$, $Time_K(L.R)$ stands for computation time for Linear Programming or Linear Regression for $W$ and $K$ number of variables respectively; The definition of the set of parameters can be found in the reference papers. \\}
{\small
\begin{tabular}{|p{0.19\linewidth}|p{0.17\linewidth}|p{0.30\linewidth}|p{0.12\linewidth}|p{0.12\linewidth}|}
\hline
Method & Computational Complexity & Sample complexity($M$) &Assumptions & Remarks \\
\hline
DDP & $\mathcal{O}(N^2M+W^2)+W~Time_K(L.R)$ & \begin{multline*} \max\Big\{\frac{C_2^{\prime} W^4\log\frac{WK}{\delta}}{\lambda_{\wedge}^2 \eta^8 \epsilon^4 a_{\wedge}^8}, \\ \frac{C_1^{\prime} \log\frac{W}{\delta} }{\beta_{\wedge}^2 \eta^8 \min(\lambda_{\wedge}^2 \beta_{\wedge}^2, \zeta^2 a_{\wedge}^2) } \Big\} \end{multline*} & Separable $\bm\beta$; $P1$ and $P2$ on Prior Distribution of ${\bm \theta}$ (Sec. \ref{TheoSec}); Knowledge of $\gamma$ and $d$ (defined in Algorithm \ref{AWAlg}) & $\Pr({\text{Error}})\rightarrow 0$ exponentially \\
\hline
RP & $\mathcal{O}(MNK+WK)+W~Time_K(L.R)$& N/A & Separable $\bm\beta$& \\
\hline
Binning & $\mathcal{O}(MN+\sqrt{M}W)+W~Time_K(L.R)$& N/A & Separable $\bm\beta$& \\
\hline
Recover \cite{ARORA:ref}  & $\mathcal{O}(MN^2)+$ $W~Time_W(L.P)+ W~Time_K(L.R)$ & $\max\left\{\frac{C \log(W)a^4 K^6}{\epsilon^2 p^6 \gamma^2 }, \frac{a^2 K^4 \log K }{\gamma^2} \right\}$&Separable $\bm\beta$; Robust Simplicial Property of ${\mathbf R}$& $\Pr({\text{Error}})\rightarrow 0$; Too many Linear Programmings make the algorithm impractical \\
\hline
RecL2 \cite{Arora2:ref}  & $\mathcal{O}(W^2+WK/\epsilon^2 +$ $MN^2) +$ $W~Time_K(L.R)$ & $ \max\left\{\frac{C_1 a K^3\log(W)}{\epsilon  \gamma^6 p^6},  \frac{C_2 a^3 K^3\log(W)}{\epsilon^3\gamma^4 p^4} \right\}$ & Separable $\bm\beta$; Robust Simplicial Property of ${\mathbf R}$ & $\Pr({\text{Error}})\rightarrow 0$; Requires Novel words to be linearly independent;\\
\hline
ECA \cite{Anan12} &$\mathcal{O}(W^3+MN^2)$& N/A : For the provided basic algorithm, the probability of error is at most $1/4$ but does not converge to zero & LDA model; The concentration parameter of the Dirichlet distribution $\alpha_0$ is known & Requires solving SVD for large matrix, which makes it impractical; $\Pr({\text{Error}})\nrightarrow 0$ for the basic algorithm \\
\hline
Gibbs \cite{Griffiths:ref} &N/A &N/A & LDA model &  No convergence guarantee \\
\hline
NMF \cite{betaDivergence:ref} &N/A &N/A &General model& Non-convex optimization;  No convergence guarantee\\
\hline
\end{tabular}
}
\label{table:comparex}
\end{table*}

\section{Conclusion and Discussion}
We summarize our proposed approaches (DDP, Binning and RP) while comparing with other existing methods in terms of assumptions, computational complexity and sample complexity (see Table ~\ref{table:comparex}). Among the list of the algorithms, DDP and RecL2 are the best and competitive methods. While the DDP algorithm has a polynomial sample complexity, its running time is better than that of RecL2, which depends on $1/\epsilon^2$. Although $\epsilon$ seems to be independent of $W$, by increasing $W$ the elements of ${\bm \beta}$ would be decreased and the precision ($\epsilon$) which is needed to recover ${\bm \beta}$ would be decreased. This results in a larger time complexity in RecL2. In contrast, time complexity of DDP does not scale with $\epsilon$. On the other hand, the sample complexity of both DDP and RecL2, while polynomially scaling, depend on too many different terms. This makes the comparison of these sample complexities difficult. However, terms corresponding to similar concepts appeared in the two bounds. For example, it can be seen that $p a_{\wedge} \approx \eta$, because the novel words are possibly the most rare words. Moreover, $\lambda_{\wedge}$ and $\gamma$ which are the $\ell^2$ and $\ell^1$ condition numbers of ${\mathbf R}$ are closely related.  Finally, $a = \frac{a_{\vee}}{a_{\wedge}} $, with $a_{\vee}$ and $a_{\wedge}$ being the maximum and minimum values in ${\mathbf a}$.

%

\bibliography{icmlpaper}
\bibliographystyle{icml2013}

\newpage

\appendix

{\bf {\Large Supplementary Materials}} 
\section{Proofs}
Given ${\bm \beta}$ is separable, we can reorder the rows of ${\bm \beta}$ such that ${\bm \beta} = \left[ \begin{matrix} {\mathbf D} \\ {{\bm \beta}^{\prime}} \end{matrix} \right]$, where ${\mathbf D}$ is diagonal. We will assume the same structure for ${\bm \beta}$ throughout the section. 
%
\subsection{Proof of Proposition 3}
Proposition 3 is a direct result of Theorem 1. Please refer to section \ref{sec:thm1} for more details. 
\subsection{Proof of Proposition 4}
Recall that Proposition 4 summarizes the computational complexity of the DDP Algorithm 1. Here we provide more details.

\noindent
{\bf Proposition 4 (in Section 4.1).}
{\it The running time of Data dependent projection Algorithm DDP 1 is $\mathcal{O}(MN^2 + W^2)$. }

\noindent
{\bf Proof :}
We can show that, because of the sparsity of ${\mathbf X}$, ${\mathbf C} = M {\mathbf X} {\mathbf X}^{\prime \top}$ can be computed in $\mathcal{O}(MN^2 + W)$ time. First, note that ${\mathbf C}$ is
a scaled word-word co-occurrence matrix, which can be calculated by adding up the co-occurrence matrices of each document. 
This running time can be achieved, if all $W$ words in the vocabulary are first indexed by a hash table (which takes $\mathcal{O}(W)$). Then, since each document consists
of at most $N$ words, $\mathcal{O}(N^2)$ time is needed to compute the co-occurrence matrix of each document. Finally,
the summation of these matrices to obtain ${\mathbf C}$ would cost $\mathcal{O}(MN^2)$, which results in total $\mathcal{O}(MN^2 + W)$ time complexity. Moreover, for 
each word $i$, we have to find $J_i$ and test whether $C_{i,i} - C_{i,j} \geq \gamma/2$ for all $j \in J_i$. Clearly, the cost to do this is $\mathcal{O}(W^2)$ in the worst case.
\QED

\subsection{Proof of Proposition 5}
Recall that Proposition 5 summarizes the computational complexity of RP ( Algorithm 2) and Binning (and see Section B in appendix for more details). Here we provide a more detailed proof.

\noindent
{\bf Proposition 5 (in Section 4.1)}
{\it Running time of RP (Algorithm 2) and Binning algorithm (in Appendix Section B) are $\mathcal{O}(MNK + WK)$ and $\mathcal{O}(MN + \sqrt{M}W)$, respectively. }

\noindent
{\bf Proof :}
Note that number of operations needed to find the projections is 
$\mathcal{O}(MN + W)$ in Binning and $\mathcal{O}(MNK + W)$ in RP. This can be achieved by first indexing the words by a hash table and then finding the projection of 
each document along the corresponding component of the random directions. Clearly, that takes $\mathcal{O}(N)$ time for each document.
In addition, finding the word with the maximum projection value (in RP) and the winner in each bin (in Binning) will take $\mathcal{O}(W)$. This counts to be $\mathcal{O}(WK)$ for all projections in RP and $\mathcal{O}(\sqrt{M}W)$ for all of the bins in Binning. Adding running time of these two parts, the computational complexity of the RP and Binning algorithms will be $\mathcal{O}(MNK + WK)$ and $\mathcal{O}(MN + \sqrt{M}W)$, respectively.
\QED

\subsection{Proof of Proposition 6}
Proposition 6 (in Section 4.2) is a direct result of Theorem 2. Please read section \ref{sec:thm2} for the detailed proof.

\subsection{Validation of Assumptions in Section 5 for Dirichelet Distribution} \label{P1P2A}
In this section, we prove the validity of the assumptions $P1$ and $P2$ which were made in Section 5.

For $\mathbf{x}\in\rR^{K}$ with $\sum_{i=1}^{K}x_i =1, x_i\geq 0$, $\mathbf{x}\sim Dir(\alpha_1, \ldots, \alpha_{K}) $ has pdf $\pP(\mathbf{x})=c\prod_{i=1}^{K} x_i^{\alpha_i-1} $. Let $\alpha_{\wedge} = \min\limits_{1 \leq i \leq K} \alpha_i$ and \hbox{$\alpha_0 = \sum_{i = 1}^K \alpha_i$}.

\noindent
{\bf Proposition A.1 }
For a  Dirichlet prior $Dir(\alpha_1, \ldots, \alpha_{K})$:
\begin{enumerate} 
\vspace*{-0.4cm}
\item The correlation matrix ${\mathbf R}$ is positive definite with minimum eigenvalue $\lambda_{\wedge} \geq \frac{\alpha_{\wedge}}{\alpha_0 (\alpha_0 + 1)}$, 
\vspace*{-0.1cm}
\item $\forall 1\leq i\neq j\leq K$, $\frac{R_{i,i}}{a_i a_i}- \frac{R_{i,j}}{a_i a_j} = \frac{\alpha_0}{\alpha_i (\alpha_0 + 1)} >0$.
\end{enumerate}
\begin{proof} 
The covariance matrix of $Dir(\alpha_1, \ldots, \alpha_K)$, denoted as $\Sigma$, can be written as 
\begin{equation} \label{SDir}
\Sigma_{i,j} = \left\{ \begin{array}{ll}
\frac{-\alpha_i \alpha_j}{\alpha_0^2 (\alpha_0 + 1)} & \mbox{if } i \neq j \\
\frac{\alpha_i (\alpha_0 - \alpha_i)}{\alpha_0^2 (\alpha_0 + 1)} & \mbox{otherwise}
\end{array}
\right.
\end{equation}
Compactly we have \hbox{$\Sigma = \frac{1}{\alpha_0^2 (\alpha_0 + 1)} \left(-{\bm \alpha} {\bm \alpha}^{\top} + \alpha_0 \diag({\bm \alpha})\right) $} with ${\bm \alpha} = (\alpha_1, \ldots, \alpha_K)$. The mean vector ${\bm \mu} = \frac{1}{\alpha_0} {\bm \alpha}$. Hence we obtain
\begin{align*}
{\mathbf R} & = \frac{1}{\alpha_0^2 (\alpha_0 + 1)} \left(-{\bm \alpha} {\bm \alpha}^{\top} + \alpha_0 \diag({\bm \alpha})\right) + \frac{1}{\alpha_0^2} {\bm \alpha}{\bm {\alpha}}^{\top} \\
& = \frac{1}{\alpha_0 (\alpha_0 + 1)} \left( {\bm \alpha} {\bm \alpha}^{\top} + \diag({\bm \alpha}) \right)
\end{align*}
Note that $\alpha_i > 0$ for all $i$, ${\bm \alpha}{\bm {\alpha}}^{\top}$ and $\diag({\bm \alpha})$ are positive definite. Hence ${\mathbf R}$ is strictly positive definite, with eigenvalues $\lambda_{i} = \frac{\alpha_{i}}{\alpha_0 (\alpha_0 + 1)}$. Therefore $\lambda_{\wedge} \geq \frac{\alpha_{\wedge}}{\alpha_0 (\alpha_0 + 1)}$.
The second property follows by directly plug in equation \eqref{SDir}.
\end{proof}

\subsection{Convergence Property of the co-occurrence Matrix}
In this section, we prove a set of Lemmas as ingredients to prove the main Theorems 1, 2 and 3 in Section 5. These Lemmas in sequence show :
\begin{itemize}
\vspace*{-0.2cm}
\item Convergence of ${\mathbf C} = M \widetilde{\mathbf X} \widetilde{\mathbf X}^{\prime \top}$; (Lemma \ref{CijLemma})
\vspace*{-0.2cm}
\item Convergence of $C_{i,i} - 2C_{i,j} + C_{j,j}$ to a strictly positive value if $i,j$ are not novel words of the same topic; (Lemma \ref{DistLemma})  
\vspace*{-0.2cm}
\item Convergence of $J_i$ to $J_i^{*}$ such that if $i$ is novel, $C_{i,i} - C_{i,j}$ converges to a strictly positive value for $j \in J_i^{*}$, and if $i$ is non-novel, $\exists j \in J_{i}^{*}$ such that $C_{i,i} - C_{i,j}$ converges to a non-positive value (Lemmas \ref{JLemma} and \ref{CLemma}).
\end{itemize}
Recall that in Algorithm 1, ${\mathbf C} = M \widetilde{\mathbf{X}}{\widetilde{\mathbf{X}}}^{\prime \top}$. Let's further define $E_{i,j}= \frac{{\bm \beta}_i}{{\bm \beta}_i {\mathbf a}} {\mathbf R} \frac{{\bm \beta}_j^{\top}}{{\bm \beta}_j {\mathbf a}}$. $\eta = \min\limits_{1 \leq i \leq W} {\bm \beta}_i {\mathbf a}$. Let $\mathbf{R}$ and $\mathbf{a}$ be the correlation matrix and mean vector of prior distribution of $\bm\theta$. 

Before we dig into the proofs, we provide two limit analysis results of Slutsky's theorem :
\begin{proposition}
\label{prop:tech}
For random variables $X_n$ and $Y_n$ and real numbers $x, y \geq 0$, if  $\Pr(|X_n - x| \geq \epsilon) \leq g_n(\epsilon)$ and $\Pr(|Y_n - y| \geq \epsilon) \leq h_n(\epsilon)$, then 
\begin{equation*}
\Pr(|X_n/Y_n - x/y| \geq \epsilon) \leq g_n\left(\frac{y \epsilon}{4}\right) + h_n\left(\frac{\epsilon y^2}{4 x}\right) + h_n\left(\frac{y}{2}\right)
\end{equation*}
And if $0 \leq x, y \leq 1$
\begin{multline*}
\Pr(|X_n Y_n - x y| \geq \epsilon) \leq g_n\left(\frac{\epsilon}{2}\right) + h_n\left(\frac{\epsilon}{2}\right) \\ + g_n\left(\frac{\epsilon}{2y}\right) + h_n\left(\frac{\epsilon}{2x}\right)
\end{multline*}
\end{proposition}
 
\begin{lemma} \label{CijLemma}
Let $C_{i,j} \triangleq M {\widetilde{\mathbf X}}_i \widetilde{{\mathbf X}}^{\prime}_j $. Then $C_{i,j} \xrightarrow{p} E_{i,j}= \frac{{\bm \beta}_i}{{\bm \beta}_i {\mathbf a}} {\mathbf R} \frac{{\bm \beta}_j^{\top}}{{\bm \beta}_j {\mathbf a}}$. Specifically, 
\begin{equation*}
\Pr\left(\left|C_{i,j} - E_{i,j} \right| \geq \epsilon \right) \leq 8 \exp(-M\epsilon^2 \eta^8 /32)  
\end{equation*}
\end{lemma}
\begin{proof}
By the definition of $C_{i,j}$, we have :
\begin{equation} 
\label{Cij}
\begin{split}
C_{i,j} = \frac{\frac{1}{M} {\mathbf X}_i {\mathbf X}^{\prime \top}_j}{(\frac{1}{M} {\mathbf X}_i {\mathbf 1})(\frac{1}{M} {\mathbf X}^{\prime}_j {\mathbf 1})} \xrightarrow{p} & \frac{\E(\frac{1}{M} {\mathbf X}_i {\mathbf X}^{\prime \top}_j)}{\E (\frac{1}{M}  {\mathbf X}_i {\mathbf 1}) \E(\frac{1}{M} {\mathbf X}^{\prime}_j {\mathbf 1})}
\end{split}
\end{equation}
as $M \rightarrow \infty$, where ${\mathbf 1} = (1, 1, \ldots, 1)^{\top}$ and the convergence follows because of convergence of numerator and denominator and then applying the Slutsky's theorem. The convergence of numerator and denominator are results of strong law of large numbers due to the fact that entries in ${\mathbf X}_i$ and ${\mathbf X}_i^{\prime}$ are independent. 

To be precise, we have:
\begin{align*}
& \frac{\E(\frac{1}{M} {\mathbf X}_i {\mathbf X}^{\prime \top}_j)}{\E (\frac{1}{M}  {\mathbf X}_i {\mathbf 1}) \E(\frac{1}{M} {\mathbf X}^{\prime}_j {\mathbf 1})}\\
 = &\frac{\E_{{\bm \theta}} \E_{{\mathbf X} | {\bm \theta} } (\frac{1}{M} {\mathbf X}_i {\mathbf X}^{\prime \top}_j)}{\E_{{\bm \theta}} \E_{{\mathbf X} | {\bm \theta} } (\frac{1}{M} {\mathbf X}_i {\mathbf 1}) \E_{{\bm \theta}} \E_{{\mathbf X} | {\bm \theta} } (\frac{1}{M}  {\mathbf X}^{\prime}_j {\mathbf 1})} \\
 = &\frac{\E_{{\bm \theta}} (\frac{1}{M} {\mathbf A}_i {\mathbf A}^{\top}_j)}{\E_{{\bm \theta}} (\frac{1}{M}  {\mathbf A}_i {\mathbf 1}) \E_{{\bm \theta}} (\frac{1}{M}  {\mathbf A}_j {\mathbf 1})} \\
 =& \frac{\E_{{\bm \theta}} (\frac{1}{M} {\bm \beta}_i {\bm \theta} {\bm \theta}^{\top} {\bm \beta}_j)}{\E_{{\bm \theta}} (\frac{1}{M} {\bm \beta}_i {\bm \theta}  {\mathbf 1} ) \E_{{\bm \theta}} (\frac{1}{M} {{\bm \beta}_j {\bm \theta} \mathbf 1})} \\
 =& \frac{{\bm \beta}_i {\mathbf R} {\bm \beta}_j^{\top}}{({\bm \beta}_i {\mathbf a}) ({\bm \beta}_j {\mathbf a})} \\
 =& E_{i,j}
\end{align*}

To show the convergence rate explicitly, we use proposition \ref{prop:tech}.
For simplicity,  define $C_{i,j} = \frac{F_{i,j}}{G_{i} H_{j}}$. Note that entries in ${\mathbf X}_i$ and ${\mathbf X}_i^{\prime}$ are independent and bounded, by Hoeffding's inequality, we obtain:
\begin{eqnarray*}
\Pr(|F_{i,j} - \E(F_{i,j}) | \geq \epsilon) \leq 2 \exp(-2M \epsilon^2) \\
\Pr(|G_{i} - \E(G_{i}) | \geq \epsilon) \leq 2 \exp(-2M \epsilon^2) \\
\Pr(|H_{j} - \E(H_{j}) | \geq \epsilon) \leq 2 \exp(-2M \epsilon^2) 
\end{eqnarray*}
Hence, 
\begin{multline*}
\Pr(|G_{i} H_{j} - \E(G_{i})\E(H_{j}) | \geq \epsilon) \leq 8 \exp(-M \epsilon^2/2) 
\end{multline*}
and 
\begin{multline} \label{ConRateCij}
\Pr\left(\left|\frac{F_{i,j}}{G_{i} H_{j}} - \frac{\E(F_{i,j})}{\E(G_{i}) \E(H_{j})} \right| \geq \epsilon\right) \leq \\ 2 \exp(-M \epsilon^2 ({\bm \beta}_j {\mathbf a} {\bm \beta}_i {\mathbf a})^2/8) + 8 \exp(-M \epsilon^2 ({\bm \beta}_j {\mathbf a} {\bm \beta}_i {\mathbf a})^4/32)  \\ + 8 \exp(-M ({\bm \beta}_j {\mathbf a} {\bm \beta}_i {\mathbf a})^2/8) \\
\end{multline}
Let $\eta = \min\limits_{1 \leq i \leq W} {\bm \beta}_i {\mathbf a} \leq 1$. We obtain 
\begin{multline*}
\Pr\left(\left|\frac{F_{i,j}}{G_{i} H_{j}} - \frac{\E(F_{i,j})}{\E(G_{i}) \E(H_{j})} \right| \geq \epsilon\right) \\ \leq 18 \exp(-M \epsilon^2 \eta^8/32)  
\end{multline*}
\end{proof}
%
%
\begin{corollary}
\label{Cor:cii-2cij+cjj}
$C_{i,i} - 2C_{i,j} + C_{j,j}$ converges as $M \rightarrow \infty$. The convergence rate is $c_1 \exp(-M c_2 \epsilon^2 \eta^8)$ for $\epsilon$ error, with $c_1$ and $c_2$ being constants in terms of $M$.
\end{corollary}
\begin{corollary}
\label{Cor:cii-cjj}
$C_{i,i} - C_{i,j}$ converges as $M \rightarrow \infty$.  The convergence rate is $d_1 \exp(-M d_2 \epsilon^2 \eta^8)$ for $\epsilon$ error, with $d_1$ and $d_2$ being constants in terms of $M$. \\
\end{corollary}
Recall that we define $\mathcal{C}_k, k=1,\ldots,K$ to be the novel words of topic $k$, and $\mathcal{C}_0$ to be the set of non-novel words.  $\supp({\bm \beta}_i)$ denotes the column indices of non-zero entries of a row vector $\bm{\beta}_i$ of $\bm\beta$ matrix.
%
\begin{lemma} \label{DistLemma}
 If $i,j\in\mathcal{C}_k$, ($i,j$ are novel words of the same topic), then $C_{i,i} - 2C_{i,j}  + C_{j,j} \xrightarrow{p} 0$. Otherwise, $\forall k$, if  $i\in\mathcal{C}_k, j\notin\mathcal{C}_k$, then $C_{i,i} - 2C_{i,j}  + C_{j,j} \xrightarrow{p} f_{(i,j)}\geq d > 0$ where $d = {\lambda_{\wedge} \beta_{\wedge}^2}$. Especially, if $i\in\mathcal{C}_0$ and $j\notin \mathcal{C}_0$, then $C_{i,i} - 2C_{i,j}  + C_{j,j} \xrightarrow{p} f_{(i,j)}\geq d > 0$
%
\end{lemma}
\begin{proof}
It was shown in lemma \ref{CijLemma} that $C_{i,j} \xrightarrow{p} \frac{{\bm \beta}_i}{{\bm \beta}_i {\mathbf a}} {\mathbf R} \frac{{\bm \beta}_j^{\top}}{{\bm \beta}_j {\mathbf a}}$, where ${\mathbf R}$ is the correlation matrix  and ${\mathbf a} = (a_1, \ldots, a_K)^{\top}$ is the mean of the prior. Hence
\begin{align*}
C_{i,i} - 2C_{i,j} + C_{j,j} \xrightarrow{p} & \left(\frac{{\bm \beta}_i}{{\bm \beta}_i {\mathbf a}} - \frac{{\bm \beta}_j}{{\bm \beta}_j {\mathbf a}}\right) {\mathbf R} \left(\frac{{\bm \beta}_i}{{\bm \beta}_i {\mathbf a}} - \frac{{\bm \beta}_j}{{\bm \beta}_j {\mathbf a}}\right) \\
& \geq \lambda_{\wedge} \left\| \frac{{\bm \beta}_i}{{\bm \beta}_i {\mathbf a}} - \frac{{\bm \beta}_j}{{\bm \beta}_j {\mathbf a}} \right\|^2
\end{align*}
Note that we've assumed ${\mathbf R}$ to be positive definite with its minimum eigenvalue lower bounded by a positive value, $\lambda_{\wedge} > 0$. 

If $i,j \in\mathcal{C}_k$ for some $k$, then $\frac{{\bm \beta}_i}{{\bm \beta}_i {\mathbf a}} - \frac{{\bm \beta}_j}{{\bm \beta}_j {\mathbf a}} = {\mathbf 0}$ and hence $C_{i,i} - 2C_{i,j} + C_{j,j} \xrightarrow{p} 0$ .

Otherwise, if $\supp({\bm \beta}_i) \neq \supp({\bm \beta}_j)$, then $\left\| \frac{{\bm \beta}_i}{{\bm \beta}_i {\mathbf a}} - \frac{{\bm \beta}_j}{{\bm \beta}_j {\mathbf a}} \right\|^2 \geq \beta_{\wedge}^2$, (note that ${\bm \beta}_i {\mathbf a} \leq 1$) which proves the first part of the lemma.

For the second part, note that if $i\in\mathcal{C}_0$ and $j \notin\mathcal{C}_0$, the support of ${\bm \beta}_i$ and ${\bm \beta}_j$ is necessarily different. Hence, the previous analysis directly leads to the conclusion.
\end{proof}
Recall that in Algorithm 1, $J_i=\{j~:~ j\neq i, C_{i,i}-2C_{i,j}+C_{j,j} \geq d/2 \}$. we have :
%
\begin{lemma} \label{JLemma}
$J_i$ converges in probability in the following senses:
\begin{enumerate}
\vspace*{-0.3cm}
\item For a novel word $i\in\mathcal{C}_k$, define {$J^{*}_{i} = {\mathcal{C}_k}^{c}$ }. Then for all novel words $i$, $\lim\limits_{M \rightarrow \infty} \Pr(J_i \subseteq J^{*}_i) = 1$.
\vspace*{-0.3cm}
\item For a nonnovel word $i\in\mathcal{C}_0$, define {$J^{*}_{i} ={\mathcal{C}_0}^{c}$}. Then for all non-novel words $i$, $\lim\limits_{M \rightarrow \infty} \Pr(J_i \supseteq J^{*}_i) = 1$.  \\
\end{enumerate}
\end{lemma}
%
\begin{proof}
Let $d \triangleq {\lambda_{\wedge} \beta_{\wedge}^2}$. 
According to the lemma \ref{DistLemma}, whenever $\supp({\bm \beta}_j) \neq \supp({\bm \beta}_i)$, $D_{i,j} \triangleq C_{i,i} - 2C_{i,j} + C_{j,j} \xrightarrow{p} f_{(i,j)} \geq d$ for the novel word $i$. In another word,  for a novel word $i\in\mathcal{C}_k$ and $j\notin\mathcal{C}_k$, $D_{i,j}$ will be concentrated around a value greater than or equal to $d$. Hence, the probability that $D_{i,j}$ be less than $d/2$ will vanish. In addition, by union bound we have
\begin{align*}
\Pr(J_i \nsubseteq J^{*}_i) & \leq \Pr(J_i \neq J^{*}_i) \\
& = \Pr(\exists j \in J^{*}_i ~:~ j \notin J_i) \\
& \leq \sum_{j \in J^{*}_i} \Pr(j \notin J_i) \\
& \leq \sum\limits_{j \notin\mathcal{C}_k} \Pr(D_{i,j} \leq d/2)
\end{align*}
Since $\sum_{j \notin\mathcal{C}_k} \Pr(D_{i,j} \leq d/2)$ is a finite sum of vanishing terms given $i\in\mathcal{C}_k$, $\Pr(J_i \nsubseteq J^{*}_i)$ also vanish as $M\rightarrow\infty$ and hence we prove the first part.
For the second part, note that for a non-novel word $i\in\mathcal{C}_0$, $D_{i,j}$ converges to a value no less than $d$ provided that $j\notin\mathcal{C}_0$ (according to the lemma \ref{DistLemma}). Hence
\begin{align*}
\Pr(J_i \nsupseteq J^{*}_i) & \leq \Pr(J_i \neq J^{*}_i) \\
& = \Pr(\exists j \in J^{*}_i ~:~ j \notin J_i) \\
& \leq \sum_{j \in J^{*}_i} \Pr(j \notin J_i) \\
& \leq \sum\limits_{j \notin\mathcal{C}_0} \Pr(D_{i,j} \leq d/2)
\end{align*}
Similarly $\sum_{j \notin\mathcal{C}_0} \Pr(D_{i,j} \leq d/2)$  vanishes for a non-novel word $i\in\mathcal{C}_0$ as $M\rightarrow \infty$, $\Pr(J_i \nsupseteq J^{*}_i)$ will also vanish and hence concludes the second part.
\end{proof}

As a result of Lemma \ref{CijLemma}, \ref{DistLemma} and \ref{JLemma}, the convergence rate of events in Lemma \ref{JLemma} is :
\begin{corollary}
\label{Cor:Jrate}
For a novel word $i\in\mathcal{C}_k$ we have $\Pr(J_i \nsubseteq J^{*}_i) \leq W c_1 \exp(-M c_3 d^2 \eta^8)$. And for a non-novel word $i\in\mathcal{C}_0$, $\Pr(J_i \nsupseteq J^{*}_i) \leq K c_1 \exp(-M c_4 d^2 \eta^8)$, where $c_1$, $c_3$, and $c_4$ are constants and $d = {\lambda_{\wedge} \beta_{\wedge}^2}$. \\
\end{corollary}
\begin{lemma} \label{CLemma}
If $\forall  i \neq j $, $\frac{R_{i,i}}{a_i a_i} - \frac{R_{i,j}}{a_i a_j} \geq \zeta$, we have the following results on the convergence of $C_{i,i} - C_{i,j}$ :
\begin{enumerate}
\vspace*{-0.3cm}
\item If $i$ is a novel word, $\forall j \in J_i \subseteq J^{*}_i ~:~ C_{i,i} - C_{i,j} \xrightarrow{p} g_{(i,j)}\geq \gamma >0$, where $J^{*}_i$ is defined in  lemma \ref{JLemma}, $\gamma \triangleq \zeta { a_{\wedge} \beta_{\wedge}}$ and $a_{\wedge}$ is the minimum component of ${\mathbf a}$.
\vspace*{-0.3cm}
\item If $i$ is a non-novel word, $\exists j \in J^{*}_i$ such that $ C_{i,i} - C_{i,j} \xrightarrow{p} g_{(i,j)} \leq 0$. 
\end{enumerate}
\end{lemma}
%
\begin{proof}
Let's reorder the words so that $i\in\mathcal{C}_i$. Using the equation \eqref{Cij}, $C_{i,i} \xrightarrow{p} \frac{R_{i,i}}{a_i a_i}$ and  $C_{i,j} \xrightarrow{p} \sum_{k = 1}^{K} b_k  \frac{R_{i,k}}{a_i a_k}$ with $b_k \triangleq \frac{\beta_{j,k} a_k}{\sum_{l = 1}^{K} \beta_{j,l} a_l}$. Not that $b_k$'s are non-negative and sum up to one.

By the assumption, $\frac{R_{i,i}}{a_i a_i} - \frac{R_{i,j}}{a_i a_j} \geq \zeta $ for $j \neq i$. 
Note that $\forall j\in J_i \subseteq J_{i}^{*}$, there exists some index $k \neq i$ such that $b_k \neq 0$. Then 
\begin{align*}
C_{i,i} - C_{i,j} \xrightarrow{p} & \frac{R_{i,i}}{a_i a_i} - \sum_{k = 1}^{K} b_k  \frac{R_{i,k}}{a_i a_k} \\
& = \sum_{k = 1}^{K} b_k \left( \frac{R_{i,i}}{a_i a_i} - \frac{R_{i,k}}{a_i a_k} \right) \\
& \geq \zeta \sum_{k \neq i} b_k
\end{align*}
Since ${\bm \beta}_j {\mathbf a} \leq 1$, we have $\sum_{k \neq i} b_k \geq \frac{\beta_{\wedge} a_{\wedge}}{{\bm \beta}_j {\mathbf a}} \geq {\beta_{\wedge} a_{\wedge}}$, and the first part of the lemma is concluded.
To prove the second part, note that for $i\in\mathcal{C}_0$ and $j\notin\mathcal{C}_0$,
\begin{align*}
C_{i, j} \xrightarrow{p} \sum_{k = 1}^{K} b_k \frac{R_{j, k}}{a_{j} a_k}
\end{align*}
with $b_k = \frac{\beta_{i,k}}{{\bm \beta}_i {\mathbf a}}$.
Now define :
\begin{align} \label{jstar}
j^{*}_i \triangleq \argmax\limits_{j \in J^{*}_i} \sum_{k=1}^{K} b_k \frac{R_{j, k}}{a_j a_k}
\end{align}
We obtain,
\begin{align*}
C_{i, i} \xrightarrow{p} \sum_{l = 1}^{K} b_l \sum_{k=1}^{K} b_k \frac{R_{l, k}}{a_l a_k} \leq \sum_{k=1}^{K} b_k \frac{R_{j^{*}_i, k}}{a_{j^{*}_i} a_k}
\end{align*}
As a result, $C_{i,i} - C_{i,j^{*}_i} \xrightarrow{p}  \sum_{l = 1}^{K} b_l \sum_{k=1}^{K} b_k \frac{R_{l, k}}{a_l a_k} - \sum_{k=1}^{K} b_k \frac{R_{j^{*}_i, k}}{a_{j^{*}_i} a_k} \leq 0$ and the proof is complete.
\end{proof}
\subsection{Proof of Theorem 1}
\label{sec:thm1}
Now we can prove the Theorem 1 in Section 5. To summarize the notations, let  $\beta_{\wedge}$ be a strictly positive lower bound on non-zero elements of ${\bm \beta}$, $\lambda_{\wedge}$ be the minimum eigenvalue of ${\mathbf R}$, and $a_{\wedge}$ be the minimum component of mean vector $\mathbf{a}$. Further we define $\eta = \min\limits_{1 \leq i \leq W} {\bm \beta}_i {\mathbf a}$ and $\zeta \triangleq \min\limits_{1\leq i \neq j \leq K}\frac{R_{i,i}}{a_i a_i} - \frac{R_{i,j}}{a_i a_j} >0$.

\noindent
{\bf Theorem 1 (in Section 5.1)}

{\it For parameter choices $d = {\lambda_{\wedge} \beta_{\wedge}^2}$ and $\gamma = \zeta {a_{\wedge} \beta_{\wedge}}$ the DDP algorithm is consistent as \hbox{$M \rightarrow \infty$.} Specifically, true novel and non-novel words are asymptotically declared as novel and non-novel, respectively. 
Furthermore, for 
\begin{equation*}
M \geq \frac{C_1 \left(\log W + \log \left( \frac{1}{\delta_1} \right)\right)}{\beta_{\wedge}^2 \eta^8 \min(\lambda_{\wedge}^2 \beta_{\wedge}^2, \zeta^2 a_{\wedge}^2) }
\end{equation*}
where $C_1$ is a constant, Algorithm \ref{AWAlg} finds all novel words
without any outlier with probability at least $1 - \delta_1$, where $\eta = \min\limits_{1 \leq i \leq W} {\bm \beta}_i {\mathbf a}$. \\ }
\begin{proof}[Proof of Theorem 1]
Suppose that $i$ is a novel word. The probability that $i$ is not detected by the DDP Algorithm can be written as
\begin{align*}
\Pr(J_i & \nsubseteq J^{*}_i \text{ or } (J_i \subseteq J^{*}_i \\ & ~~~~~~~ \text{ and } \exists j \in J_i : C_{i,i} - C_{i,j} \leq \gamma/2)) \\
& \leq \Pr(J_i \nsubseteq J^{*}_i) \\
& ~ + \Pr((J_i \subseteq J^{*}_i \text{ and } \exists j \in J_i : C_{i,i} - C_{i,j} \leq \gamma/2)) \\
& \leq \Pr(J_i \nsubseteq J^{*}_i) + \Pr(\exists j \in J^{*}_i : C_{i,i} - C_{i,j} \leq \gamma/2 ) \\
& \leq \Pr(J_i \nsubseteq J^{*}_i) + \sum\limits_{j \in J^{*}_i} \Pr(C_{i,i} - C_{i,j} \leq \gamma/2)
\end{align*}
The first and second term in the right hand side converge to zero according to Lemma \ref{JLemma} and \ref{CLemma}, respectively. Hence, this probability of failure in detecting $i$ as a novel word converges to zero.
On the other hand, the probability of claiming a non-novel word as a novel word by the Algorithm DDP can be written as :
\begin{align*}
\Pr(J_i & \nsupseteq J^{*}_i \text{ or } (J_i \supseteq J^{*}_i \\ & ~~~~~~~ \text{ and } \forall j \in J_i : C_{i,i} - C_{i,j} \geq \gamma/2)) \\
& \leq \Pr(J_i \nsupseteq J^{*}_i) \\
& ~ + \Pr((J_i \supseteq J^{*}_i \text{ and } \forall j \in J_i : C_{i,i} - C_{i,j} \geq \gamma/2)) \\
& \leq \Pr(J_i \nsupseteq J^{*}_i) + \Pr(\forall j \in J^{*}_i : C_{i,i} - C_{i,j} \geq \gamma/2 ) \\
& \leq \Pr(J_i \nsupseteq J^{*}_i) + \Pr(C_{i,i} - C_{i,j^{*}_i} \geq \gamma/2)
\end{align*}
where $j^{*}_i$ was defined in equation \eqref{jstar}. We have shown in Lemma \ref{JLemma} and \ref{CLemma} that both of the probabilities in the right hand side converge to zero. This concludes the consistency of the algorithm.
Combining the convergence rates given in the Corollaries \ref{Cor:cii-2cij+cjj}, \ref{Cor:cii-cjj} and \ref{Cor:Jrate}, the probability that the DDP Algorithm fails in finding all novel words without any outlier will be bounded by 
$W e_1 \exp(-M e_2 \min(d^2, \gamma^2) \eta^8)$, where $e_1$ and $e_2$ are constants and $d$ and $\gamma$ are defined in the Theorem.
\end{proof}

\subsection{Proof of Theorem 2}
\label{sec:thm2}
\noindent
{\bf Theorem 2 (in Section 5.2)}
{\it For $d = {\lambda_{\wedge} \beta_{\wedge}^2}$, given all true novel words as the input, the clustering algorithm, \hbox{Algorithm \ref{ACAlg}} (ClusterNovelWords) asymptotically (as $M\rightarrow \infty$ recovers $K$ novel word indices of different types, namely, the support of the corresponding ${\bm \beta}$ rows are different for any two retrieved indices. Furthermore, if 
\begin{equation*}
M \geq \frac{C_2 \left( \log W + \log \left( \frac{1}{\delta_2} \right)\right) } {\eta^8 \lambda_{\wedge}^2 \beta_{\wedge}^4}
\end{equation*}
then Algorithm \ref{ACAlg} clusters all novel words correctly with probability at least $1 - \delta_2$. \\}
\begin{proof}[Proof of Theorem 2]
The statement follows using ${|\mathcal{I}|} \choose  {2}$ number of union bounds on the probability that $C_{i,i} - 2C_{i,j} + C_{j,j}$ is outside an interval of the length $d/2$ centered
around the value it converges to. The convergence rate of the related random variables are given in Lemma \ref{CijLemma}.
Hence the probability that the clustering algorithm fails in clustering all the novel words truly is bounded by $e_1 W^2 \exp(-M e_2 \eta^8 d^2)$, where $e_1$ and $e_2$ are constants and $d$ is defined in the theorem.
\end{proof}

\subsection{Proof of Theorem 3}
\noindent
{\bf Theorem 3 (in Section 5.3)}
{\it Suppose that Algorithm \ref{TEAlg} outputs $\widehat{\bm \beta}$ given the indices of $K$ distinct novel words. Then, $\widehat{\bm \beta} \xrightarrow{p} {\bm \beta}$. Specifically, if 
\begin{equation*}
M \geq \frac{C_3 W^4(\log(W) + \log(K) + \log(1/\delta_3))}{\lambda_{\wedge}^2 \eta^8 \epsilon^4 a_{\wedge}^8}
\end{equation*}
then for all $i$ and $j$, $\widehat{\beta}_{i,j}$ will be $\epsilon$ close to $\beta_{i,j}$ with probability at least $1 - \delta_3$, with $\epsilon < 1$, $C_3$ being a constant, $a_{\wedge} = \min_{i} a_i$ and $\eta = \min\limits_{1 \leq i \leq W} {\bm \beta}_i {\mathbf a}$.\\ }
 
\begin{proof}
We reorder the rows so that ${\mathbf Y}$ and ${\mathbf Y}^{\prime}$ be the first $K$ rows of ${\mathbf X}$ and ${\mathbf X}^{\prime}$, respectively.  
For the optimization objective function in Algorithm \ref{TEAlg}, if $i < K$, $\mathbf{b} = \mathbf{e}_i$ achieves the minimum, where all components of $\mathbf{e}_i$ are zero, except its $i^{\text{th}}$ component, which is one.
Now fix $i$, we denote the objective function as ${Q}_M (\mathbf{b}) = M (\widetilde{\mathbf X}_i - {\mathbf b} {\mathbf Y}) (\widetilde{\mathbf X}^{\prime}_i - {\mathbf b} {\mathbf Y}^{\prime})^{\top}$, and denote the optimal solution as $\mathbf{b}_{M}^{*}$.  
By the previous lemmas, $Q_M({\mathbf b}) \xrightarrow{p} \bar{Q}({\mathbf b}) = {\mathbf b} {\mathbf D} {\mathbf R} {\mathbf D} {\mathbf b}^{\top} - 2 {\mathbf b} {\mathbf D} {\mathbf R} \frac{{\bm \beta}_i^{\top}}{{\bm \beta}_i {\mathbf a}} + \frac{{\bm \beta}_i}{{\bm \beta}_i {\mathbf a}} {\mathbf R} \frac{{\bm \beta}_i^{\top}}{{\bm \beta}_i {\mathbf a}}$, where ${\mathbf D} = \diag({\mathbf a})^{-1}$. Note that if ${\mathbf R}$ is positive definite, $\bar{Q}$ is uniquely minimized at ${\mathbf b}^{*} = \frac{{\bm \beta}_i}{{\bm \beta}_i {\mathbf a}} {\mathbf D}^{-1} $.
Following the notation in Lemma \ref{CijLemma} and its proof,  
\begin{equation*}
\Pr\left(\left|C_{i,j} - E_{i,j} \right| \geq \epsilon \right) \leq 8 \exp(-M \epsilon^2\eta^8 /32)  
\end{equation*}
where $C_{i,j} = M \widetilde{\mathbf{X}}_i {\widetilde{\mathbf X}_j}^{\top}$, $E_{i,j}= \frac{{\bm \beta}_i}{{\bm \beta}_i {\mathbf a}} {\mathbf R} \frac{{\bm \beta}_j^{\top}}{{\bm \beta}_j {\mathbf a}}$, and $\eta = \min\limits_{1 \leq i \leq W} {\bm \beta}_i {\mathbf a}$. Note that ${\mathbf b} \in \mathcal{B} = \{ {\mathbf b} :  0\leq b_k \leq 1, \sum b_k =1\}$. Therefore, $\forall s, r \in \{ 1, \ldots, K, i \} : |C_{s,r} - E_{s,r}| \leq \epsilon$ implies that
\begin{align*}
 \forall {\mathbf b} \in \mathcal{B} : & |Q_M(\mathbf{b})-\bar{Q}(\mathbf{b})| \leq |C_{i,i}-E_{i,i}| \\
 & + \sum\limits_{k=1}^{K}b_k|C_{k,i} -E_{k,i} | + \sum\limits_{k=1}^{K}b_k|C_{i,k} -E_{i,k} |  \\
 & + \sum\limits_{r=1}^{K}\sum\limits_{s=1}^{K}b_r b_s|C_{r,s} -E_{r,s} | \\
& ~ \leq 4\epsilon  
\end{align*}
Hence
\begin{multline} \label{UnionC}
\Pr \left( \exists \mathbf{b} \in \mathcal{B} : |Q_M(\mathbf{b})-\bar{Q}(\mathbf{b})| \geq 4\epsilon\right) \\ \leq \Pr \left( \exists i,j \in \{ 1, \ldots, K, i \}: |C_{i,j} - E_{i,j} | \geq \epsilon \right)
\end{multline}

Using $(K+1)^2$ union bounds for the right hand side of the equation \ref{UnionC}, we obtain the following equation with $c_1$ and $c_2$ being two constants:
\begin{multline} \label{UnifQ}
\Pr \left(\exists \mathbf{b} \in \mathcal{B} : |Q_M(\mathbf{b})-\bar{Q}(\mathbf{b})| \geq \epsilon \right) \\ \leq c_1 (K+1)^2 \exp(-c_2 M \epsilon^2\eta^8)
\end{multline}
Now we show that $\mathbf{b}_{M}^{*}$ converge to $\mathbf{b}^{*}$. Note that $\mathbf{b}^{*}$ is the unique minimizer of the strictly convex function $\bar{Q}(\mathbf{b})$. The strict convexity of $\bar{Q}$ is followed by the fact that ${\mathbf R}$ is assumed to be positive definite. Therefore, we have, $\forall \epsilon_0 > 0$, $\exists \delta > 0$ such that $\Vert \mathbf{b} -\mathbf{b}^{*} \Vert \geq \epsilon_0 \Rightarrow \bar{Q}(\mathbf{b})- \bar{Q}(\mathbf{b}^{*})\geq \delta$. Hence, 
\begin{align*}
& \Pr(\Vert \mathbf{b}_{M}^{*} -\mathbf{b}^{*} \Vert \geq \epsilon_0) \\
\leq & \Pr(\bar{Q}(\mathbf{b}_{M}^{*})- \bar{Q}(\mathbf{b}^{*})\geq \delta) \\
{\leq} & \Pr(\bar{Q}(\mathbf{b}_{M}^{*})-Q_M (\mathbf{b}_{M}^{*})+ Q_M (\mathbf{b}_{M}^{*})- Q_M (\mathbf{b}^{*})+ \\
& Q_M (\mathbf{b}^{*}) - \bar{Q}(\mathbf{b}^{*})\geq \delta)\\
\stackrel{(i)}{\leq} & \Pr(\bar{Q}(\mathbf{b}_{M}^{*})-Q_M (\mathbf{b}_{M}^{*})+  Q_M (\mathbf{b}^{*}) - \bar{Q}(\mathbf{b}^{*})\geq \delta)\\
\stackrel{(ii)}{\leq} & \Pr(2\sup\limits_{{\mathbf b} \in \mathcal{B}} |Q_M(\mathbf{b})-\bar{Q}(\mathbf{b}) |\geq \delta) \\
\leq & \Pr(\exists {{\mathbf b} \in \mathcal{B}} : |Q_M(\mathbf{b})-\bar{Q}(\mathbf{b}) | \geq \delta/2) \\
\stackrel{(iii)}{\leq} & c_1 (K+1)^2 \exp\left(-\frac{c_2}{4} \delta^2 \eta^8 M \right)
\end{align*}

where $(i)$ follows because $Q_M (\mathbf{b}_{M}^{*})- Q_M (\mathbf{b}^{*}) \leq 0$ by definition, $(ii)$ holds
considering the fact that ${\mathbf b}, {\mathbf b}^{*} \in \mathcal{B}$ and $(iii)$ follows as a result of equation \ref{UnifQ}.
For the $\epsilon_0$ and $\delta$ relationship, let $\mathbf{y}=\mathbf{b}-\mathbf{b}^{*}$, 
\begin{align*}
&\bar{Q}(\mathbf{b})-\bar{Q}(\mathbf{b}^{*})=\mathbf{y}(\mathbf{DRD})\mathbf{y}^{\top} \geq \Vert \mathbf{y}\Vert^{2} \lambda_{*}
\end{align*}
where $\lambda_{*} > 0$ is the minimum eigenvalue of ${\mathbf D} {\mathbf R} {\mathbf D}$. Note that $\lambda_{*} \geq (\min\limits_{1 \leq j \leq K} a_j^{-1})^2 \lambda_{\wedge}$, where $\lambda_{\wedge} > 0$ is a lower bound on the minimum eigenvalues of ${\mathbf R}$. But $0 < a_j \leq 1$, hence $\lambda_{*} \geq \lambda_{\wedge}$. Hence we could set $\delta = \lambda_{\wedge} \epsilon_0^2$. In sum, we could obtain 
\begin{equation*}
\Pr(\Vert \mathbf{b}_{M}^{*} -\mathbf{b}^{*} \Vert \geq \epsilon_0) \leq c_1 (K+1)^2 \exp(-c_2^{\prime} M \epsilon_0^4 \lambda_{\wedge}^2\eta^8 )
\end{equation*}
for the constants $c1$ and $c_2^{\prime}$.
Or simply $\mathbf{b}_{M}^{*} \xrightarrow{p} \mathbf{b}^{*}$.
Note that before column normalization, we let $\widehat{\bm\beta}_i= (\frac{1}{M} \mathbf{X}_i \mathbf{1}) (\mathbf{b}_{M}^{*})$. The convergence of the first term (to $\bm\beta_i \mathbf{a}$), as we have already verified in Lemma \ref{CijLemma}, and using Slutsky's theorem, we get $\widehat{\bm\beta}_i \xrightarrow{p} \bm\beta_i \mathbf{D}^{-1}$. Hence after column normalization, which involves convergence of $W$ random variables, by Slutsky's theorem again we can prove that $\widehat{\bm\beta}_{i}\xrightarrow{p} \bm\beta_{i}$ for any $1\leq i\leq W$. This concludes our proof and directly implies the convergence in the Mean-Square sense. 
To show the exact convergence rate, we apply the Proposition \ref{prop:tech}. For $\widehat{\bm\beta}_i$ before column normalization, note that $\frac{1}{M} \mathbf{X}_i \mathbf{1}$ converges to ${\bm\beta}_i\mathbf{a}$ with error probability $2\exp\left( -2\epsilon^2 M \right)$, we obtain
\begin{multline*}
\Pr(|\widehat{\beta}_{i,j} - {\beta}_{i,j}a_j| \geq \epsilon) \leq e_1 (K+1)^2 \exp(-e_2\lambda_{\wedge}^2 \eta^8 M \epsilon^4) \\ + e_3 \exp\left( -2 e_4 \epsilon^2  M \right) 
\end{multline*}
for constants $e_1, \ldots, e_4$.
On the other hand, the column normalization factors can be obtained by $\mathbf{1}^{\top}\widehat{\bm\beta}$. Denote normalization factor of the $j^{\text{th}}$ column by $P_j = \sum_{i = 1}^{W} \widehat{\beta}_{i,j}$ and hence  $\Pr(|P_j-{a}_j|\geq\epsilon)\leq e_1 W (K+1)^2 \exp(-e_2\lambda_{\wedge}^2 \eta^8 M \epsilon^4/W^4) + e_3 W \exp \left( -e_4 \epsilon^2 M/ W^2 \right)$. Now using the Proposition \ref{prop:tech} again we obtain that after column normalization,
\begin{multline*}
\Pr\left(\left|\frac{\widehat{\beta}_{i,j}}{\sum_{k = 1}^{W} \widehat{\beta}_{k, j}} - {\beta}_{i,j} \right| \geq \epsilon\right) \\ \leq f_1 (K+1)^2 \exp(-f_2\lambda_{\wedge}^2 \eta^8 M \epsilon^4 a_{\wedge}^4) \\ + f_3 \exp\left( -2 f_4 \epsilon^2  M a_{\wedge}^2 \right)  \\ + f_5 W (K+1)^2 \exp(-f_6\lambda_{\wedge}^2 \eta^8 M \epsilon^4 a_{\wedge}^8 /W^4) \\ + f_7 W \exp \left( -f_8 \epsilon^2 M a_{\wedge}^4 / W^2 \right) 
\end{multline*}
for constants $f_1, \ldots, f_8$ and $a_{\wedge}$ being the minimum value of $a_i$'s. Assuming $\epsilon < 1$, we can simplify the previous expression to obtain 
\begin{multline*}
\Pr\left(\left|\frac{\widehat{\beta}_{i,j}}{\sum_{k = 1}^{W} \widehat{\beta}_{k, j}} - {\beta}_{i,j} \right| \geq \epsilon\right) \\ \leq b_1 W (K+1)^2 \exp(-b_2\lambda_{\wedge}^2 \eta^8 M \epsilon^4 a_{\wedge}^8 /W^4)
\end{multline*}
for constants $b_1$ and $b_2$. Finally, to get the error probability of the whole matrix, we can use $WK$ union bounds. Hence we have :
\begin{multline*}
\Pr\left(\exists i,j : \left|\frac{\widehat{\beta}_{i,j}}{\sum_{k = 1}^{W} \widehat{\beta}_{k, j}} - {\beta}_{i,j} \right| \geq \epsilon\right) \\ \leq b_1 W^2 K (K+1)^2 \exp(-b_2\lambda_{\wedge}^2 \eta^8 M \epsilon^4 a_{\wedge}^8 /W^4)
\end{multline*}
Therefore, the sample complexity of $\epsilon$-close estimation of $\beta_{i,j}$ by the Algorithm \ref{TEAlg} with probability at least $1 - \delta_3$ will be given by:
\begin{equation*}
M \geq \frac{C^{\prime}W^4(\log(W) + \log(K) + \log(1/\delta_3))}{\lambda_{\wedge}^2 \eta^8 \epsilon^4 a_{\wedge}^8}
\end{equation*}
\end{proof}

\section{Experiment results}
\subsection{Sample Topics extracted on NIPS dataset}
%
\input{table3}
%

\subsection{Sample Topics extracted on New York Times dataset}

\input{table4}



\end{document}

%% file: figure2.tex
\begin{figure}[htb]
\begin{minipage}[b]{.90\linewidth}
  \centering
  \centerline{\includegraphics[width=8.0cm]{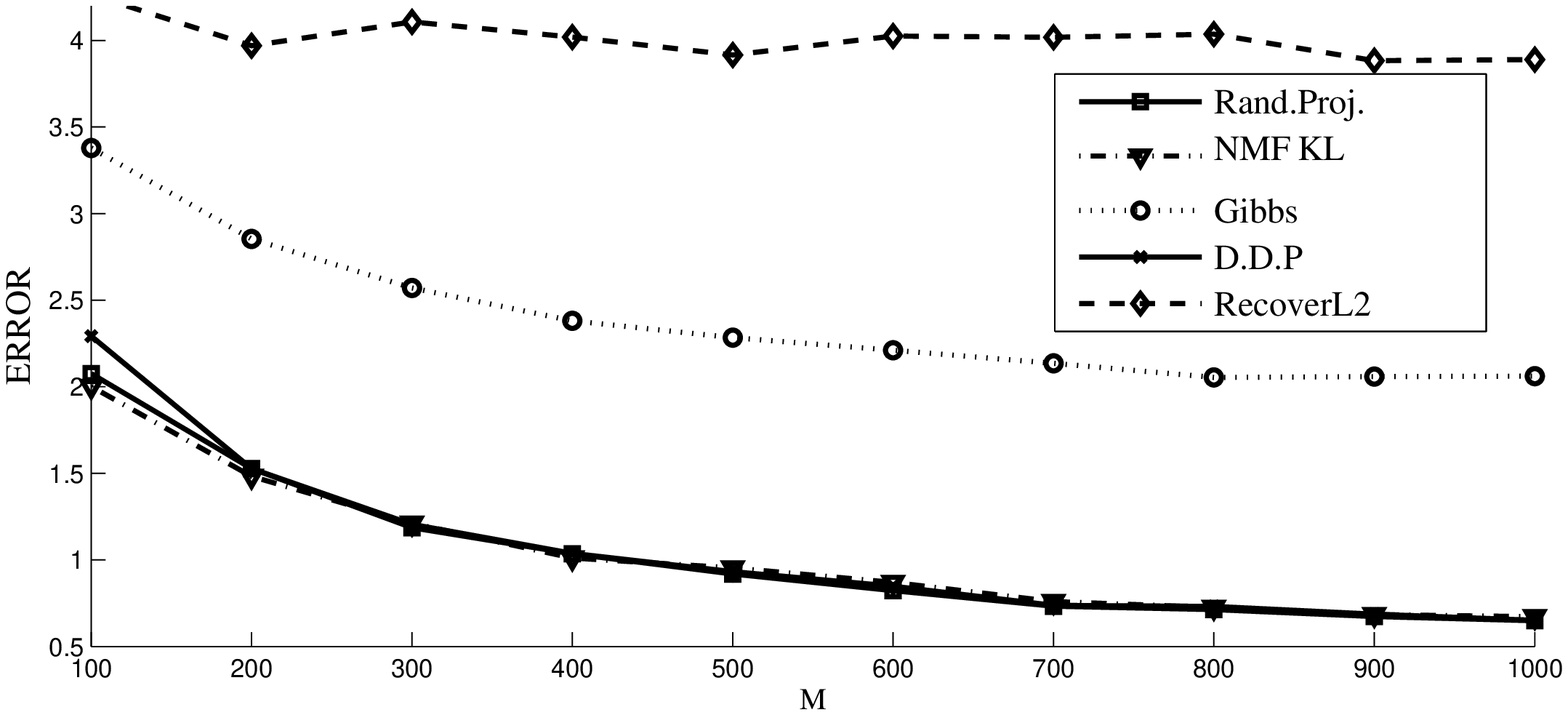}}
\end{minipage}
\vfill
\begin{minipage}[b]{0.90\linewidth}
  \centering
  \centerline{\includegraphics[width=8.0cm]{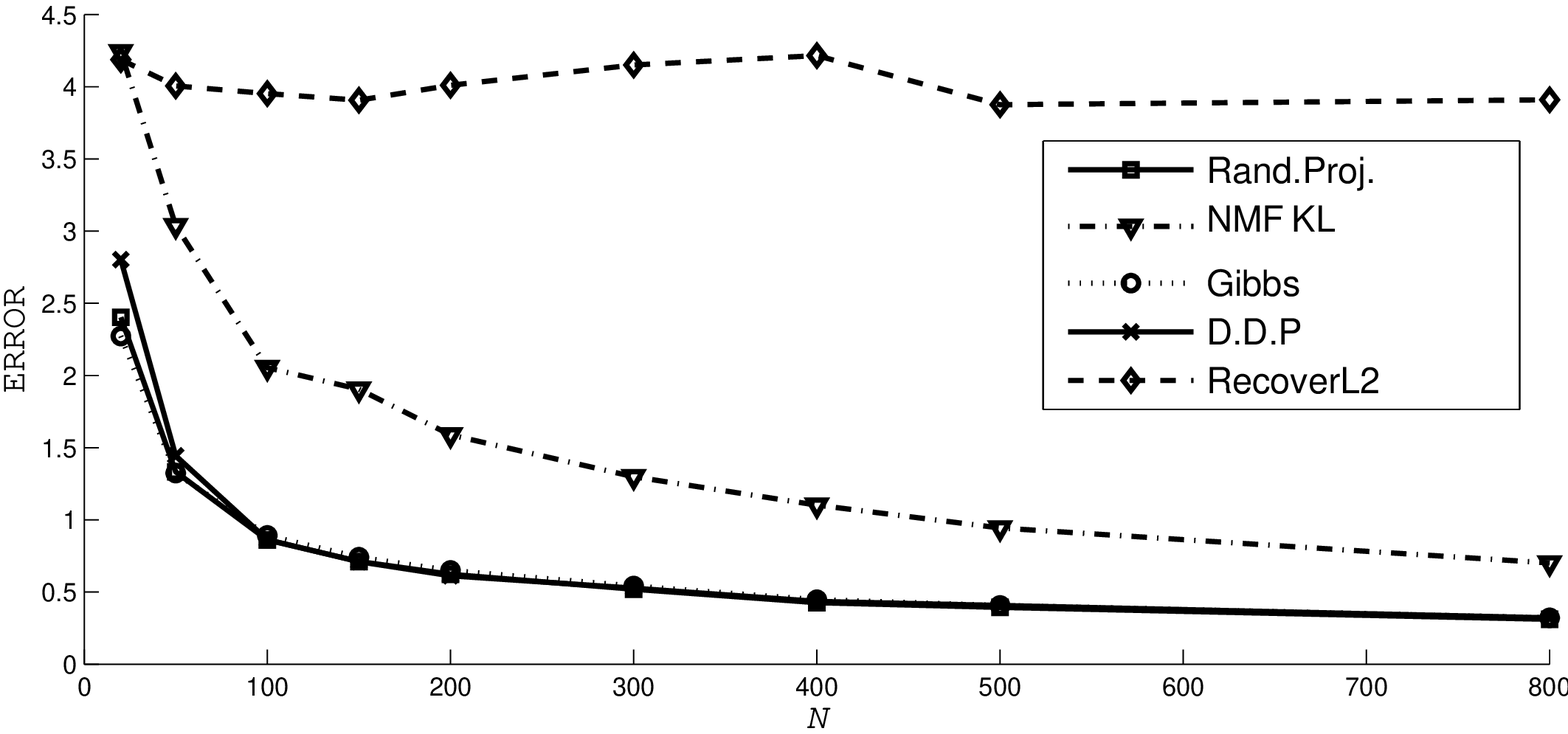}}
\end{minipage}
\caption{Error of estimated topic matrix in $\ell_1$ norm. Upper: $W
  = 500,\rho = 0.2, N = 100, K = 5$; Lower: $W = 500,\rho = 0.2, M = 500, K = 10$. Top and Bottom plots depict error with varying documents $M$ (for fixed $N$) and varying words $N$ (for fixed $M$) respectively. RP \& DDP show consistently better performance.}
\label{fig:synthetic}
\vspace{-0.2in}
\end{figure}

%% file: figure3.tex
\begin{figure}[!htb]
\centering
\begin{tabular}{@{\hskip 0.05cm}m{1.0cm}@{\hskip 0.05cm}m{1.0cm}@{\hskip 0.05cm}m{1.0cm}@{\hskip 0.05cm}m{1.0cm}@{\hskip 0.05cm}m{1.0cm}@{\hskip 0.05cm}m{1.0cm}@{\hskip 0.05cm}}
\multicolumn{2}{c}{(a)}&\multicolumn{2}{|c|}{(b)}&\multicolumn{2}{c}{(c)}\\

\parbox[c]{1cm}{\vspace*{0.1cm}\includegraphics[width=1.0cm]{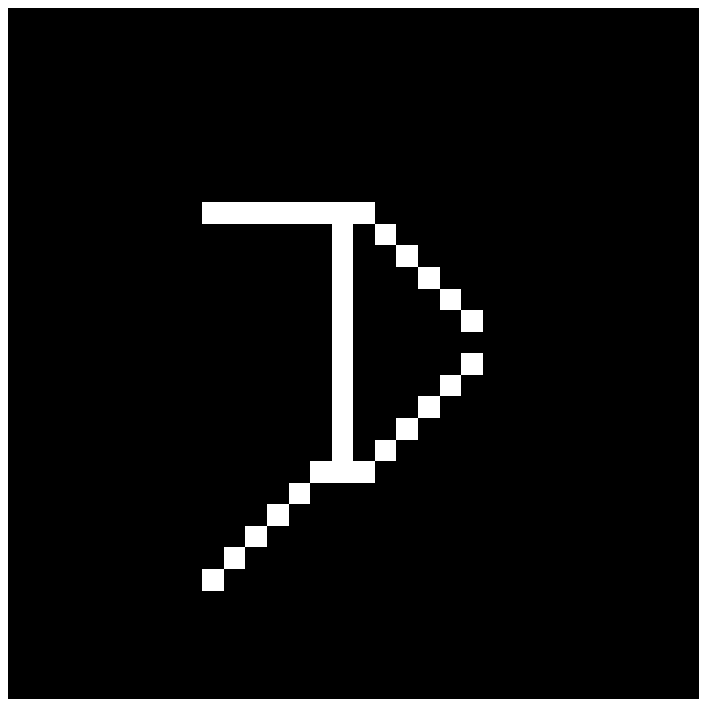}
\vspace*{-0.3cm}}
&
\parbox[c]{1cm}{\vspace*{0.1cm}\includegraphics[width=1.0cm]{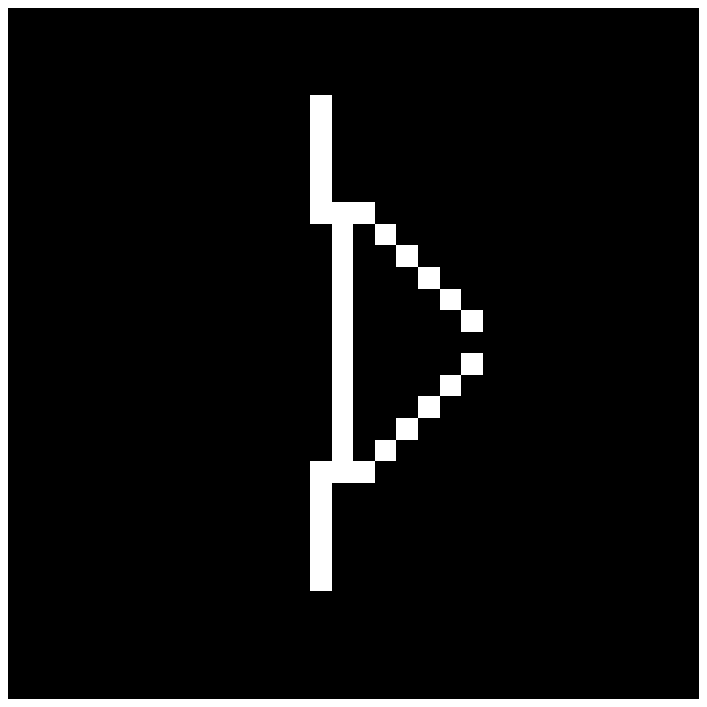}
\vspace*{-0.3cm}}
&
\parbox[c]{1cm}{\vspace*{0.1cm}\includegraphics[width=1.0cm]{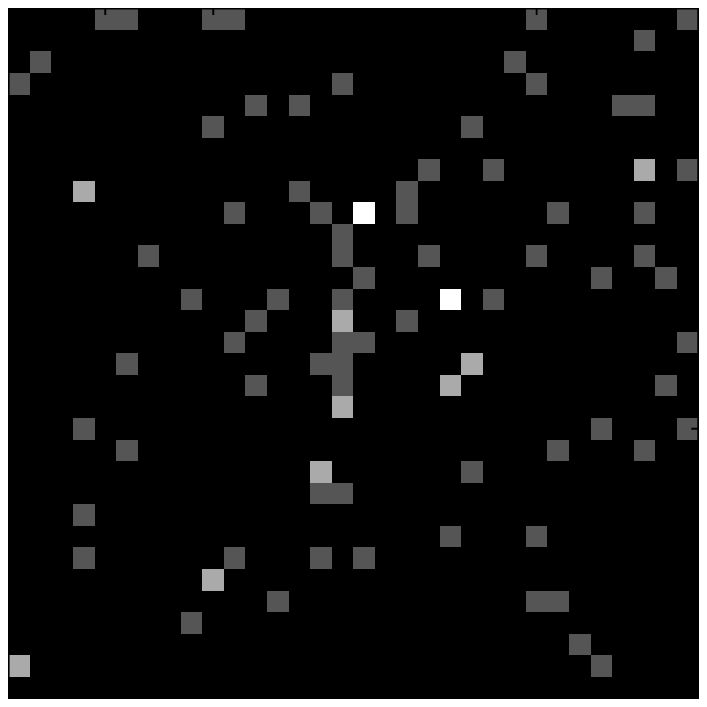}
\vspace*{-0.3cm}}
&
\parbox[c]{1cm}{\vspace*{0.1cm}\includegraphics[width=1.0cm]{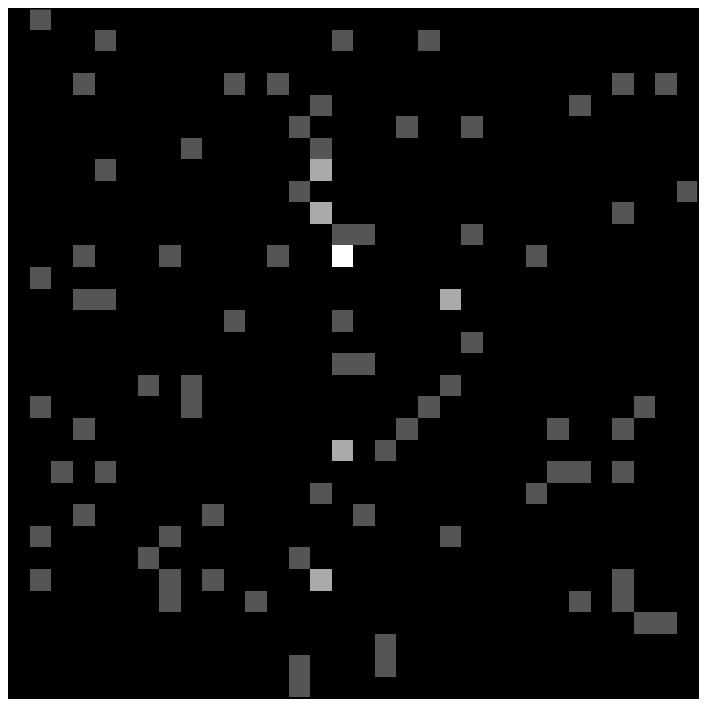}
\vspace*{-0.3cm}}
&
\parbox[c]{1cm}{\vspace*{0.1cm}\includegraphics[width=1.0cm]{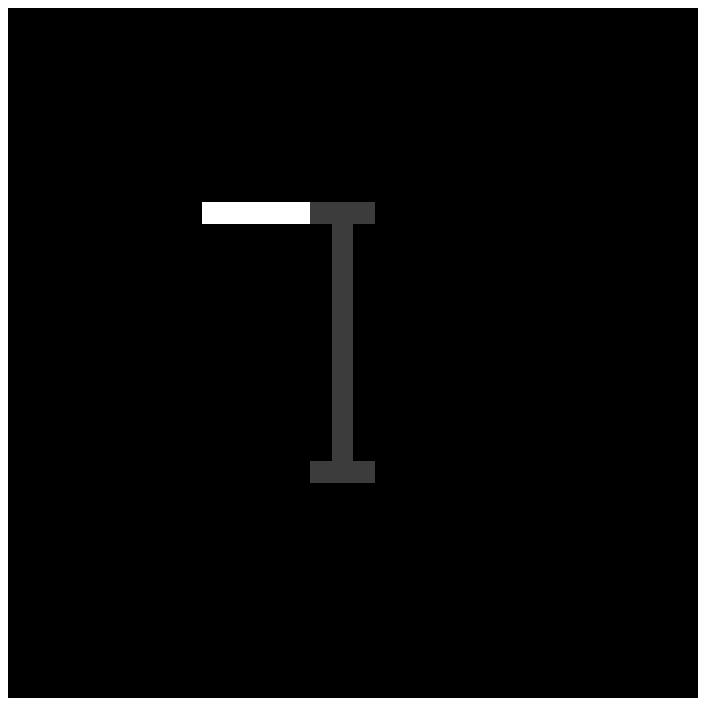}
\vspace*{-0.3cm}}
&
\parbox[c]{1cm}{\vspace*{0.1cm}\includegraphics[width=1.0cm]{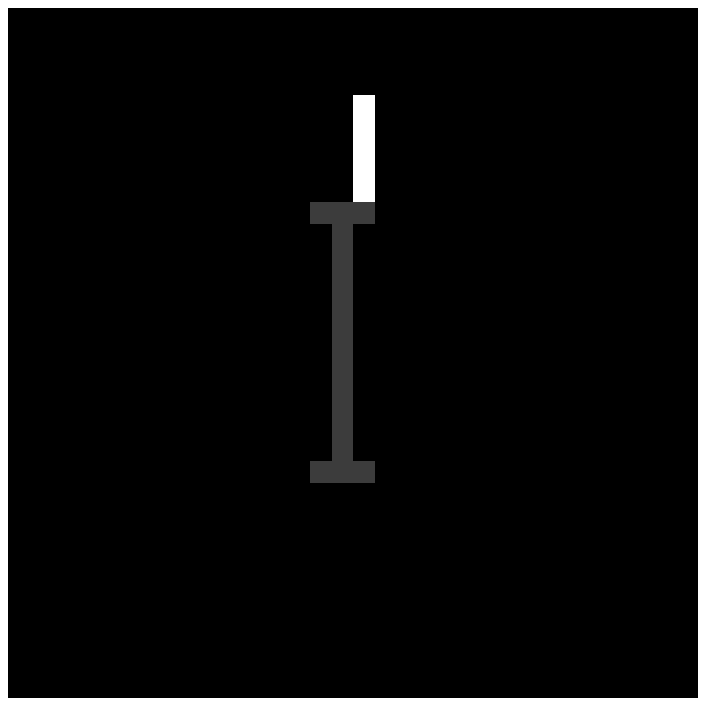}
\vspace*{-0.3cm}}
\\

\parbox[c]{1cm}{\vspace*{0.1cm}\includegraphics[width=1.0cm]{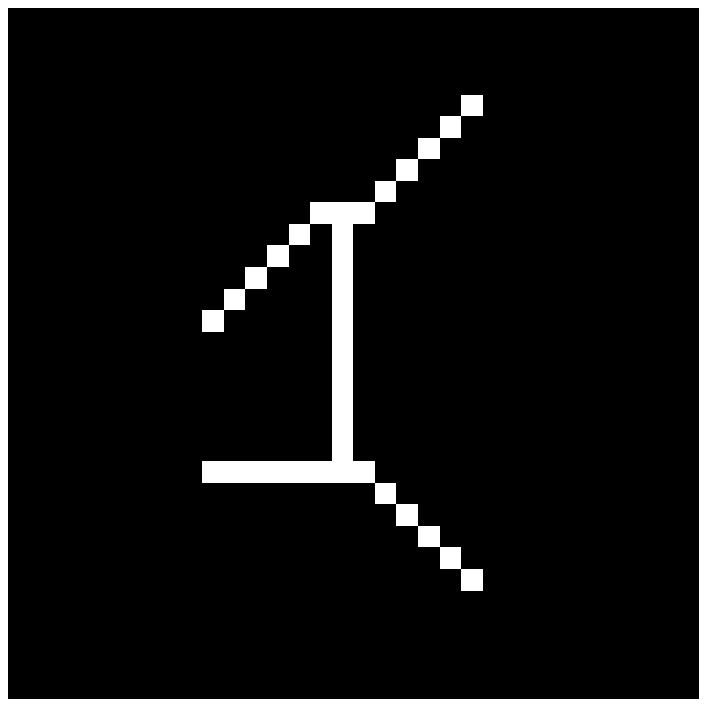}
\vspace*{-0.3cm}}
&
\parbox[c]{1cm}{\vspace*{0.1cm}\includegraphics[width=1.0cm]{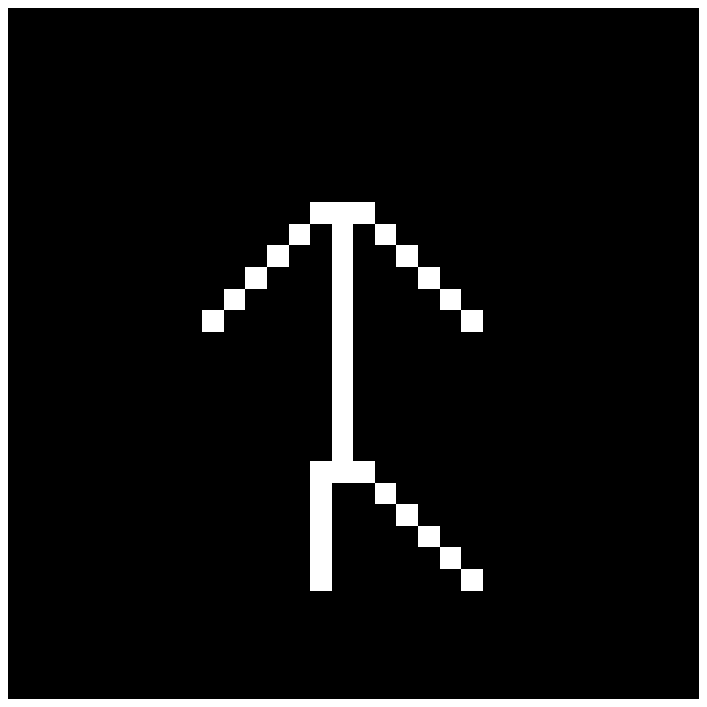}
\vspace*{-0.3cm}}
&
\parbox[c]{1cm}{\vspace*{0.1cm}\includegraphics[width=1.0cm]{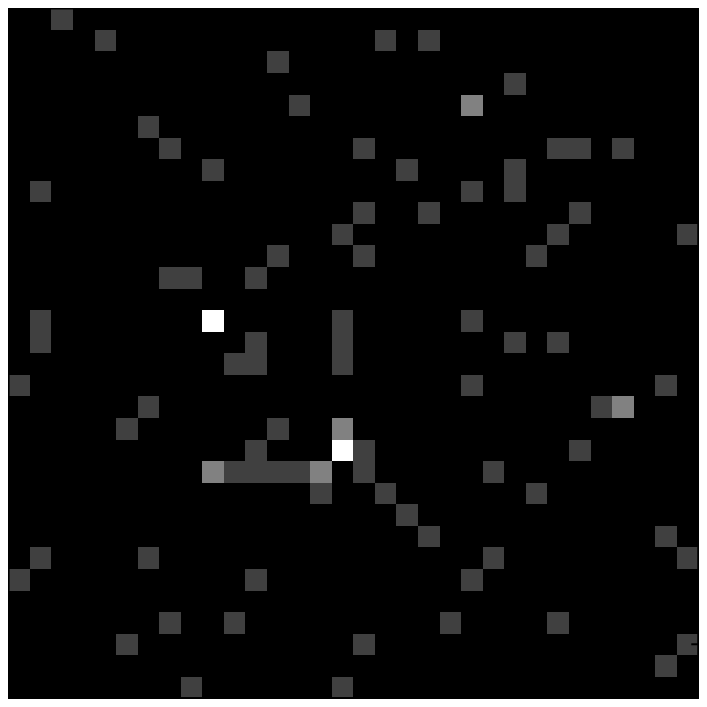}
\vspace*{-0.3cm}}
&
\parbox[c]{1cm}{\vspace*{0.1cm}\includegraphics[width=1.0cm]{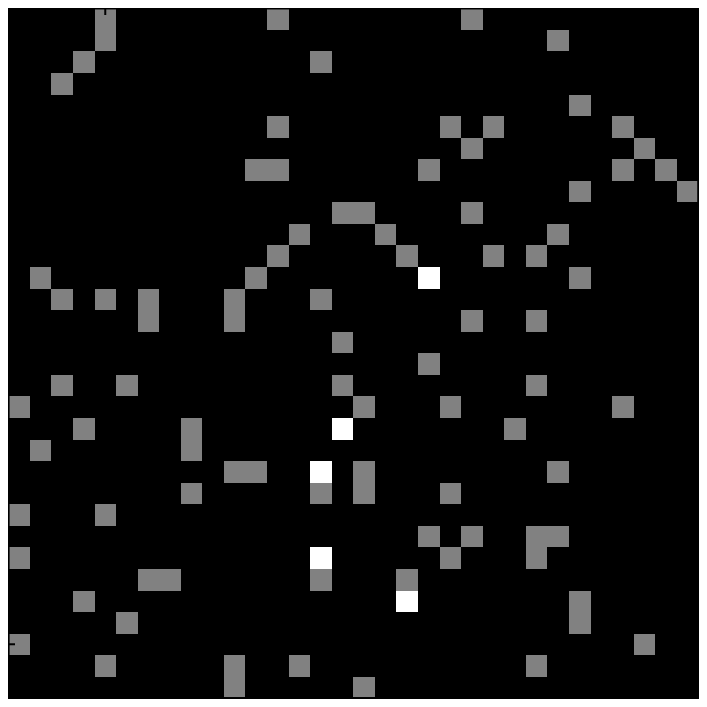}
\vspace*{-0.3cm}}
&
\parbox[c]{1cm}{\vspace*{0.1cm}\includegraphics[width=1.0cm]{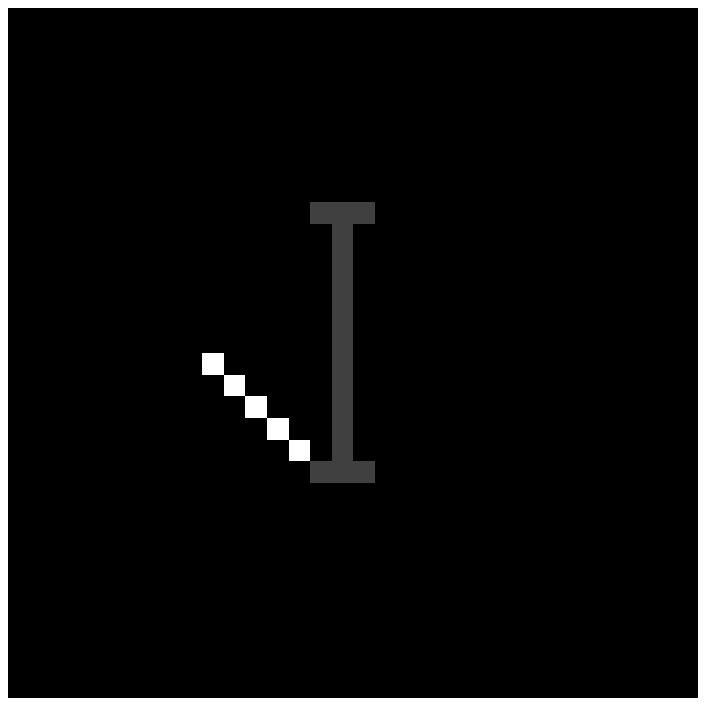}
\vspace*{-0.3cm}}
&
\parbox[c]{1cm}{\vspace*{0.1cm}\includegraphics[width=1.0cm]{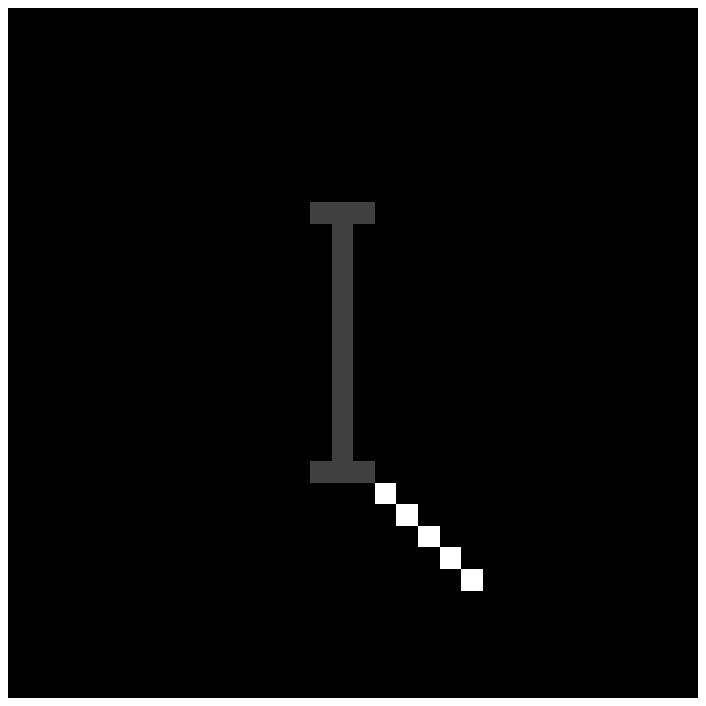}
\vspace*{-0.3cm}}
\\
\end{tabular}
%
%
%
%
%
%
\caption{(a) Example ``clean'' images  in Swimmer
  dataset; (b) Corresponding images with sampling ``noise'' ; (c) Examples of ideal topics.}
\label{fig:sample_swimmer}
\end{figure}
%

%% file: figure5.tex
\begin{figure}[!htb]
\centering
\begin{tabular}{m{0.5cm}@{\hskip 0.05cm}m{1.1cm}@{\hskip 0.05cm}m{1.1cm}@{\hskip 0.05cm}m{1.1cm}@{\hskip 0.05cm}m{1.1cm}@{\hskip 0.05cm}m{1.1cm}@{\hskip 0.05cm}m{1.1cm}@{\hskip 0.05cm}}
%
%
%
%
a) &
%
\parbox[c]{1cm}{\includegraphics[width=1.0cm]{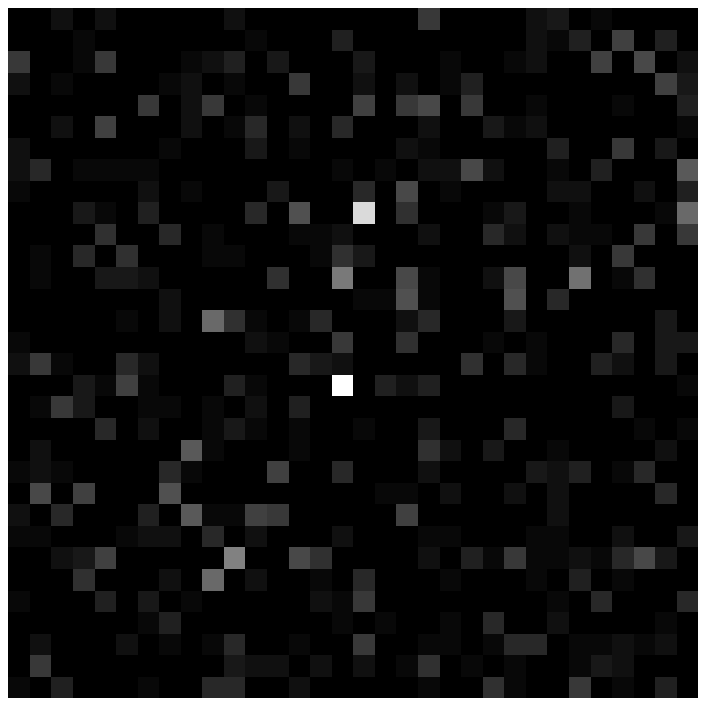}\vspace*{0.0cm}}
&
%
\parbox[c]{1cm}{\includegraphics[width=1.0cm]{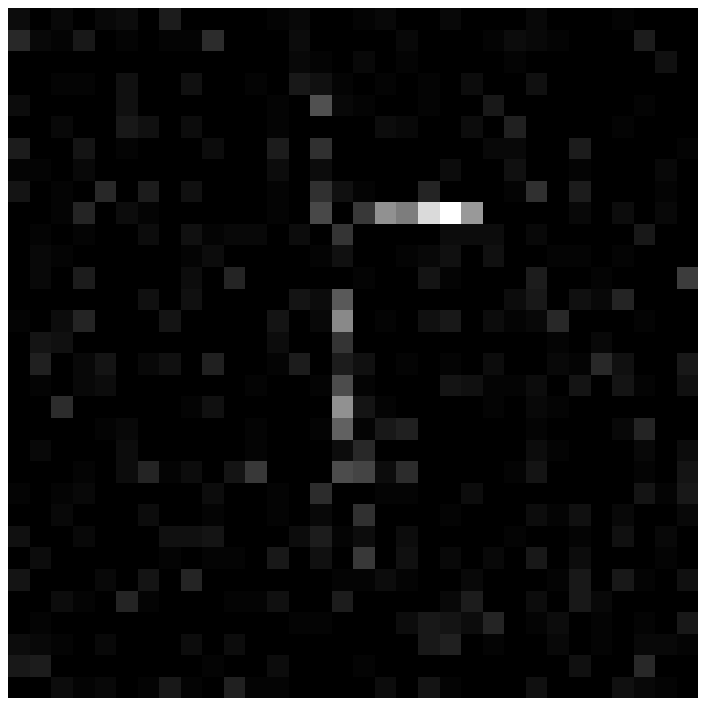}\vspace*{0.0cm}}
&
%
\parbox[c]{1cm}{\includegraphics[width=1.0cm]{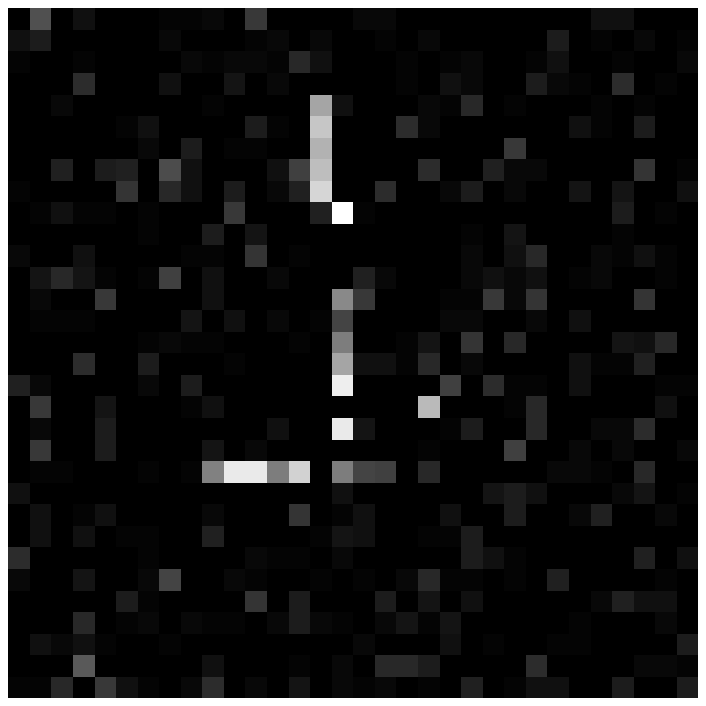}\vspace*{0.0cm}}
&
%
\parbox[c]{1cm}{\includegraphics[width=1.0cm]{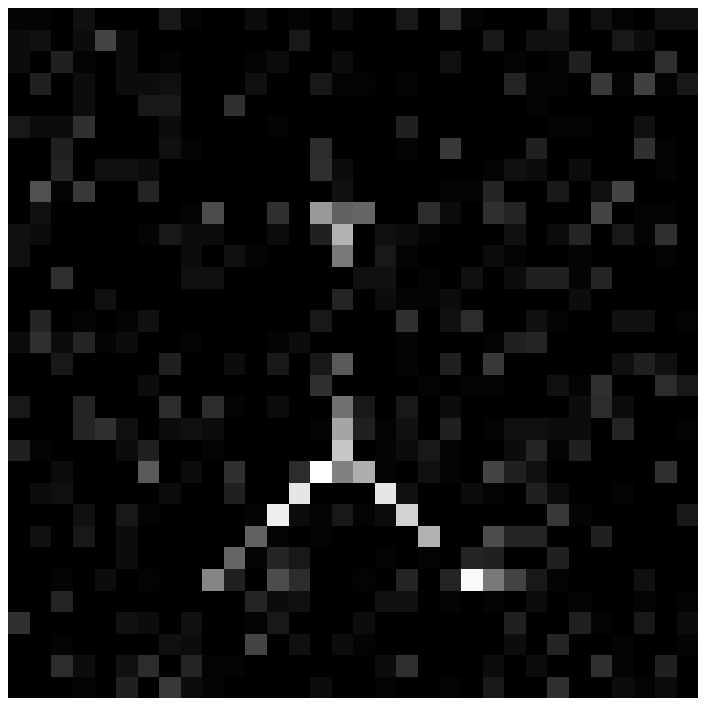}\vspace*{0.0cm}}
&
%
\parbox[c]{1cm}{\includegraphics[width=1.0cm]{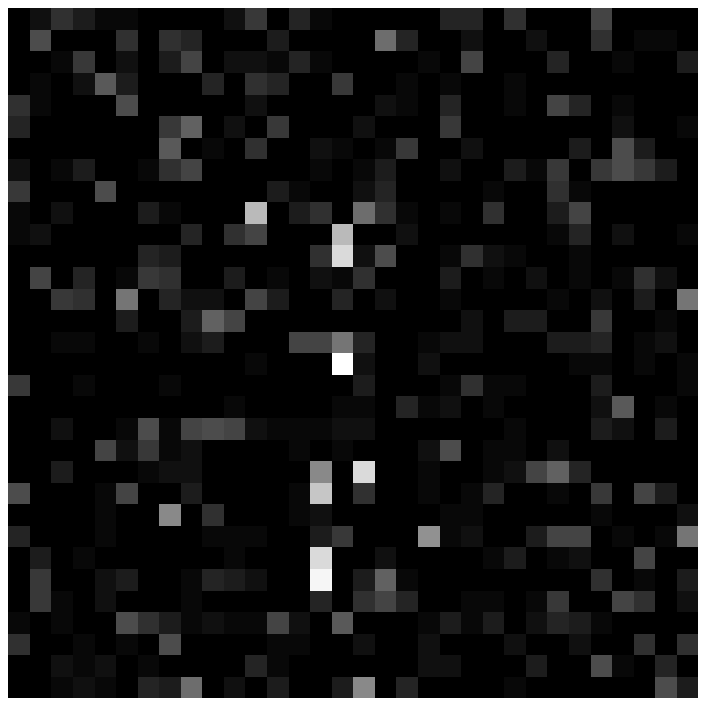}\vspace*{0.0cm}}
& 
\\
b) &
%
\parbox[c]{1cm}{\includegraphics[width=1.0cm]{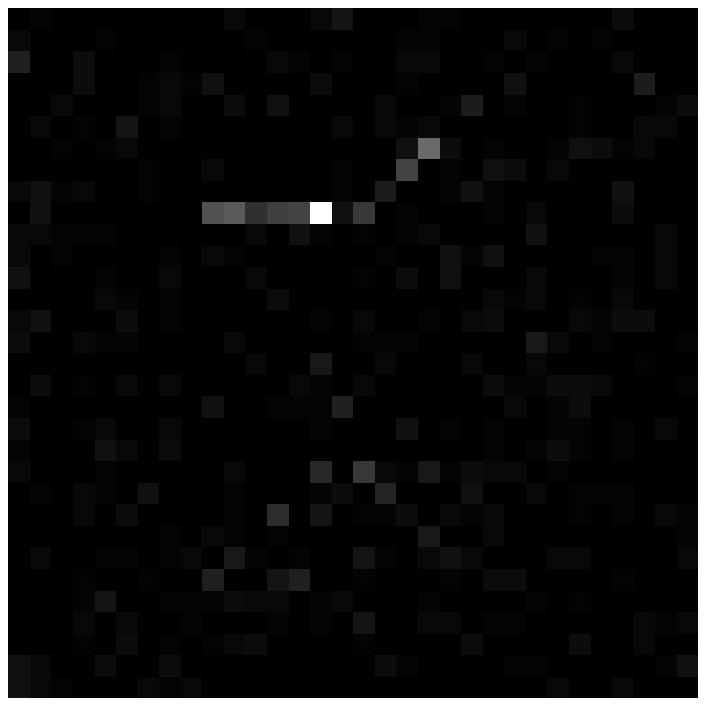}\vspace*{0.0cm}}
&
%
\parbox[c]{1cm}{\includegraphics[width=1.0cm]{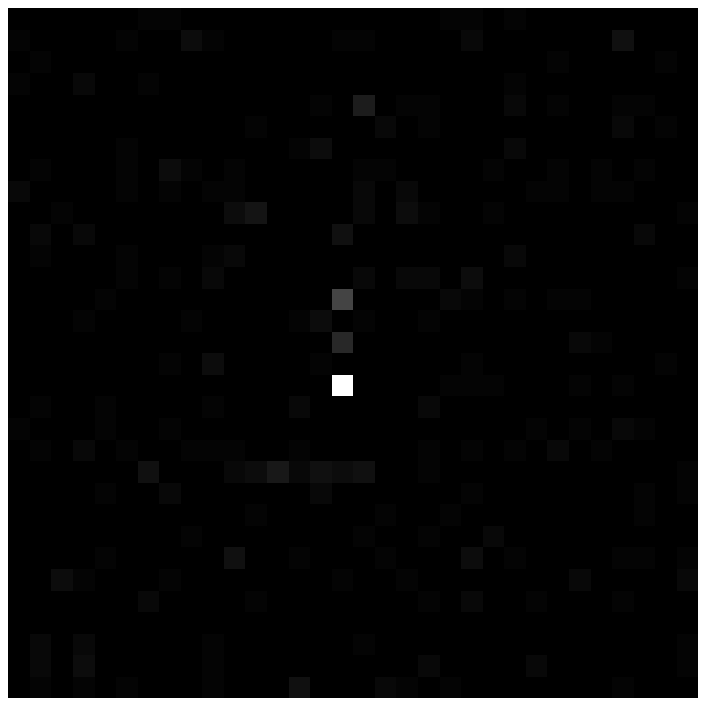}\vspace*{0.0cm}}
&
%
\parbox[c]{1cm}{\includegraphics[width=1.0cm]{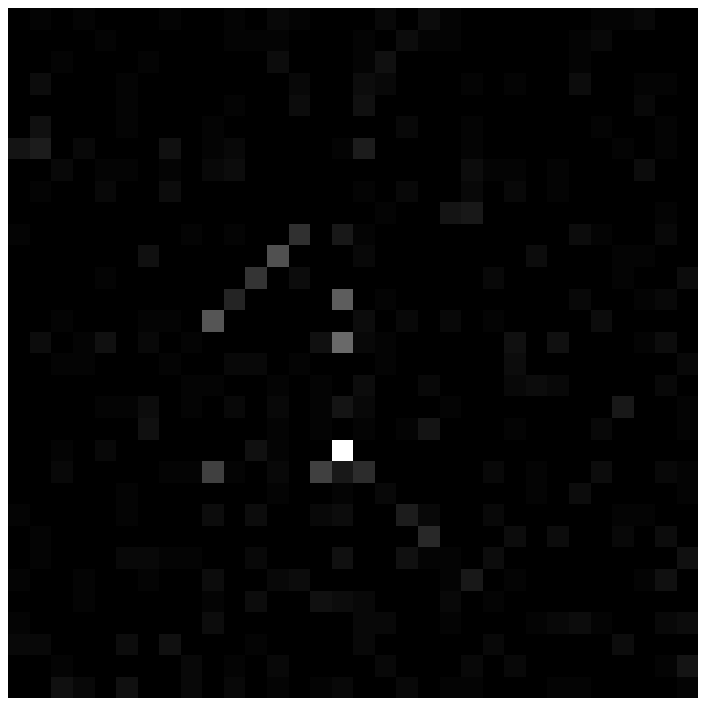}\vspace*{0.0cm}}
&
%
\parbox[c]{1cm}{\includegraphics[width=1.0cm]{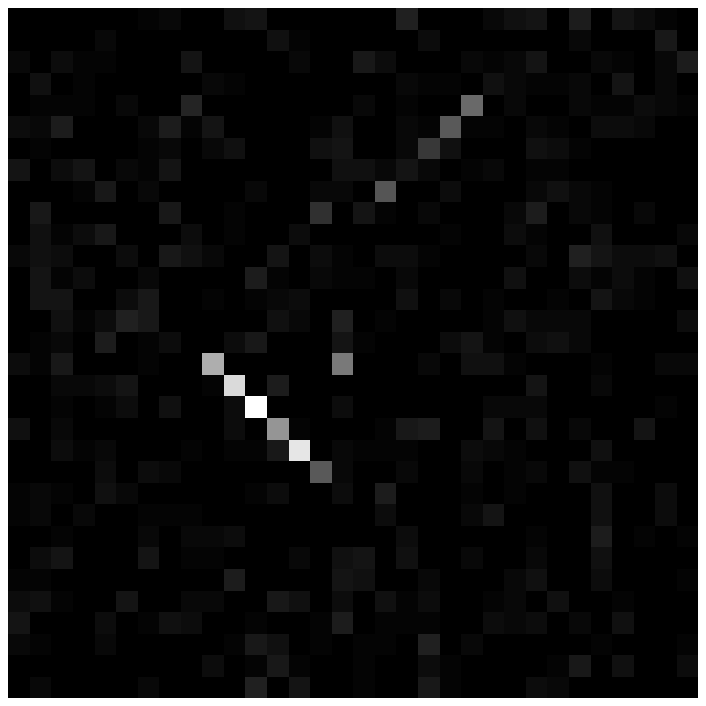}\vspace*{0.0cm}}
&
%
\parbox[c]{1cm}{\includegraphics[width=1.0cm]{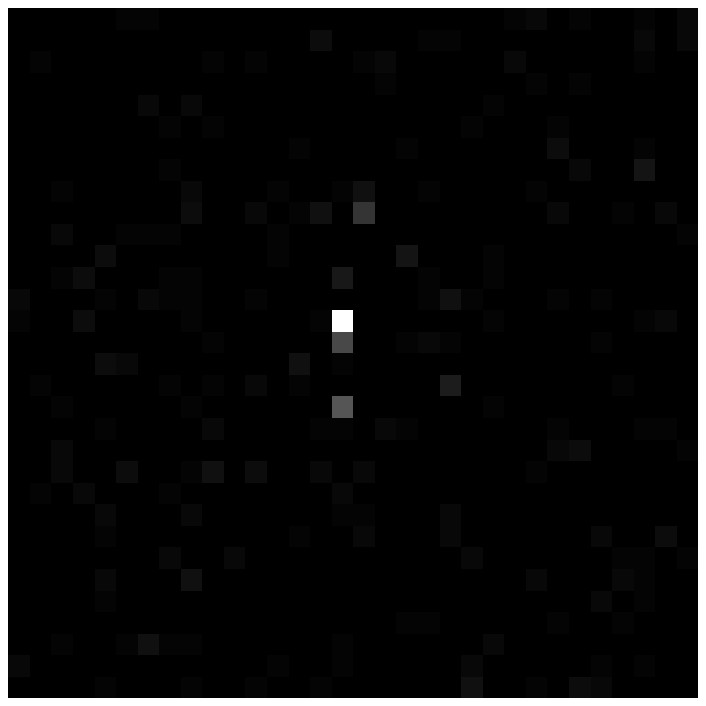}\vspace*{0.0cm}}
&
%
\parbox[c]{1cm}{\includegraphics[width=1.0cm]{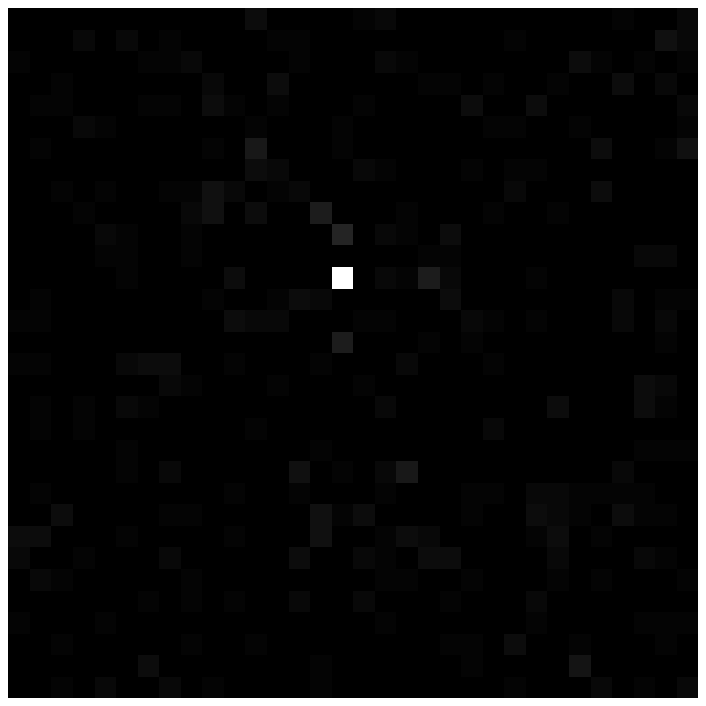}\vspace*{0.0cm}}

\\
c) &
%
\parbox[c]{1cm}{\includegraphics[width=1.0cm]{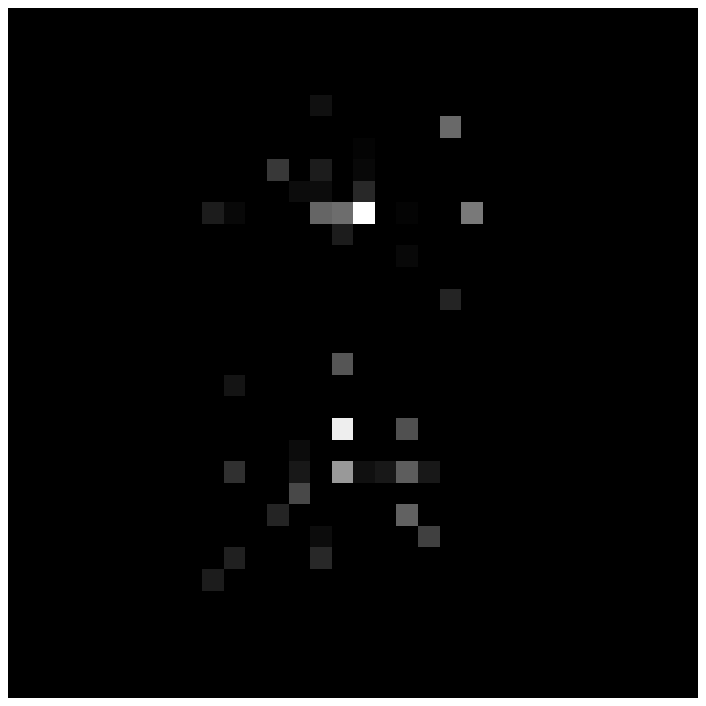}\vspace*{+0.05cm}}
&
%
\parbox[c]{1cm}{\includegraphics[width=1.0cm]{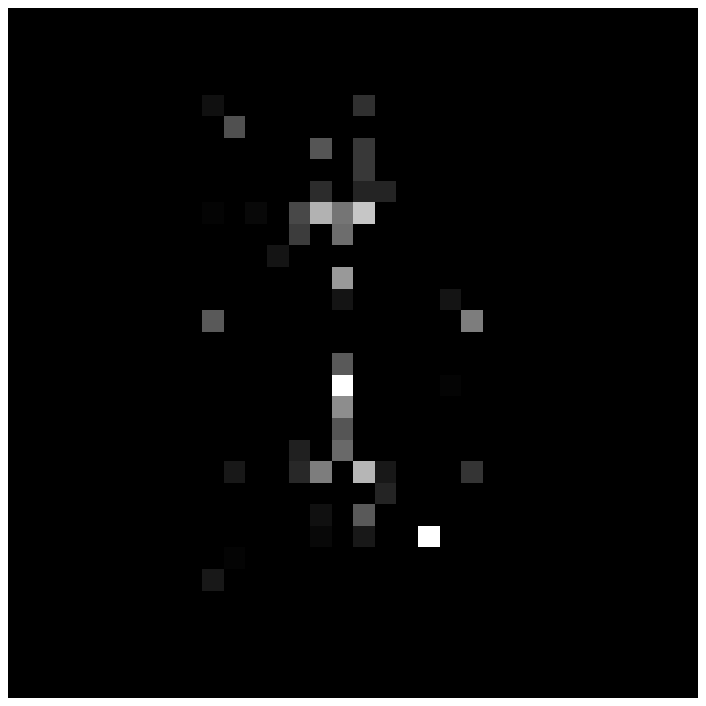}\vspace*{+0.05cm}}
&
%
\parbox[c]{1cm}{\includegraphics[width=1.0cm]{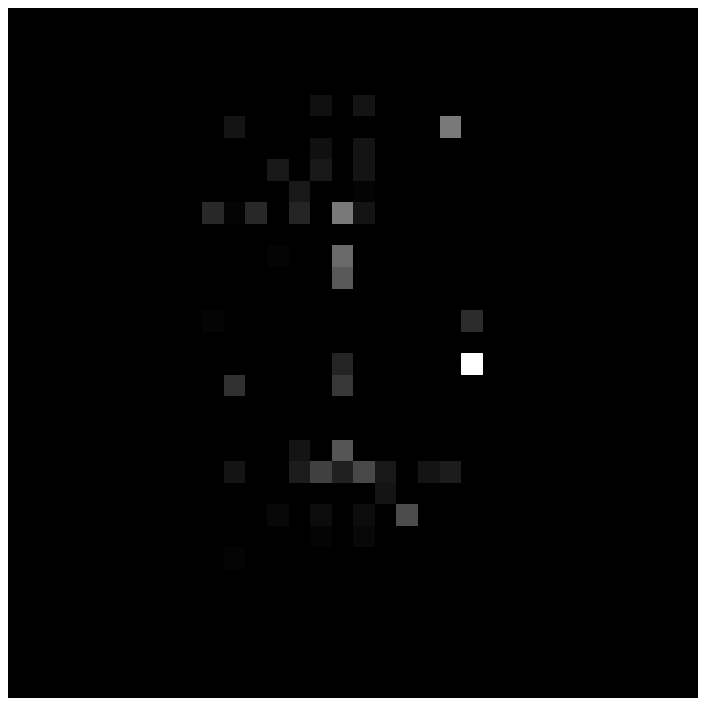}\vspace*{+0.05cm}}
&
%
\parbox[c]{1cm}{\includegraphics[width=1.0cm]{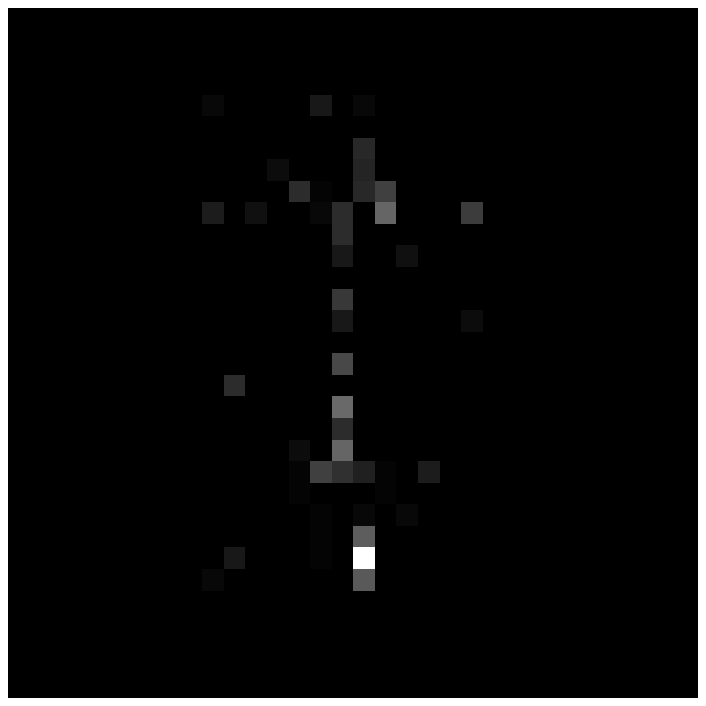}\vspace*{+0.05cm}}
&
%
\parbox[c]{1cm}{\includegraphics[width=1.0cm]{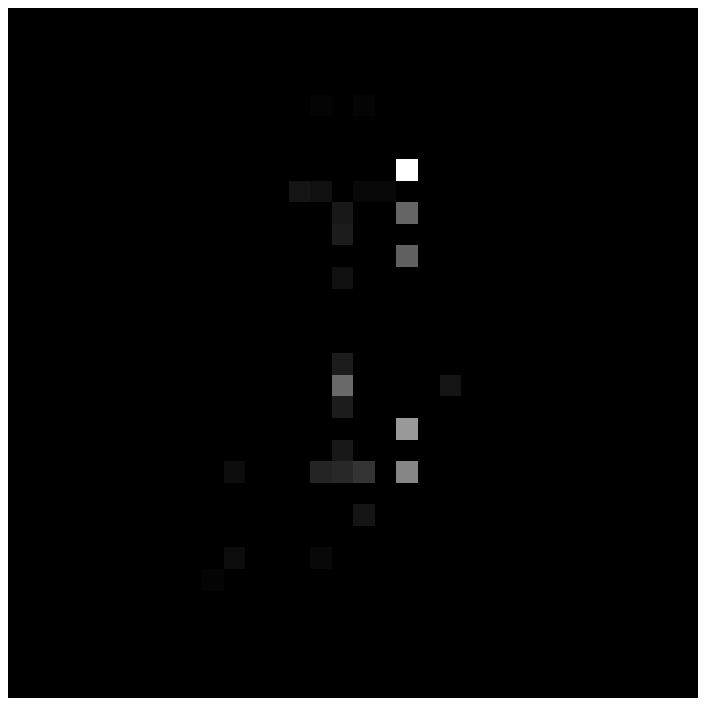}\vspace*{+0.05cm}}
& 
\\

d) &
%
\parbox[c]{1cm}{\includegraphics[width=1.0cm]{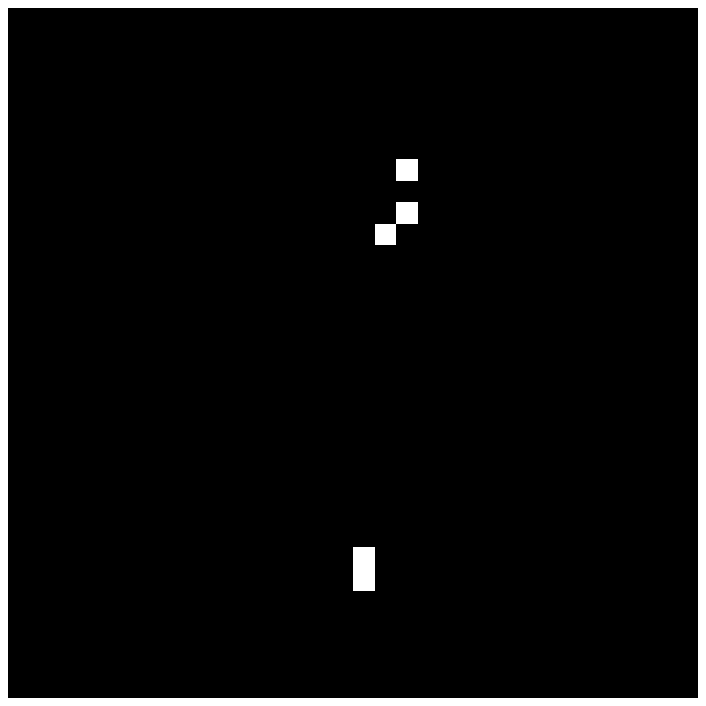}\vspace*{0.0cm}}
&
%
\parbox[c]{1cm}{\includegraphics[width=1.0cm]{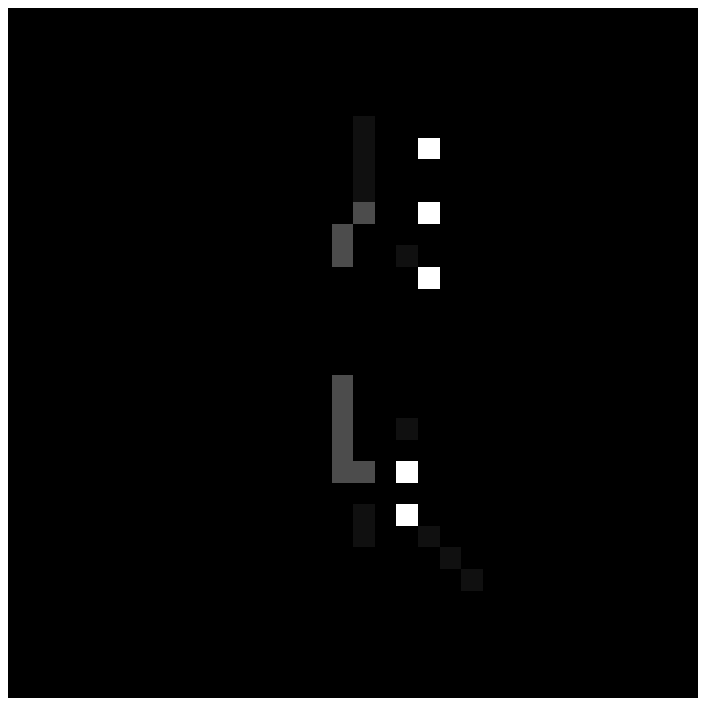}\vspace*{0.0cm}}
&
%
\parbox[c]{1cm}{\includegraphics[width=1.0cm]{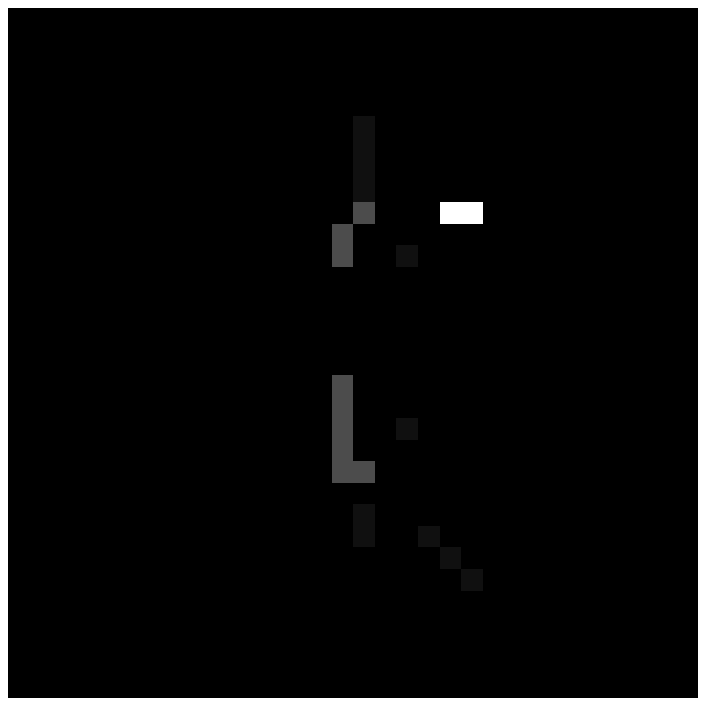}\vspace*{0.0cm}}
&
%
\parbox[c]{1cm}{\includegraphics[width=1.0cm]{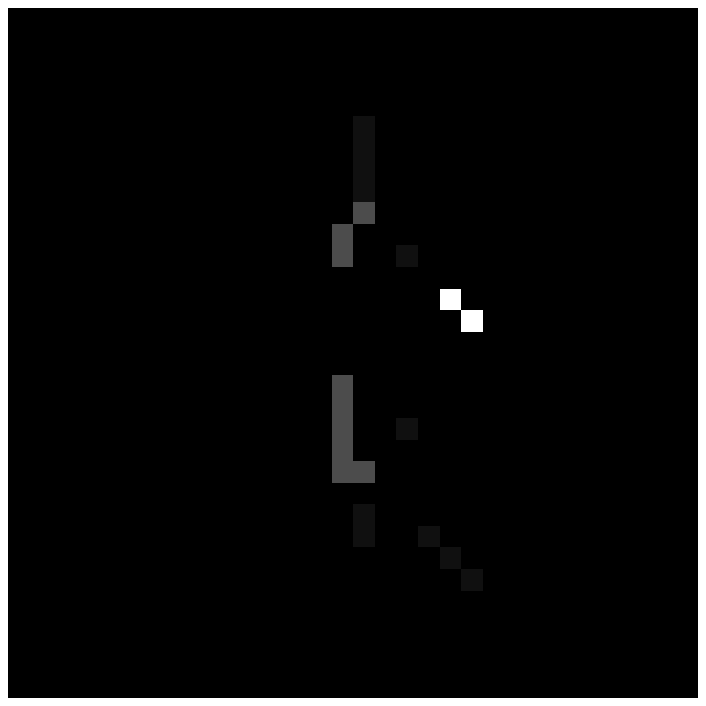}\vspace*{0.0cm}}
&
%
\parbox[c]{1cm}{\includegraphics[width=1.0cm]{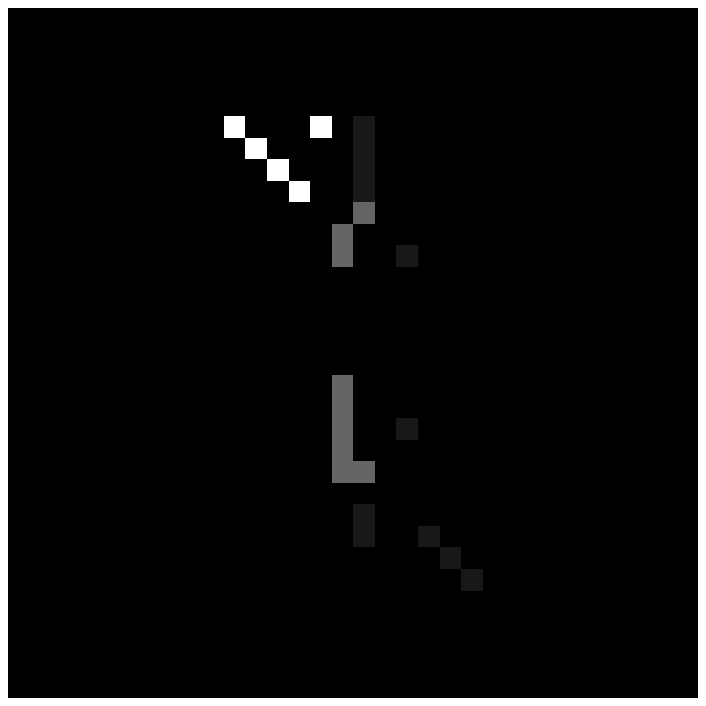}\vspace*{0.0cm}}
& 
%
%
\end{tabular}
\caption{Topic errors for (a) Gibbs \cite{Griffiths:ref},
  (b) NMF \cite{betaDivergence:ref} and c) example Topics extracted by RecL2 \cite{Arora2:ref} on the noisy Swimmer  dataset. d) Example Topic errors for RecL2 on clean Swimmer dataset. Figure depicts extracted topics that are not close to any ``ground truth''. The ground truth topics  correspond to 16 different positions of left/right arms and legs.}
\label{fig:badimages}
\end{figure}

%% file: figure4.tex
\begin{figure*}[!htb]
\begin{tabular}{|m{0.3cm}|@{\hskip 0.05cm}m{1.0cm}@{\hskip 0.05cm}m{1.0cm}@{\hskip 0.05cm}m{1.0cm}@{\hskip 0.05cm}m{1.0cm}@{\hskip 0.05cm}m{1.0cm}@{\hskip 0.05cm}m{1.0cm}@{\hskip 0.05cm}m{1.0cm}@{\hskip 0.05cm}m{1.0cm}@{\hskip 0.05cm}m{1.0cm}@{\hskip 0.05cm}m{1.0cm}@{\hskip 0.05cm}m{1.0cm}@{\hskip 0.05cm}m{1.0cm}@{\hskip 0.05cm}m{1.0cm}@{\hskip 0.05cm}m{1.0cm}@{\hskip 0.05cm}m{1.0cm}@{\hskip 0.05cm}m{1.0cm}@{\hskip 0.05cm}|}
\hline
Pos. & LA 1 & LA 2& LA 3& LA 4& RA 1 & RA 2& RA 3& RA 4& LL 1 & LL 2&
LL 3& LL 4& RL 1 & RL 2& RL 3& RL 4 \\
\hline
a)
&
\parbox[c]{1cm}{\vspace*{0.1cm}\includegraphics[width=1.0cm]{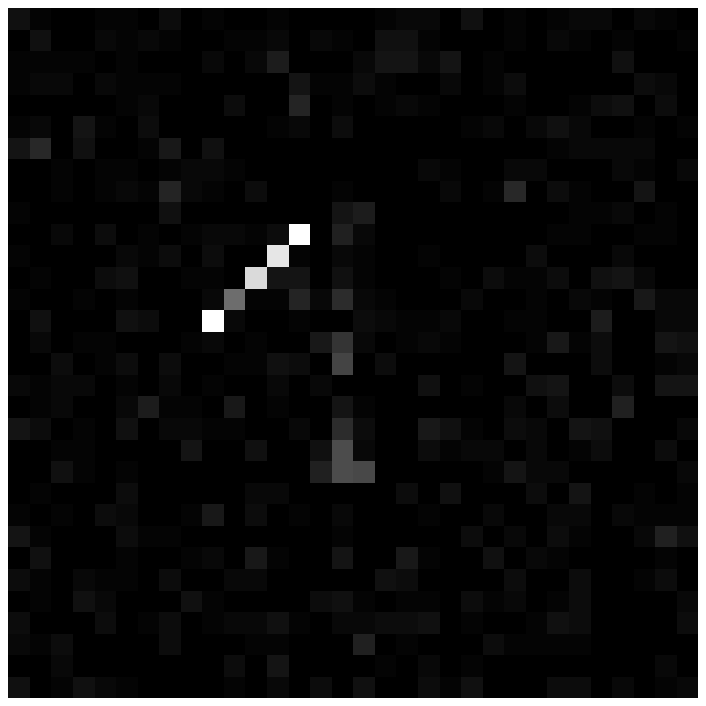}
\vspace*{-0.3cm}}
&
\parbox[c]{1cm}{\vspace*{0.1cm}\includegraphics[width=1.0cm]{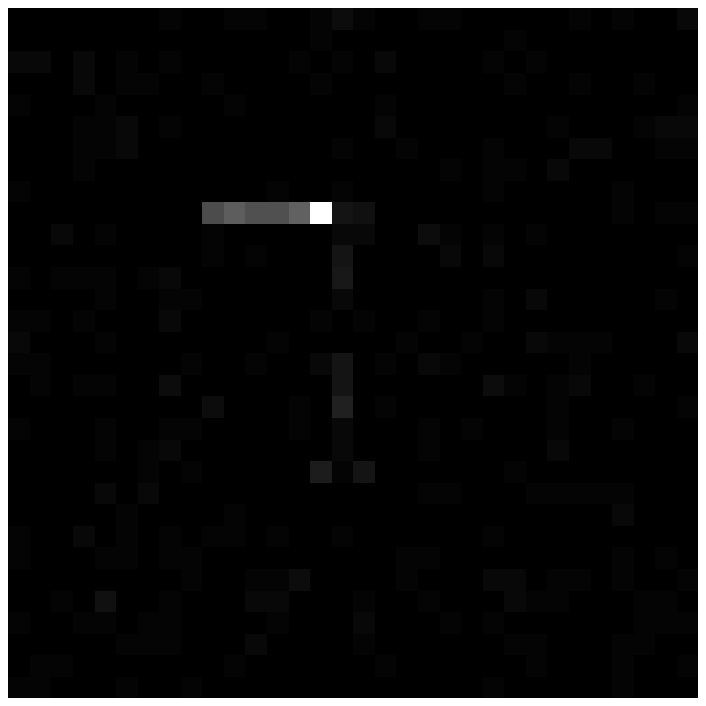}
\vspace*{-0.3cm}}
&
\parbox[c]{1cm}{\vspace*{0.1cm}\includegraphics[width=1.0cm]{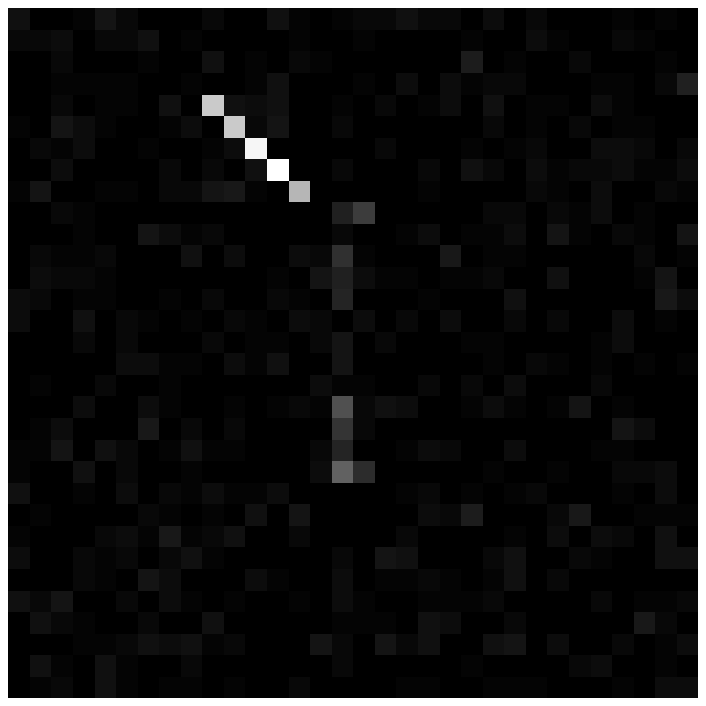}
\vspace*{-0.3cm}}
&
\parbox[c]{1cm}{\vspace*{0.1cm}\includegraphics[width=1.0cm]{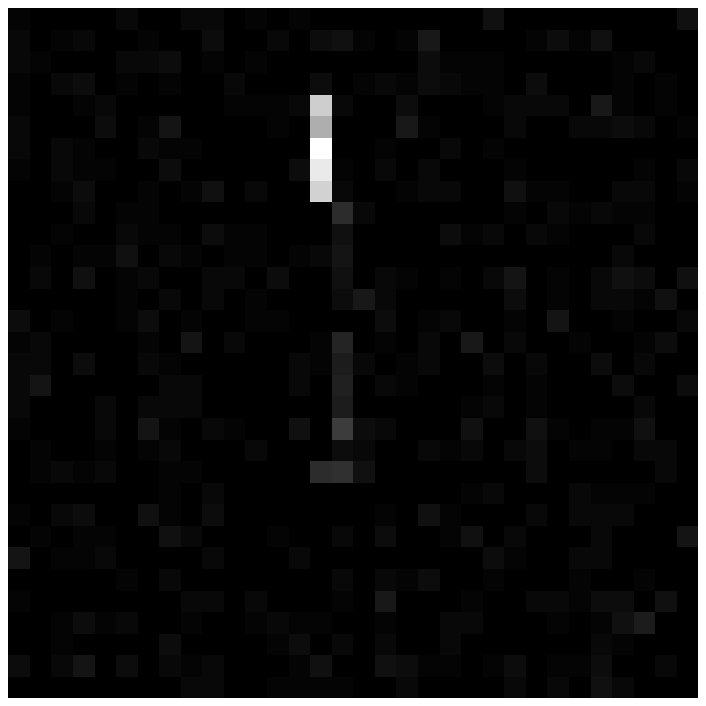}
\vspace*{-0.3cm}}
&
\parbox[c]{1cm}{\vspace*{0.1cm}\includegraphics[width=1.0cm]{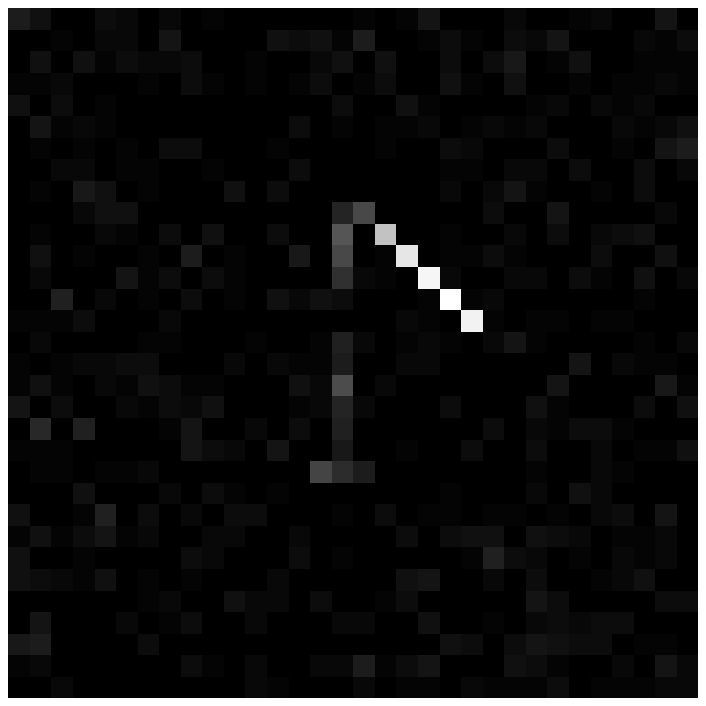}
\vspace*{-0.3cm}}
&
\parbox[c]{1cm}{\vspace*{0.1cm}\includegraphics[width=1.0cm]{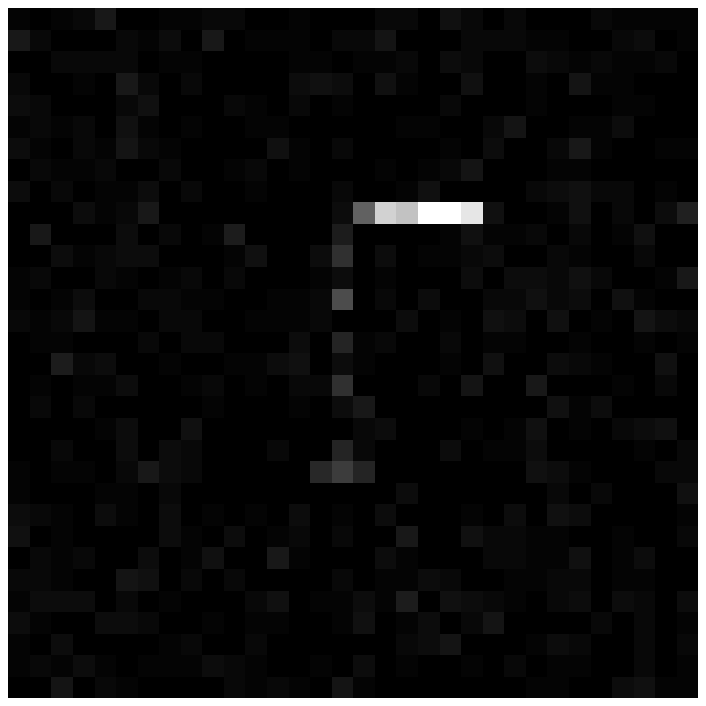}
\vspace*{-0.3cm}}
&
\parbox[c]{1cm}{\vspace*{0.1cm}\includegraphics[width=1.0cm]{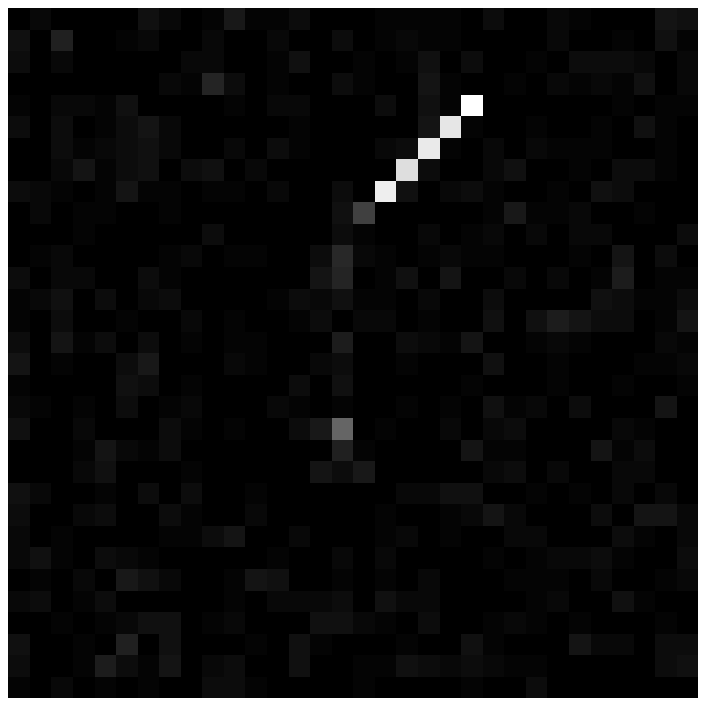}
\vspace*{-0.3cm}}
&
\parbox[c]{1cm}{\vspace*{0.1cm}\includegraphics[width=1.0cm]{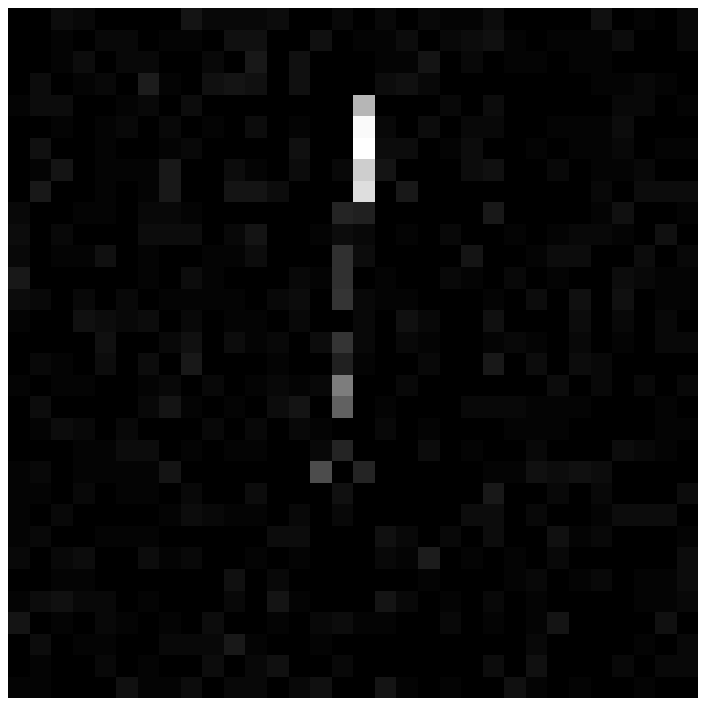}
\vspace*{-0.3cm}}
&
\parbox[c]{1cm}{\vspace*{0.1cm}\includegraphics[width=1.0cm]{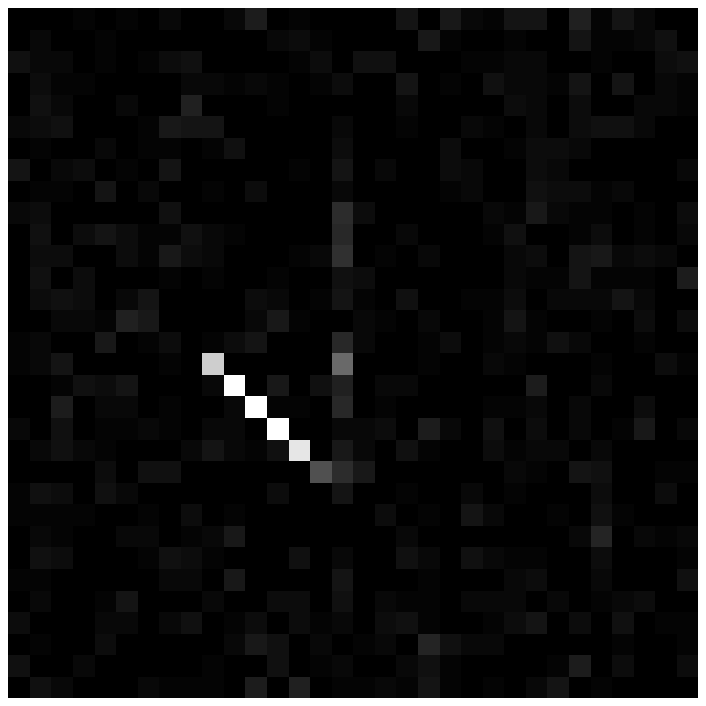}
\vspace*{-0.3cm}}
&
\parbox[c]{1cm}{\vspace*{0.1cm}\includegraphics[width=1.0cm]{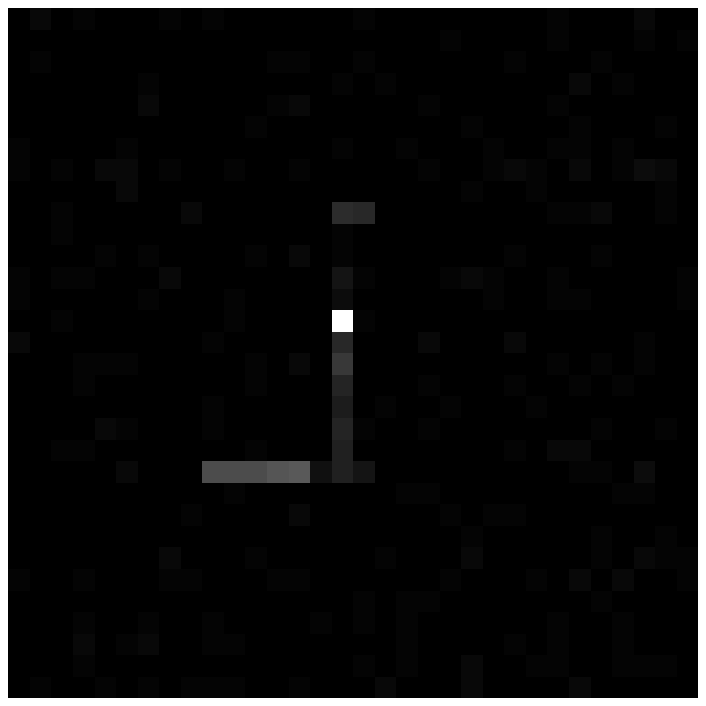}
\vspace*{-0.3cm}}
&
\parbox[c]{1cm}{\vspace*{0.1cm}\includegraphics[width=1.0cm]{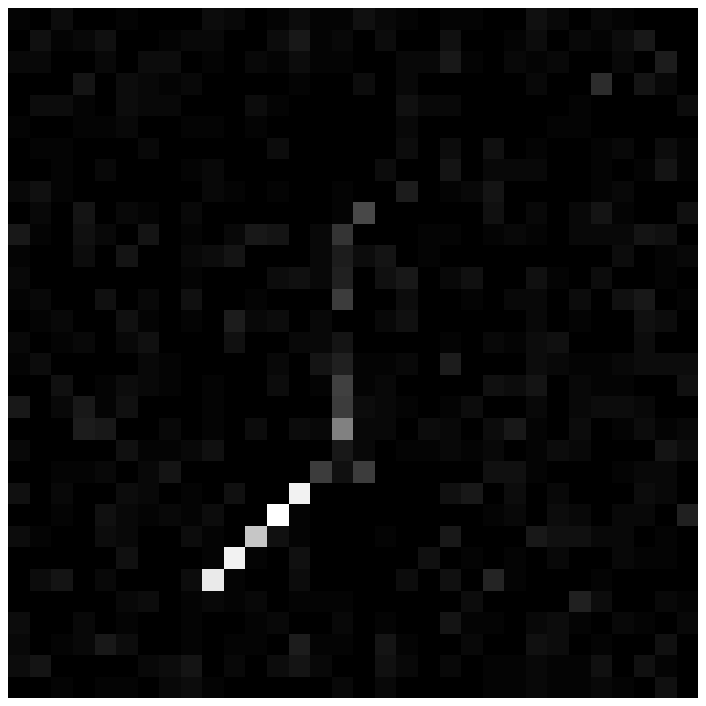}
\vspace*{-0.3cm}}
&
\parbox[c]{1cm}{\vspace*{0.1cm}\includegraphics[width=1.0cm]{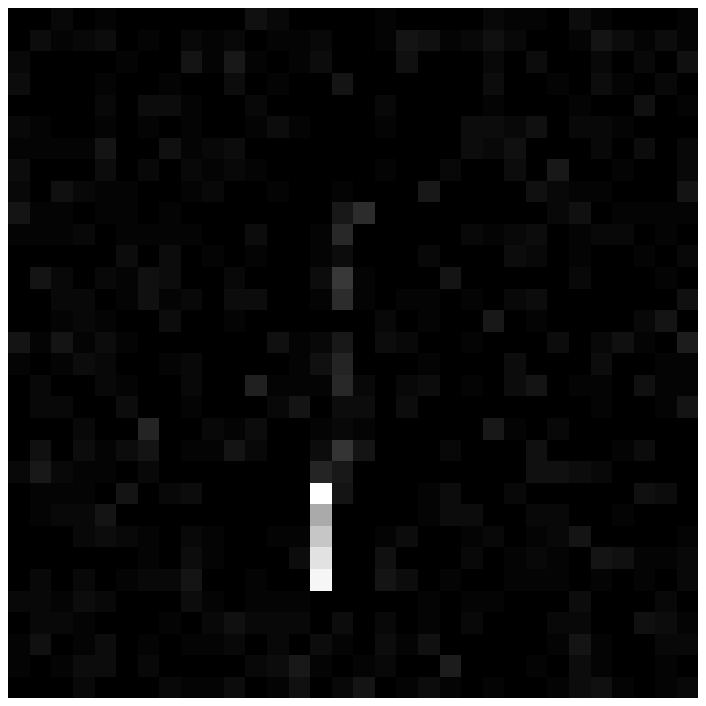}
\vspace*{-0.3cm}}
&
\parbox[c]{1cm}{\vspace*{0.1cm}\includegraphics[width=1.0cm]{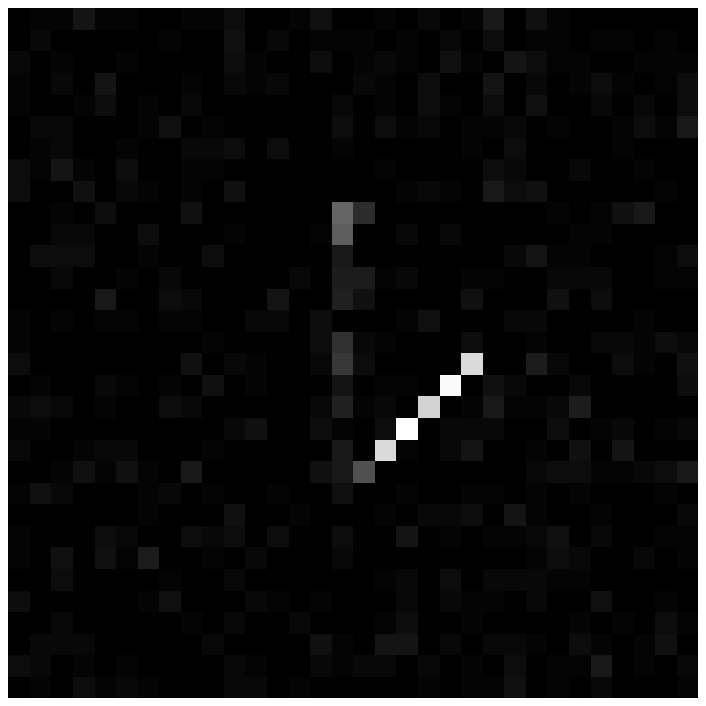}
\vspace*{-0.3cm}}
&
\parbox[c]{1cm}{\vspace*{0.1cm}\includegraphics[width=1.0cm]{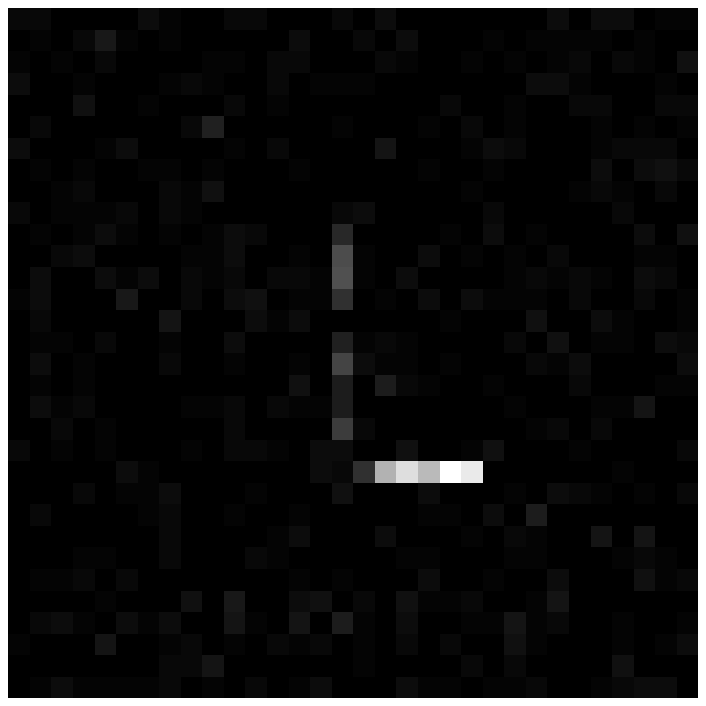}
\vspace*{-0.3cm}}
&
\parbox[c]{1cm}{\vspace*{0.1cm}\includegraphics[width=1.0cm]{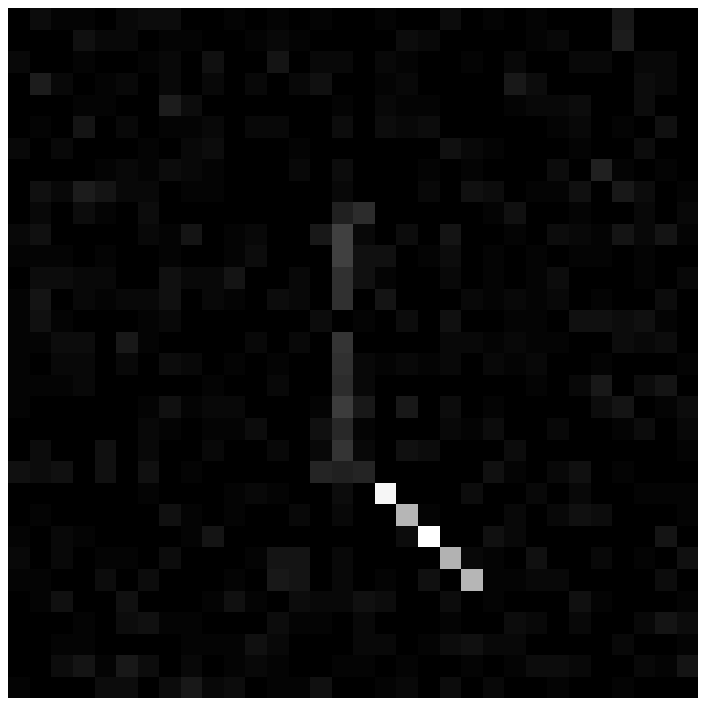}
\vspace*{-0.3cm}}
&
\parbox[c]{1cm}{\vspace*{0.1cm}\includegraphics[width=1.0cm]{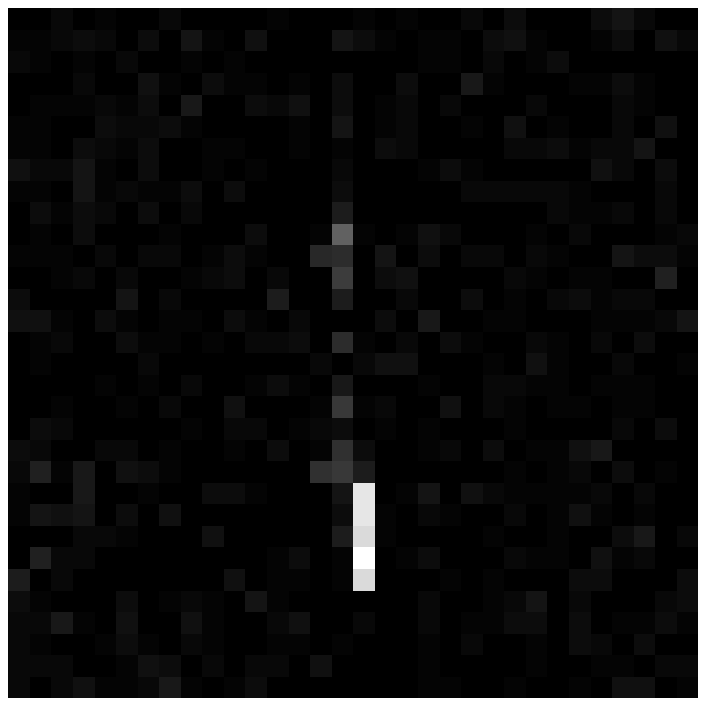}
\vspace*{-0.3cm}}
 \\
\hline
b)
&
\parbox[c]{1cm}{\vspace*{0.1cm}\includegraphics[width=1.0cm]{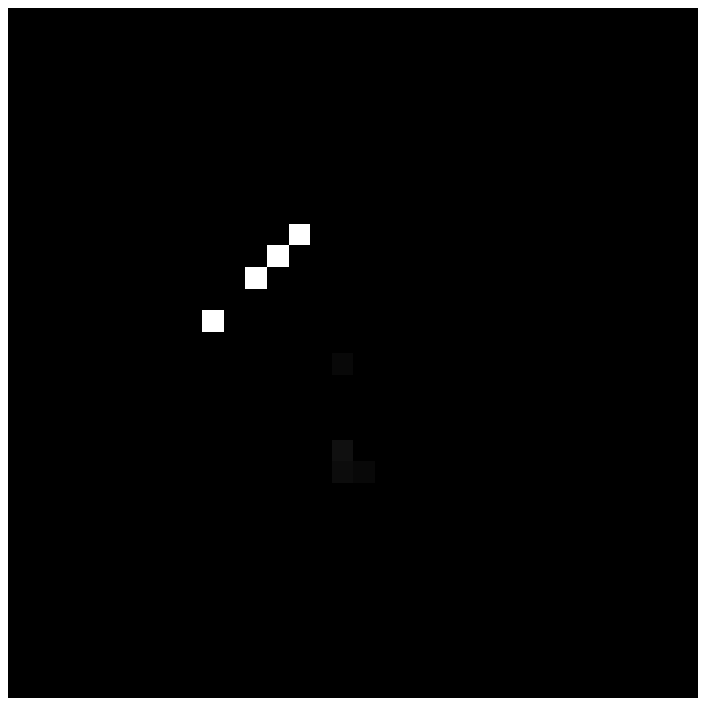}
\vspace*{-0.3cm}}
&
\parbox[c]{1cm}{\vspace*{0.1cm}\includegraphics[width=1.0cm]{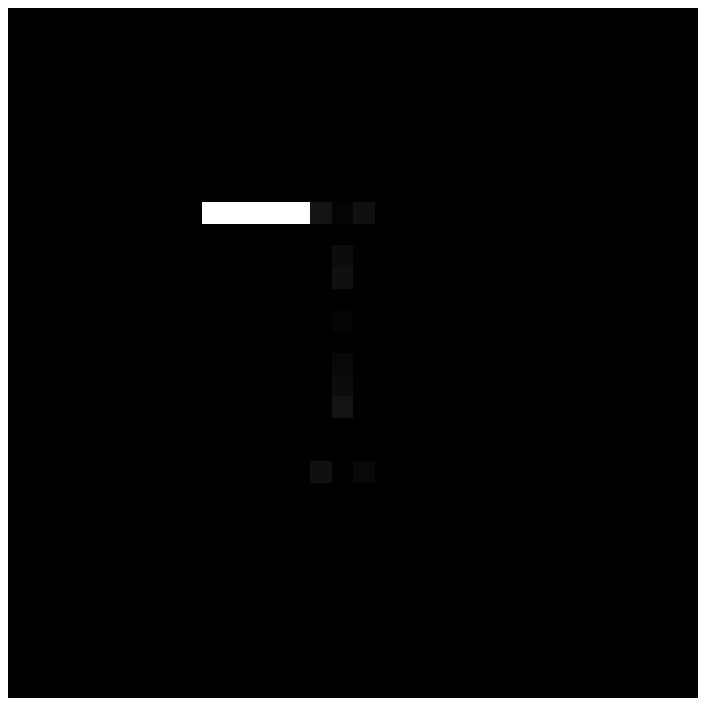}
\vspace*{-0.3cm}}
&
\parbox[c]{1cm}{\vspace*{0.1cm}\includegraphics[width=1.0cm]{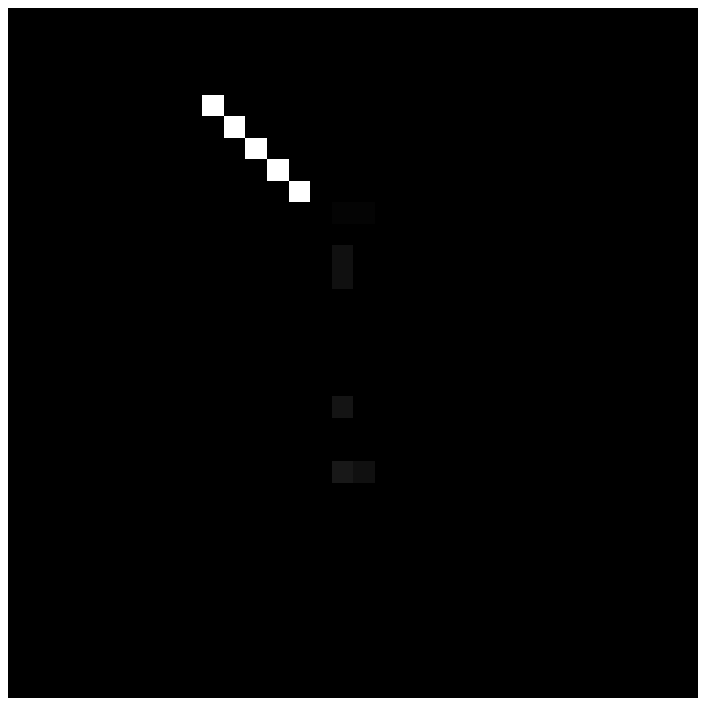}
\vspace*{-0.3cm}}
&
\parbox[c]{1cm}{\vspace*{0.1cm}\includegraphics[width=1.0cm]{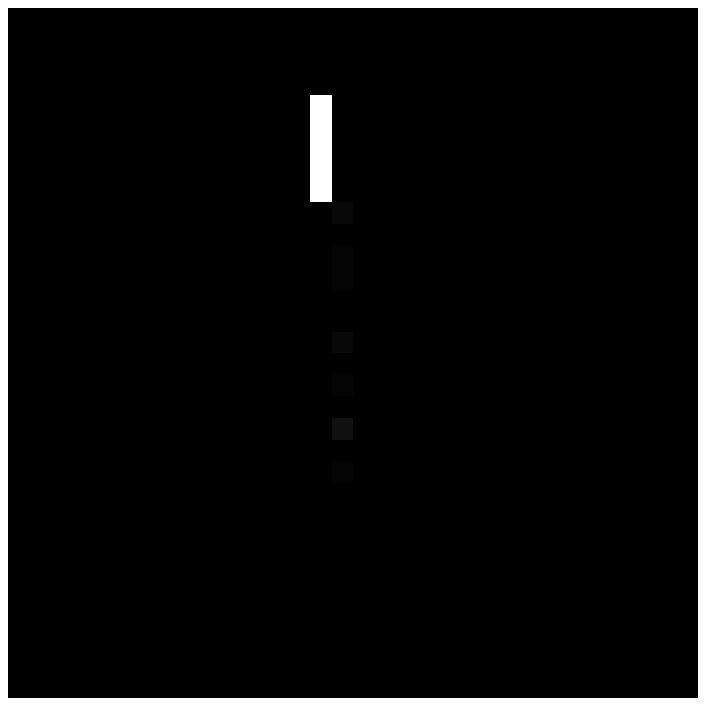}
\vspace*{-0.3cm}}
&
\parbox[c]{1cm}{\vspace*{0.1cm}\includegraphics[width=1.0cm]{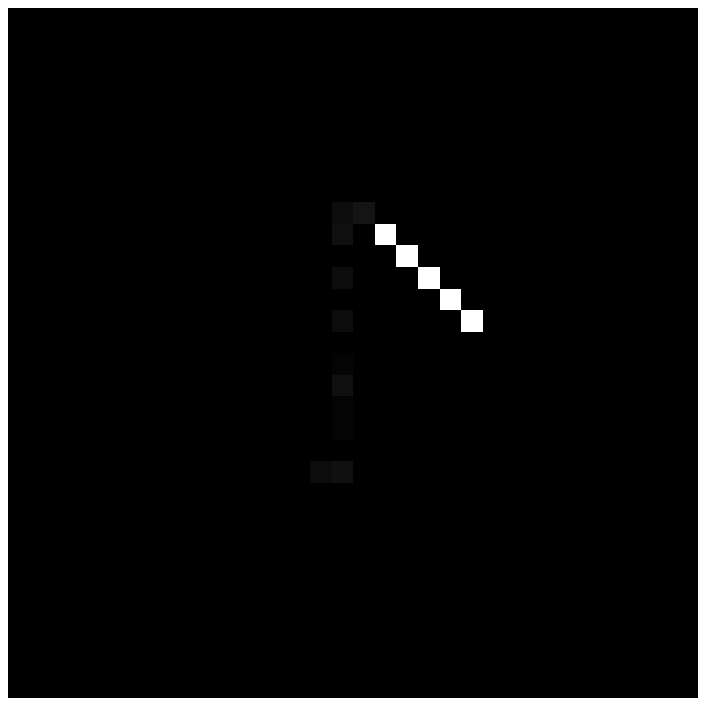}
\vspace*{-0.3cm}}
&
\parbox[c]{1cm}{\vspace*{0.1cm}\includegraphics[width=1.0cm]{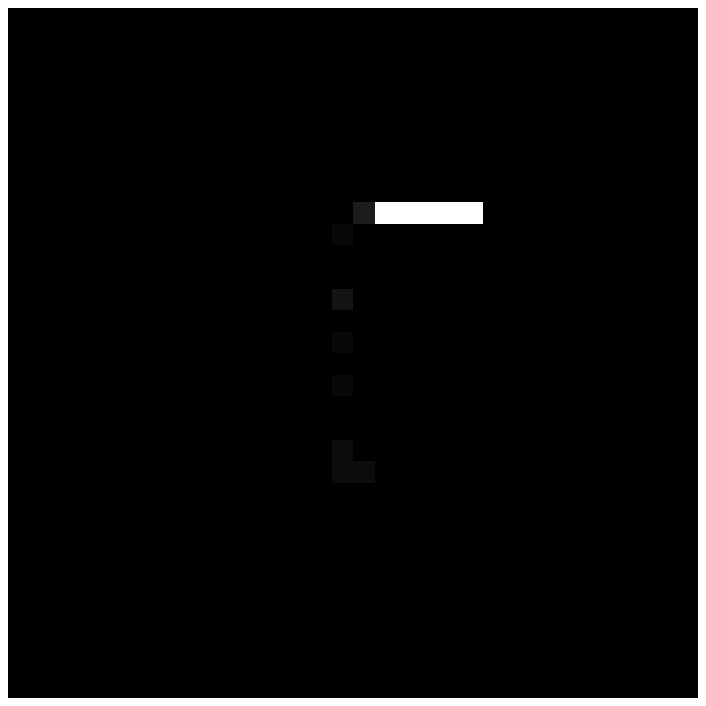}
\vspace*{-0.3cm}}
&
\parbox[c]{1cm}{\vspace*{0.1cm}\includegraphics[width=1.0cm]{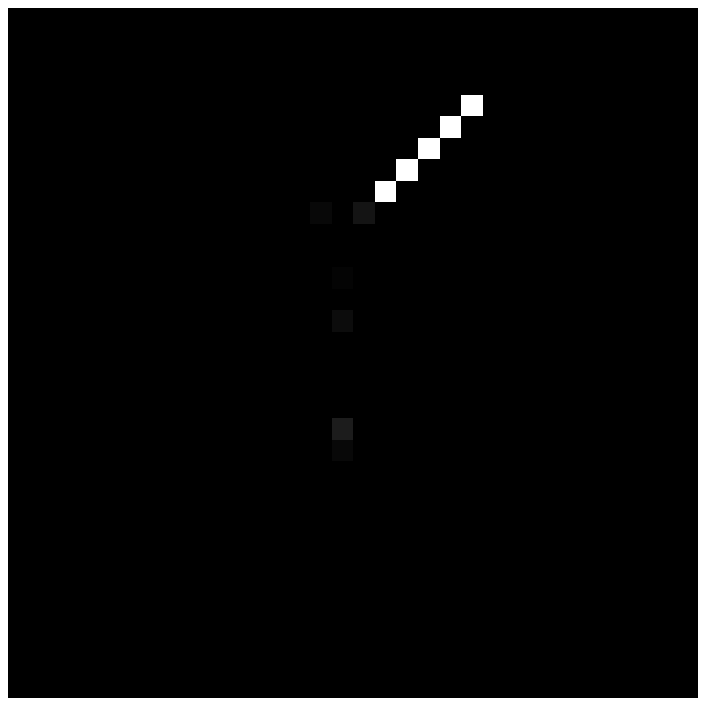}
\vspace*{-0.3cm}}
&
\parbox[c]{1cm}{\vspace*{0.1cm}\includegraphics[width=1.0cm]{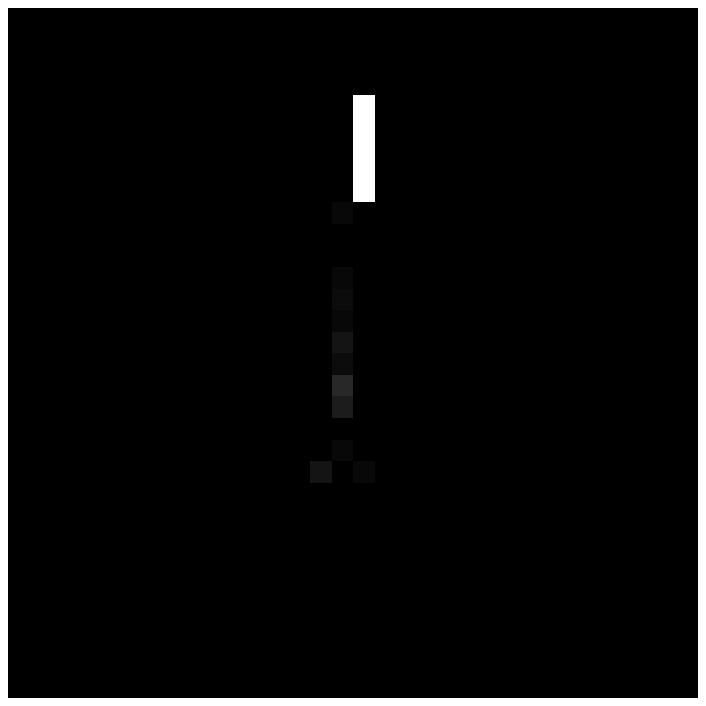}
\vspace*{-0.3cm}}
&
\parbox[c]{1cm}{\vspace*{0.1cm}\includegraphics[width=1.0cm]{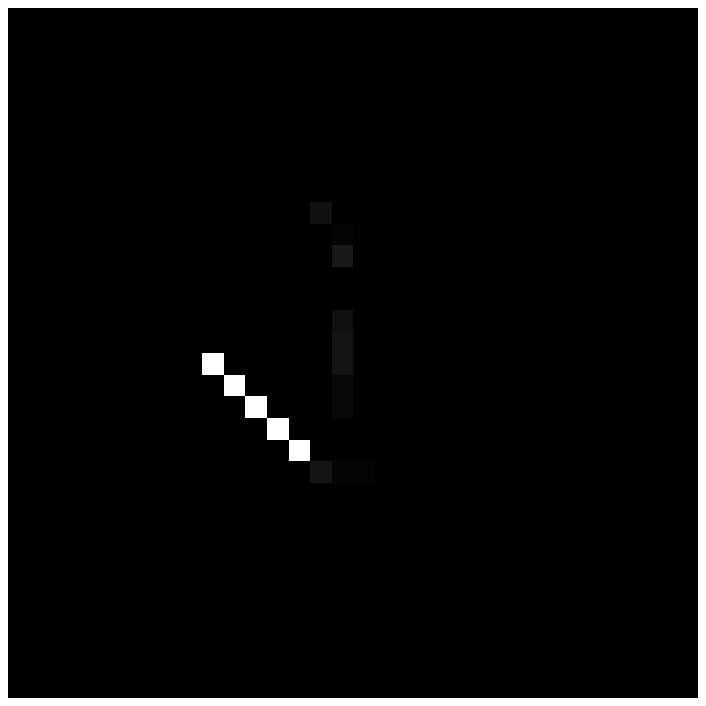}
\vspace*{-0.3cm}}
&
\parbox[c]{1cm}{\vspace*{0.1cm}\includegraphics[width=1.0cm]{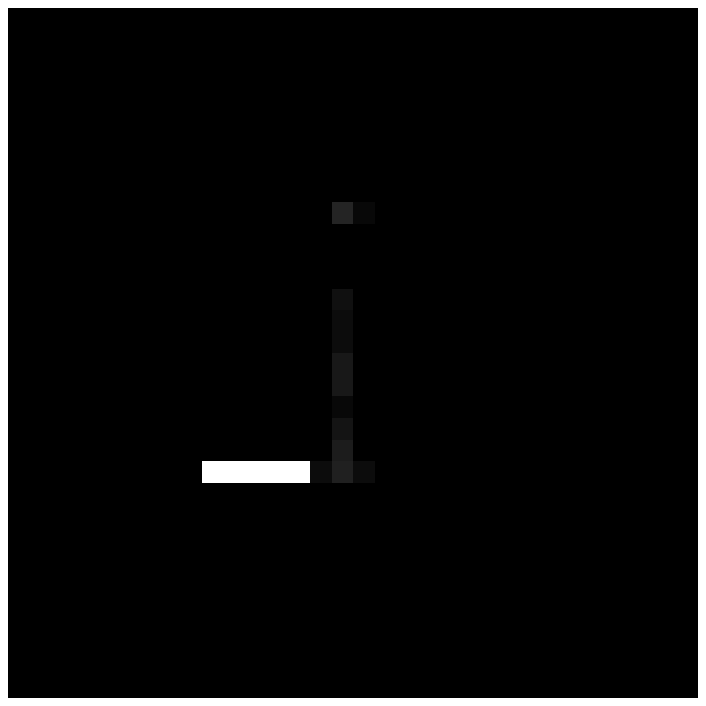}
\vspace*{-0.3cm}}
&
\parbox[c]{1cm}{\vspace*{0.1cm}\includegraphics[width=1.0cm]{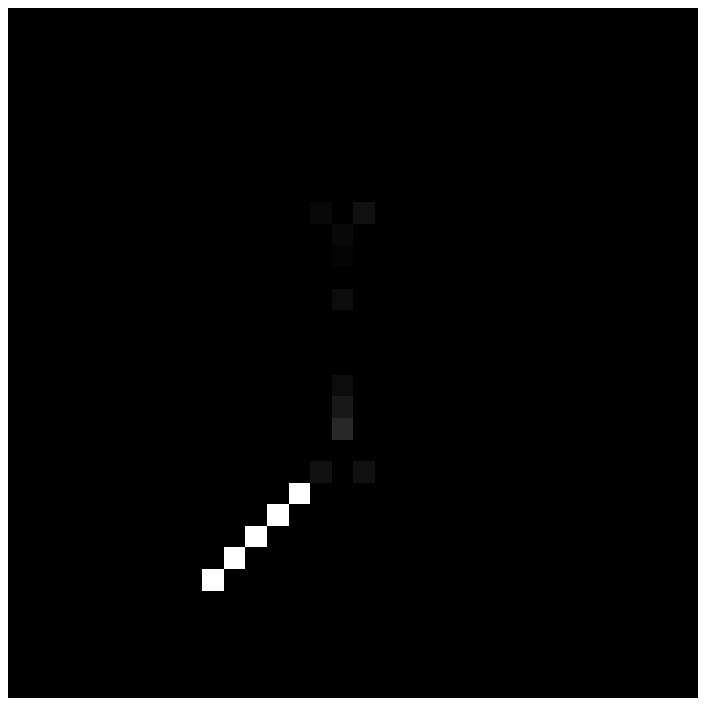}
\vspace*{-0.3cm}}
&
\parbox[c]{1cm}{\vspace*{0.1cm}\includegraphics[width=1.0cm]{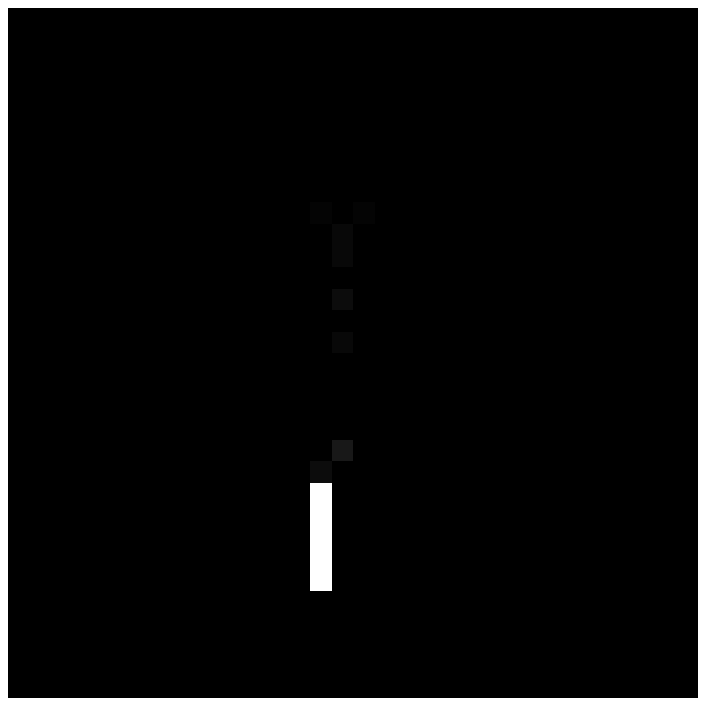}
\vspace*{-0.3cm}}
&
\parbox[c]{1cm}{\vspace*{0.1cm}\includegraphics[width=1.0cm]{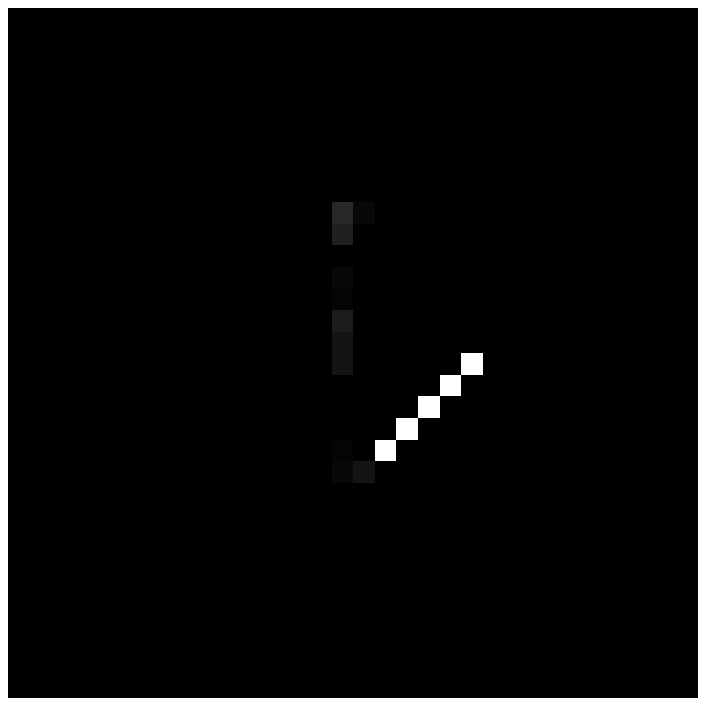}
\vspace*{-0.3cm}}
&
\parbox[c]{1cm}{\vspace*{0.1cm}\includegraphics[width=1.0cm]{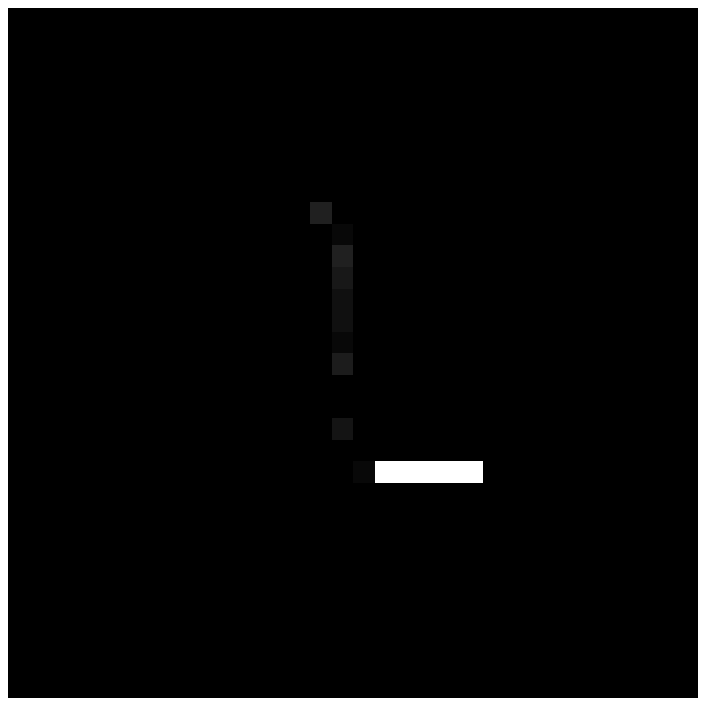}
\vspace*{-0.3cm}}
&
\parbox[c]{1cm}{\vspace*{0.1cm}\includegraphics[width=1.0cm]{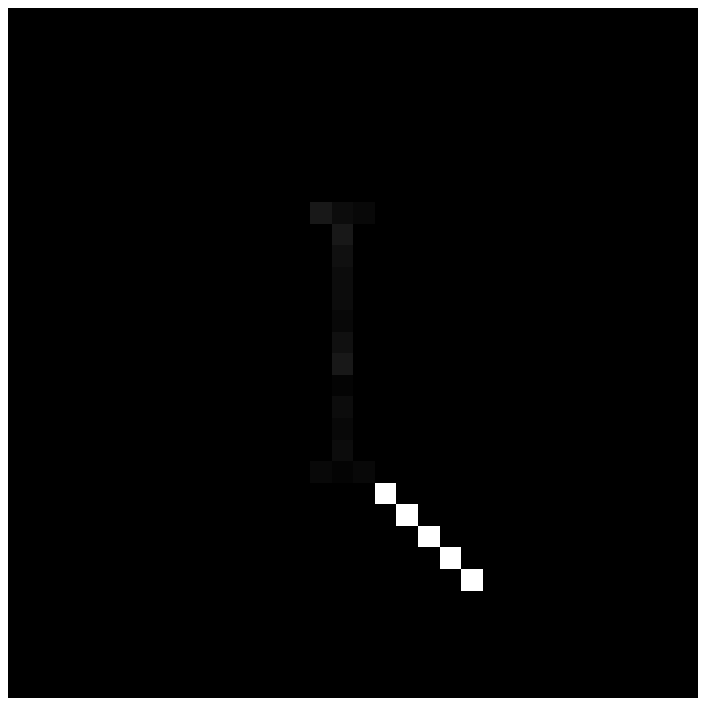}
\vspace*{-0.3cm}}
&
\parbox[c]{1cm}{\vspace*{0.1cm}\includegraphics[width=1.0cm]{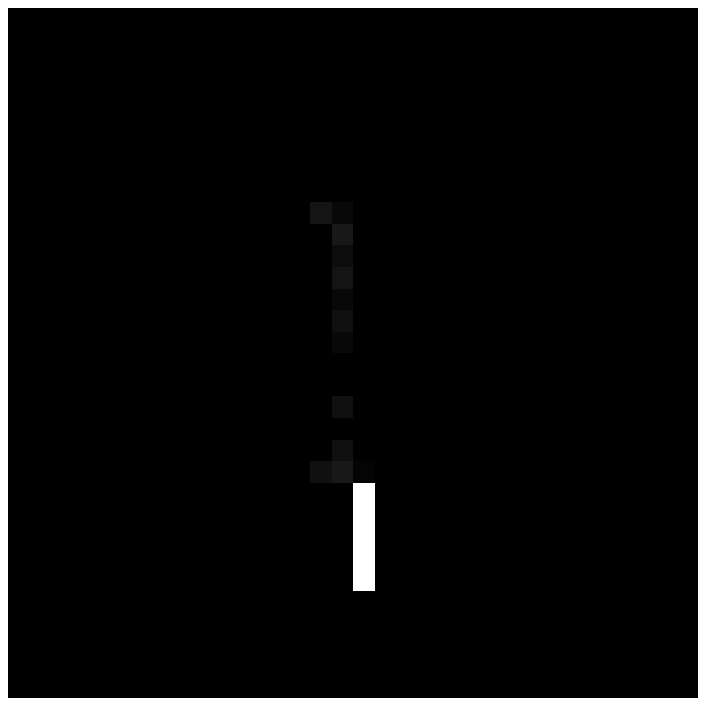}
\vspace*{-0.3cm}}
 \\
\hline
c) &
\parbox[c]{1cm}{\vspace*{0.1cm}\includegraphics[width=1.0cm]{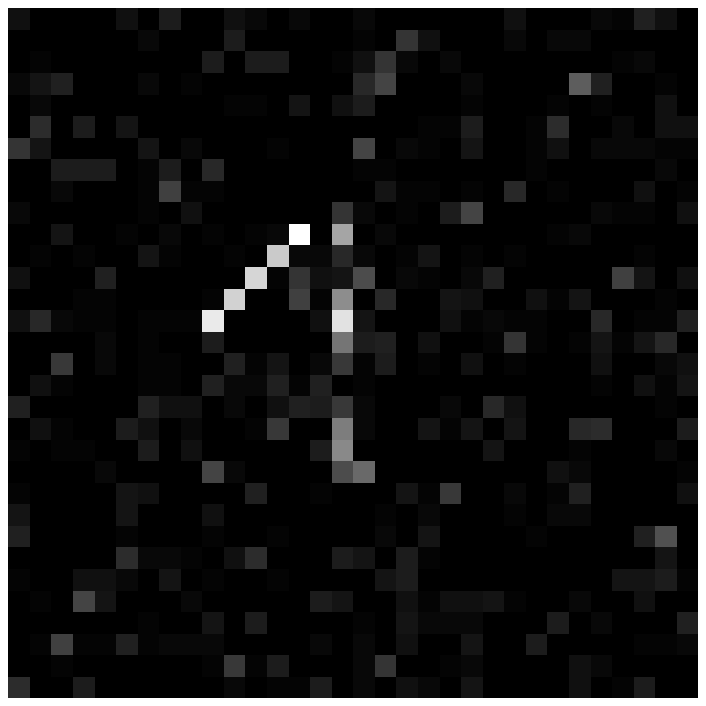}
\vspace*{-0.3cm}}
&
\parbox[c]{1cm}{\vspace*{0.1cm}\includegraphics[width=1.0cm]{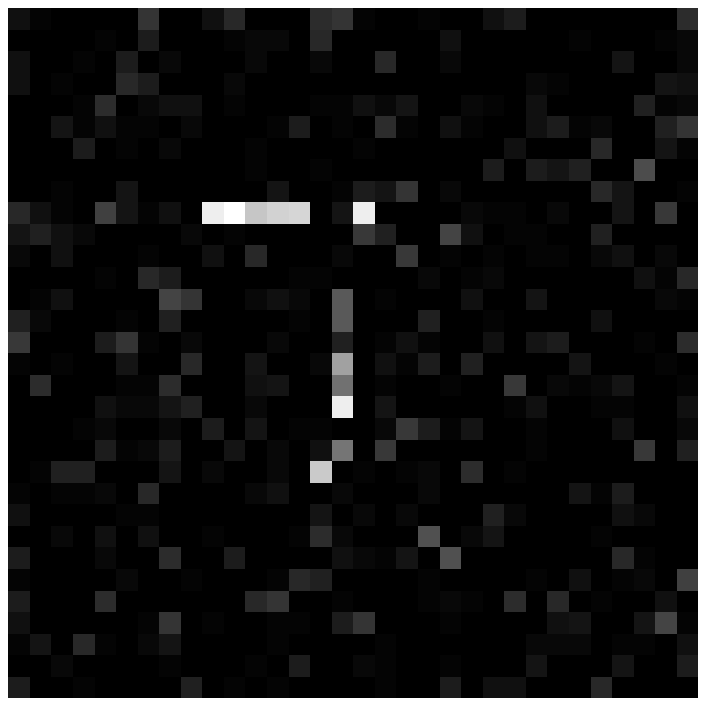}
\vspace*{-0.3cm}}
&
\parbox[c]{1cm}{\vspace*{0.1cm}\includegraphics[width=1.0cm]{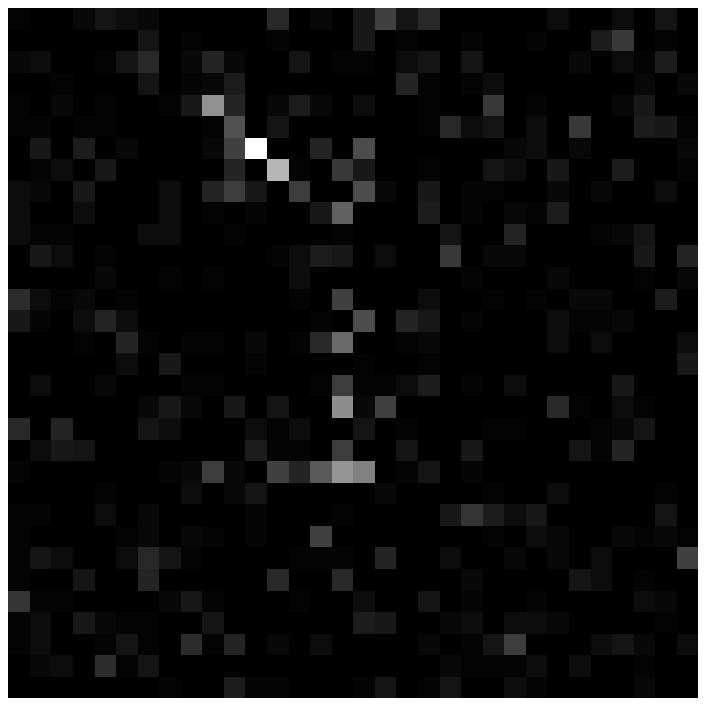}
\vspace*{-0.3cm}}
&
&
\parbox[c]{1cm}{\vspace*{0.1cm}\includegraphics[width=1.0cm]{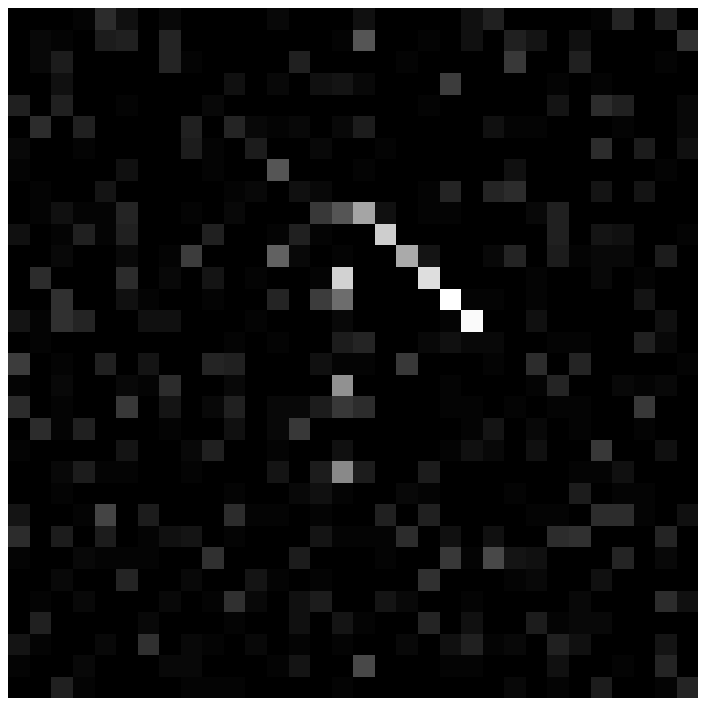}
\vspace*{-0.3cm}}
&
&
\parbox[c]{1cm}{\vspace*{0.1cm}\includegraphics[width=1.0cm]{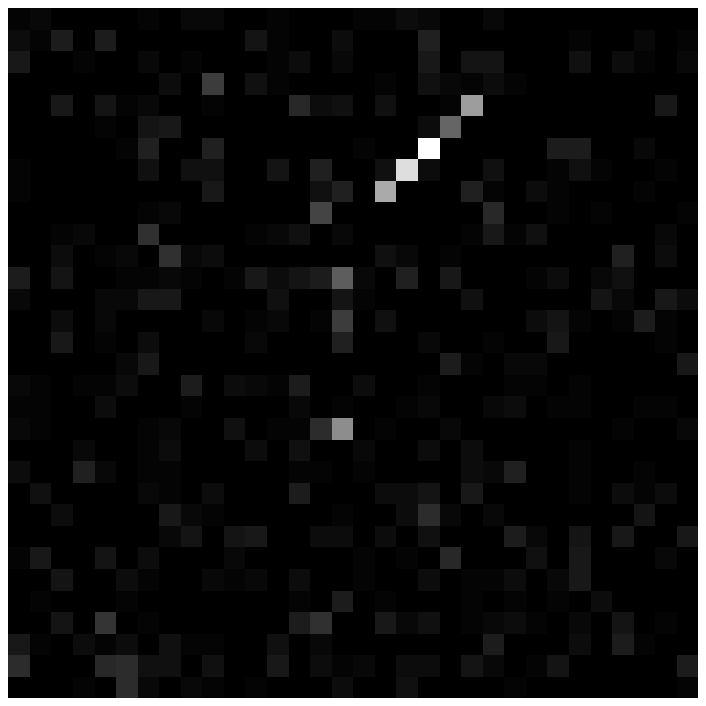}
\vspace*{-0.3cm}}
&
\parbox[c]{1cm}{\vspace*{0.1cm}\includegraphics[width=1.0cm]{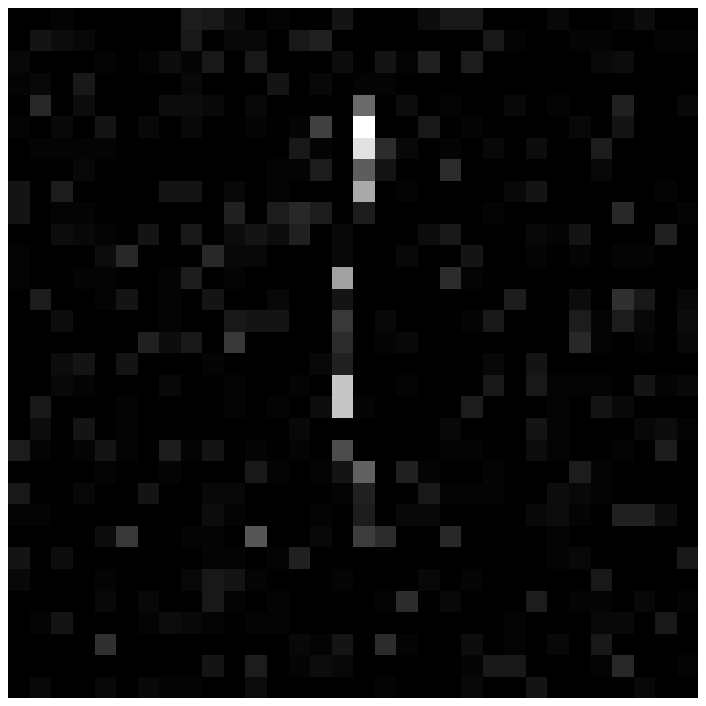}
\vspace*{-0.3cm}}
&
\parbox[c]{1cm}{\vspace*{0.1cm}\includegraphics[width=1.0cm]{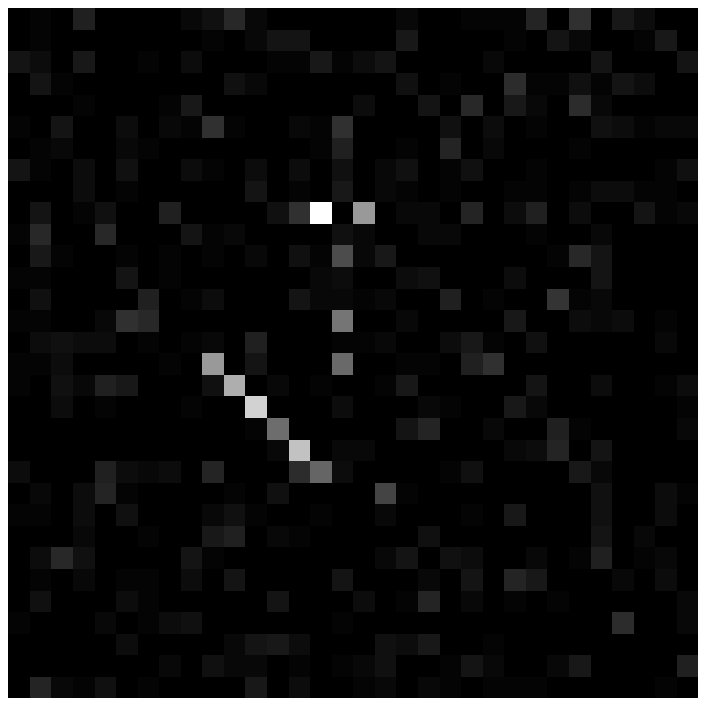}
\vspace*{-0.3cm}}
&
&
&
\parbox[c]{1cm}{\vspace*{0.1cm}\includegraphics[width=1.0cm]{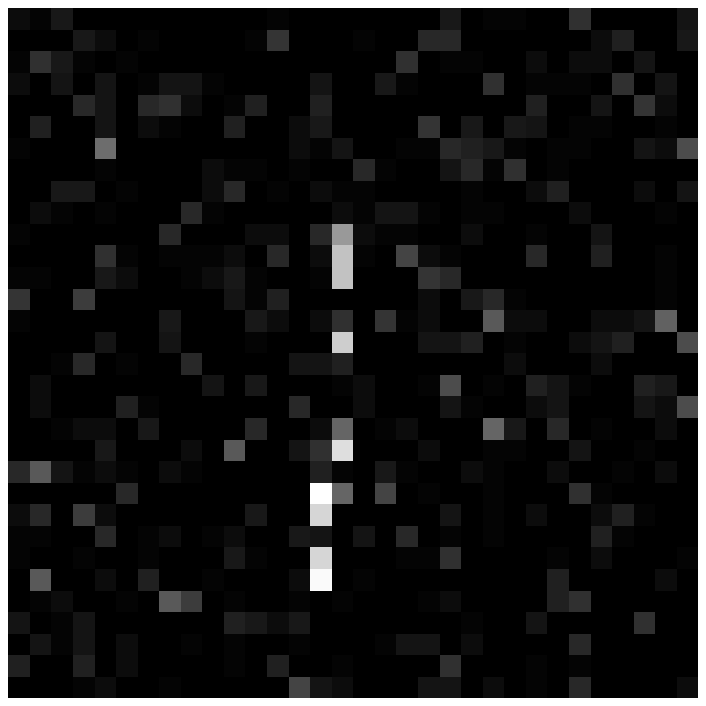}
\vspace*{-0.3cm}}
&
\parbox[c]{1cm}{\vspace*{0.1cm}\includegraphics[width=1.0cm]{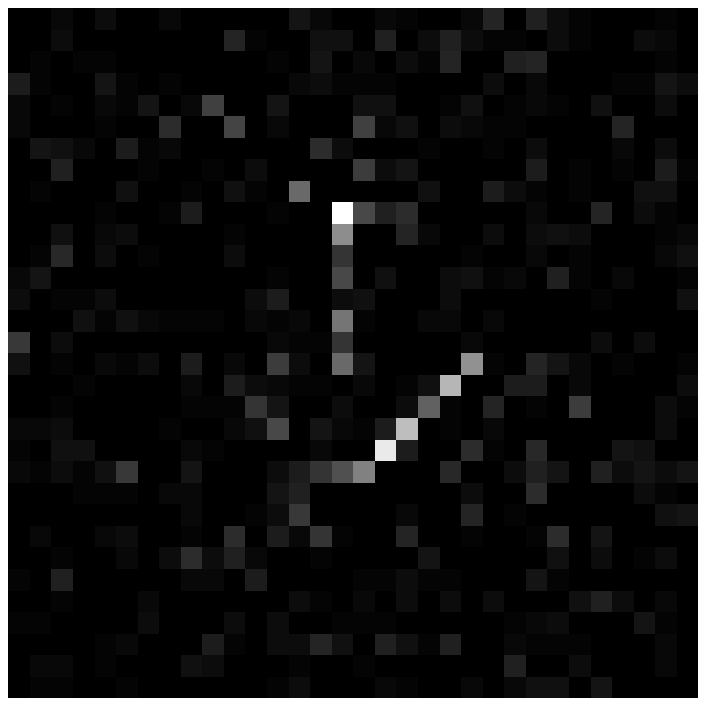}
\vspace*{-0.3cm}}
&
\parbox[c]{1cm}{\vspace*{0.1cm}\includegraphics[width=1.0cm]{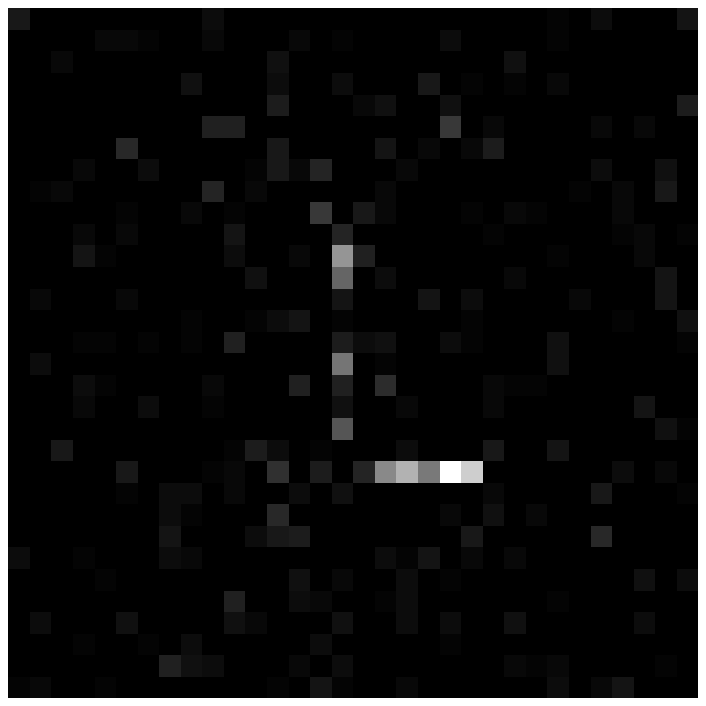}
\vspace*{-0.3cm}}
&
&
\parbox[c]{1cm}{\vspace*{0.1cm}\includegraphics[width=1.0cm]{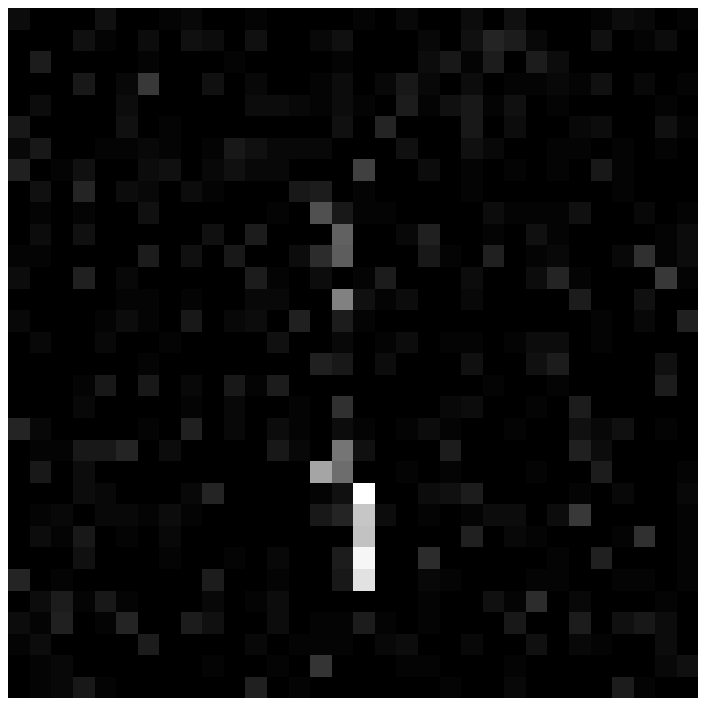}
\vspace*{-0.3cm}}
 \\
\hline
d) &
&
&
\parbox[c]{1cm}{\vspace*{0.1cm}\includegraphics[width=1.0cm]{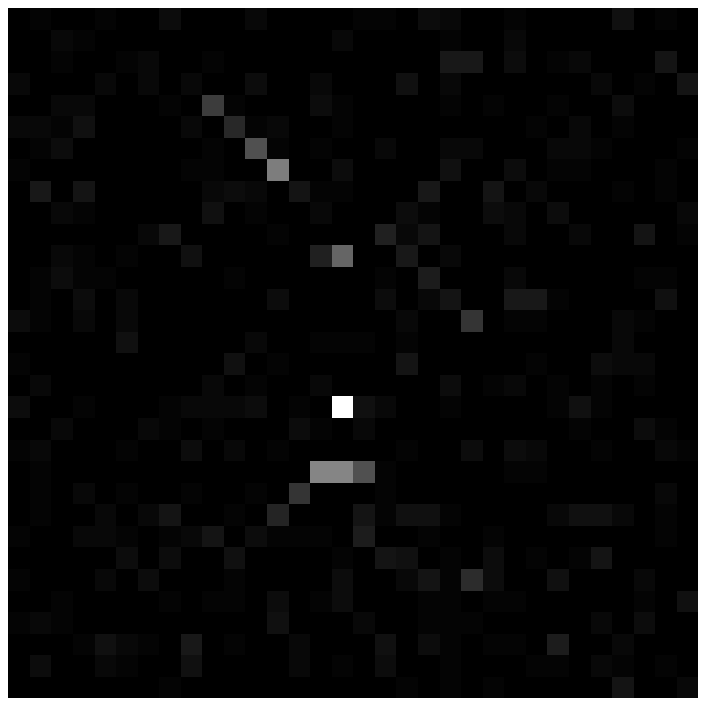}
\vspace*{-0.3cm}}
&
\parbox[c]{1cm}{\vspace*{0.1cm}\includegraphics[width=1.0cm]{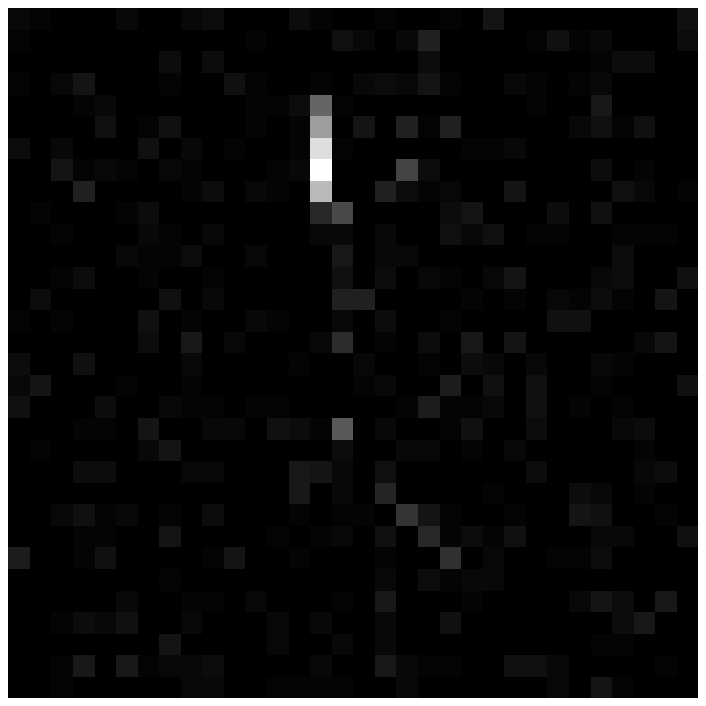}
\vspace*{-0.3cm}}
&
&
\parbox[c]{1cm}{\vspace*{0.1cm}\includegraphics[width=1.0cm]{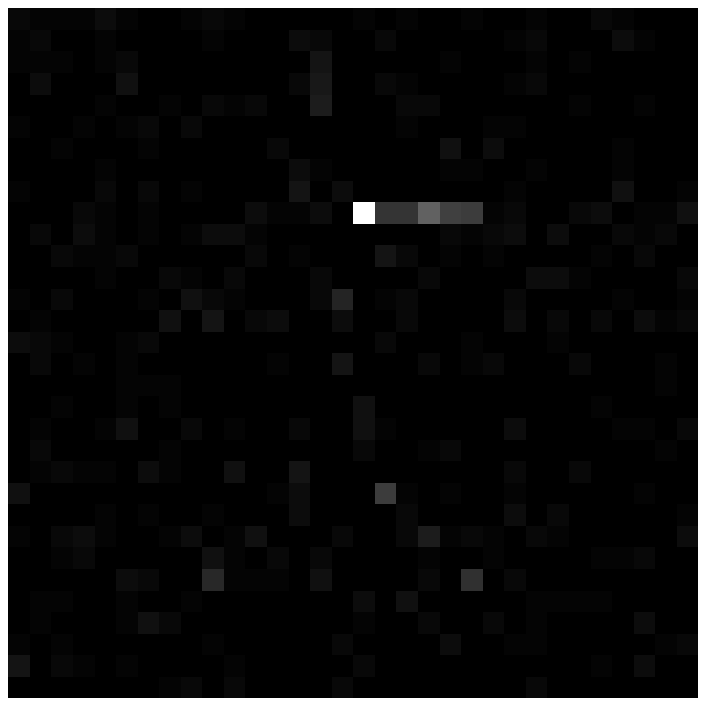}
\vspace*{-0.3cm}}
&
&
\parbox[c]{1cm}{\vspace*{0.1cm}\includegraphics[width=1.0cm]{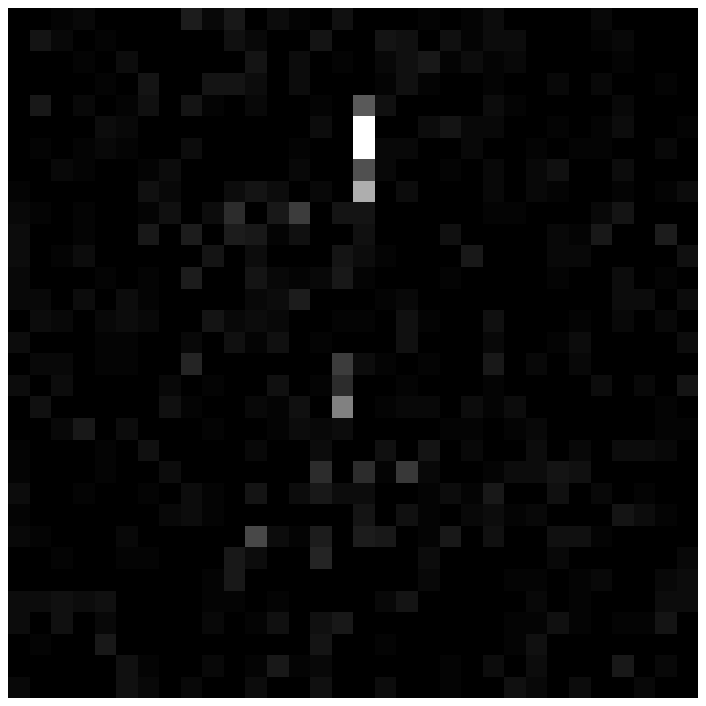}
\vspace*{-0.3cm}}
&
&
\parbox[c]{1cm}{\vspace*{0.1cm}\includegraphics[width=1.0cm]{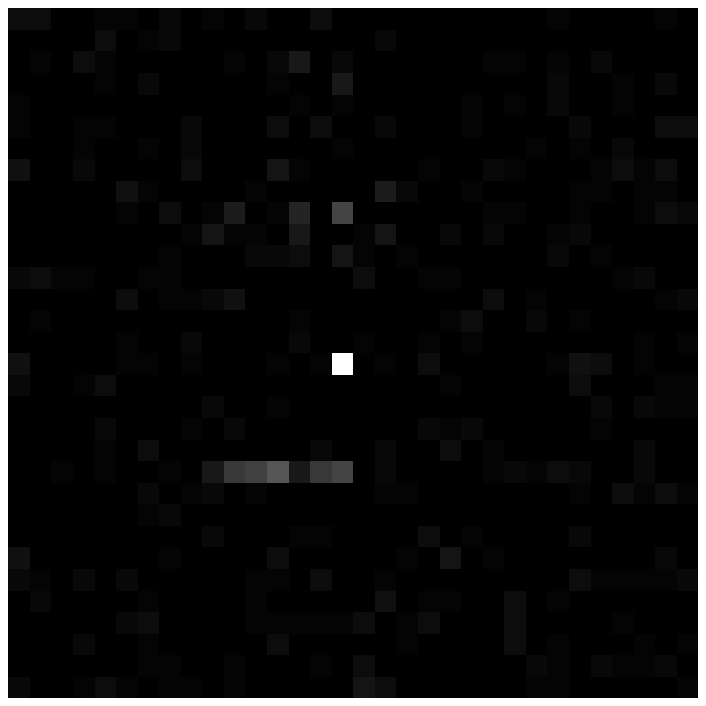}
\vspace*{-0.3cm}}
&
\parbox[c]{1cm}{\vspace*{0.1cm}\includegraphics[width=1.0cm]{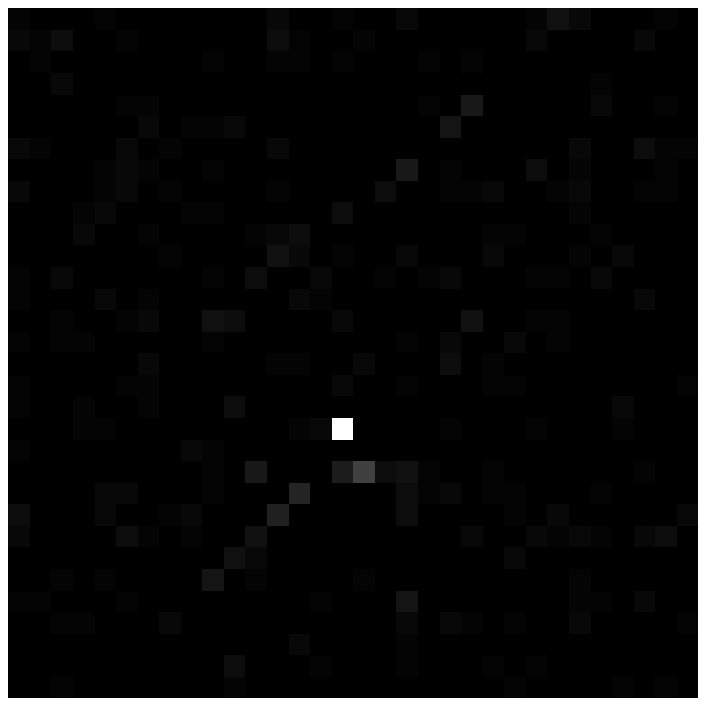}
\vspace*{-0.3cm}}
&
\parbox[c]{1cm}{\vspace*{0.1cm}\includegraphics[width=1.0cm]{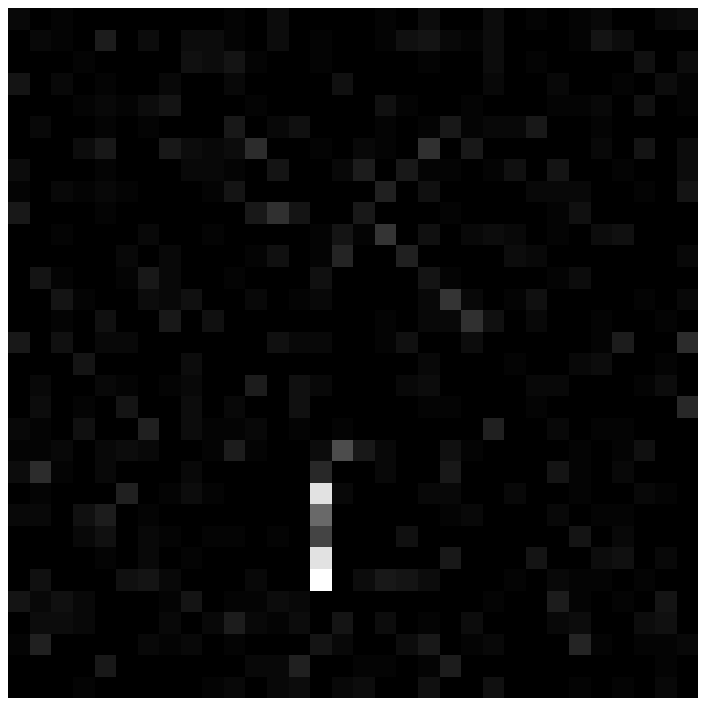}
\vspace*{-0.3cm}}
&
\parbox[c]{1cm}{\vspace*{0.1cm}\includegraphics[width=1.0cm]{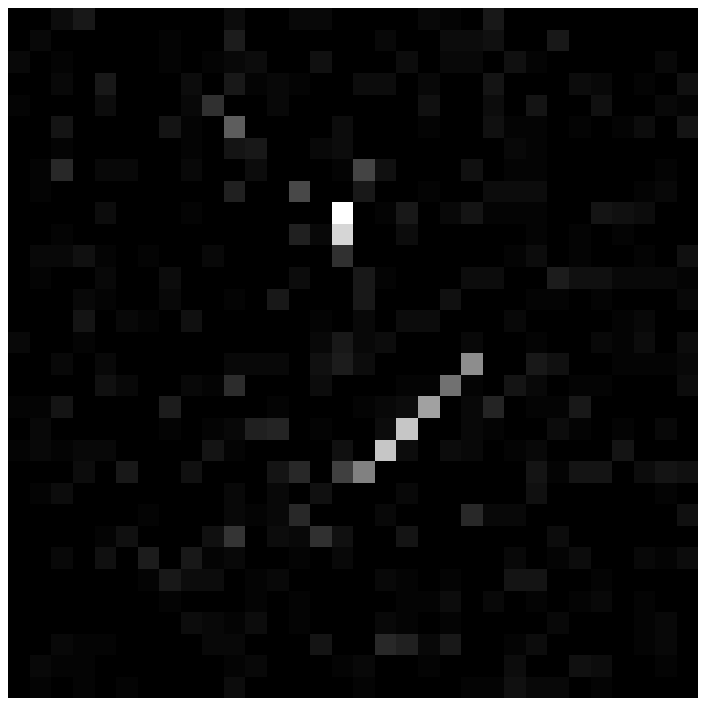}
\vspace*{-0.3cm}}
&
\parbox[c]{1cm}{\vspace*{0.1cm}\includegraphics[width=1.0cm]{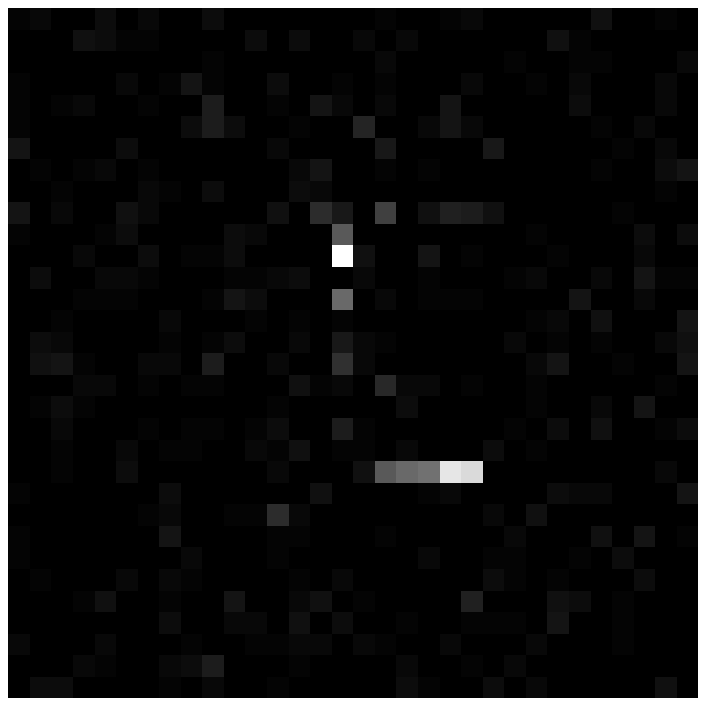}
\vspace*{-0.3cm}}
&
&
\parbox[c]{1cm}{\vspace*{0.1cm}\includegraphics[width=1.0cm]{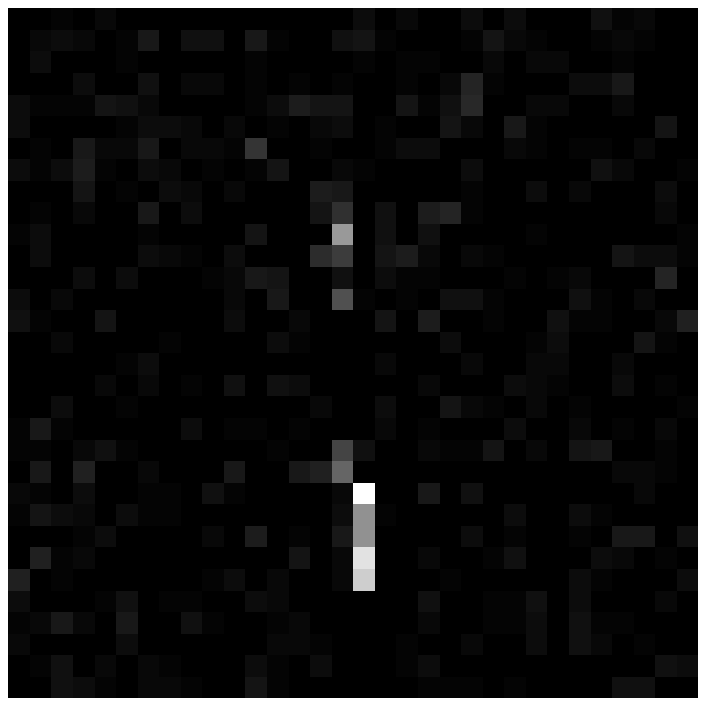}
\vspace*{-0.3cm}}
\\
\hline

e) &
&
&
&
&
&
&
&
&
&
\parbox[c]{1cm}{\vspace*{0.1cm}\includegraphics[width=1.0cm]{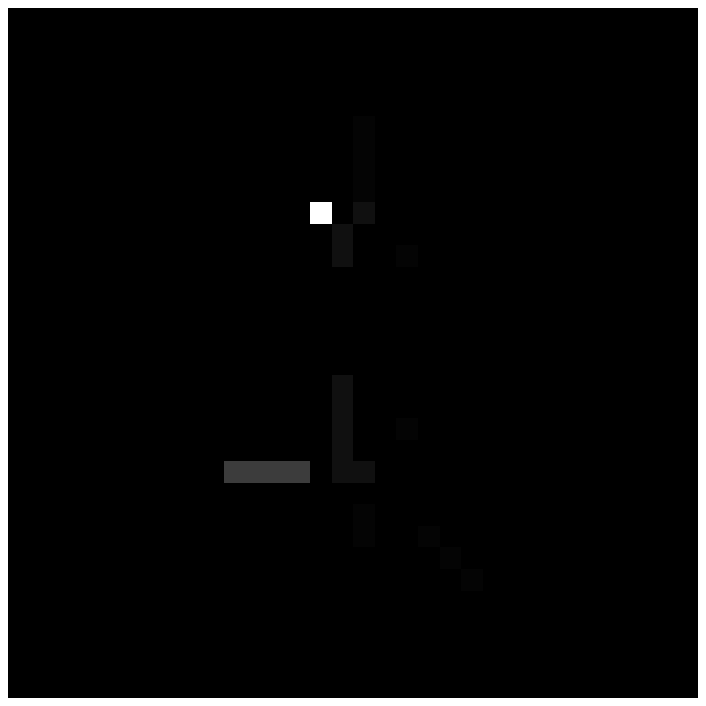}
\vspace*{-0.3cm}}
&
\parbox[c]{1cm}{\vspace*{0.1cm}\includegraphics[width=1.0cm]{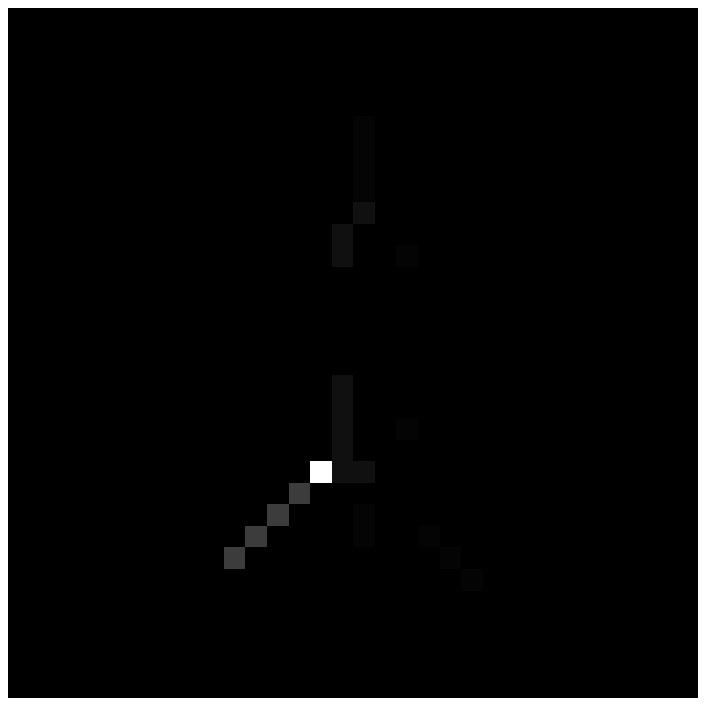}
\vspace*{-0.3cm}}
&
&
\parbox[c]{1cm}{\vspace*{0.1cm}\includegraphics[width=1.0cm]{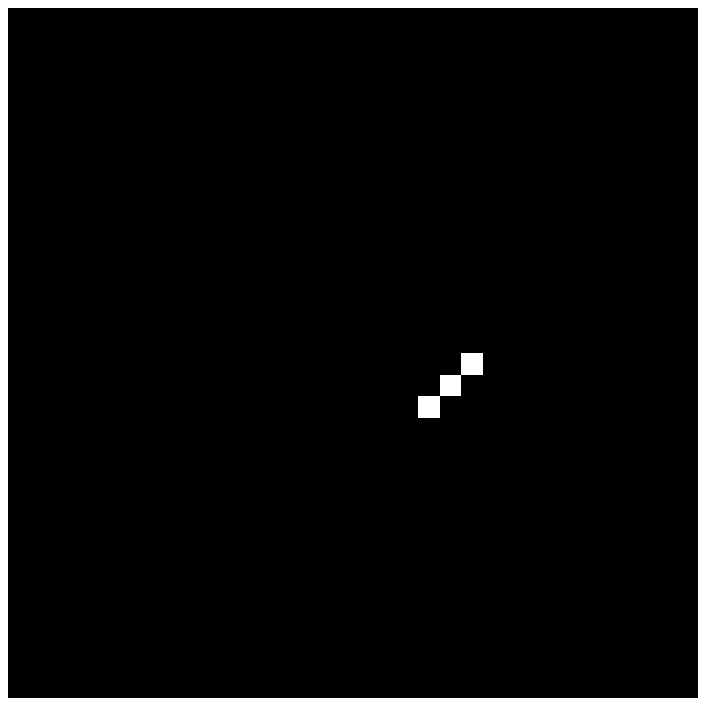}
\vspace*{-0.3cm}}
&
\parbox[c]{1cm}{\vspace*{0.1cm}\includegraphics[width=1.0cm]{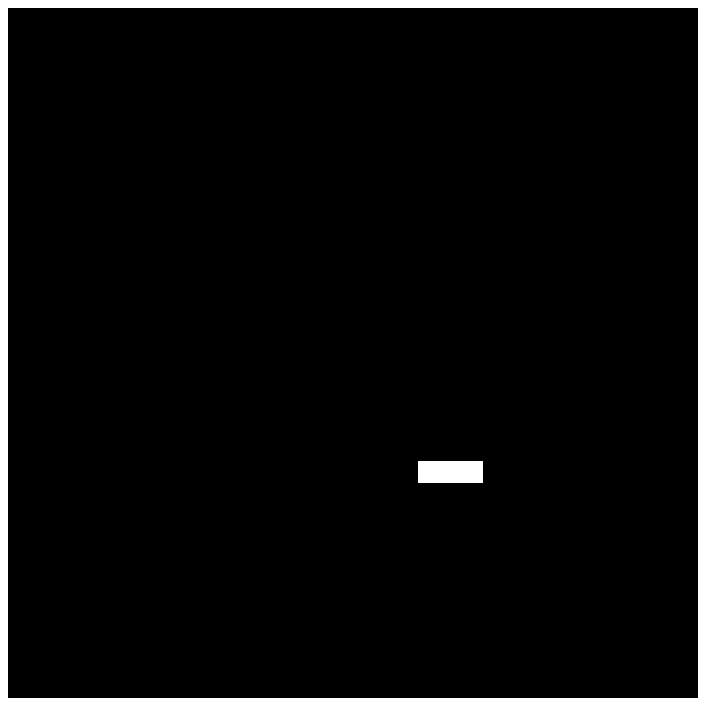}
\vspace*{-0.3cm}}
&
&
\\
\hline
\end{tabular}
\caption{Topics estimated for noisy swimmer dataset by a) proposed RP, b) proposed DDP, c) Gibbs in \cite{Griffiths:ref}, d)
  NMF in \cite{betaDivergence:ref} and e) on clean dataset by RecL2 in \cite{Arora2:ref} closest to the 16 ideal (ground truth) topics. Gibbs misses $5$ and NMF misses $6$ of the ground truth topics while RP DDP recovers all $16$ and our topic estimates look less noisy. RecL2 hits $4$ on clean dataset.}
\label{fig:noisy_swimmer}
\end{figure*}

%% file: table1.tex
%
\begin{table}[!htb]
\centering

\caption{Examples of extracted topics for \textit{NIPS} dataset by proposed Random projection method (RP), Data-dependent projection (DDP), algorithm in \cite{Griffiths:ref}(Gibbs), the practical algorithm in \cite{Arora2:ref}(RecL2). 
}

\begin{tabular}{|>{\small}m{0.1\linewidth}|>{\footnotesize}m{0.8\linewidth}|}
\hline
 RP &chip circuit noise analog current voltage gates\\
DDP	&chip	circuit	analog	voltage	pulse	vlsi	device\\
Gibbs &analog	circuit	chip	output	figure	current	 vlsi\\
RecL2	& N/A \\													
\hline
RP &visual	cells	spatial	ocular	cortical	cortex	dominance	orientation\\
DDP &visual	cells	model	cortex	orientation	cortical 	eye\\
Gibbs &cells	cortex	visual	activity	orientation	cortical	receptive\\
RecL2&orientation	knowledge	model	cells	visual	good	mit\\
\hline
RP	&learning	training	error	vector	parameters	svm	data\\
DDP	&learning	error	training	weight	network	function	neural\\
Gibbs &training	error	set	generalization	examples	test	learning\\
RecL2 &training	error	set	data	function	test	weighted\\
\hline
RP	&speech	training	recognition	performance	hmm	mlp	input\\
DDP	&training	speech	recognition	network	word	classifiers	hmm\\
Gibbs &speech 	recognition	word	training	hmm	speaker	mlp	acoustic\\
RecL2 &speech	recognition	network	neural	positions	training	learned	\\
\hline
\end{tabular}
\label{topicword:table}
\end{table}

%% file: table2.tex
\begin{table}[!htb]
\centering
\caption{Examples of estimated topics on {\it NY} Times using RP and RecL2 algorithms}
\begin{tabular}{|>{\small}m{0.1\linewidth}|>{\footnotesize}m{0.8\linewidth}|}
\hline
RP &	weather	wind	air	storm	rain	cold	\\
RecL2	& N/A \\													
\hline
RP &	feeling	sense	love	character	heart	emotion	\\
RecL2	& N/A \\
\hline
RP	& election	zzz\textunderscore florida	ballot	vote	zzz\textunderscore al\textunderscore gore	recount	\\
RecL2 &	ballot	election	court	votes	vote	zzz\textunderscore al\textunderscore gore\\
\hline
RP	& yard	game	team	season	play	zzz\textunderscore nfl		\\
RecL2 &	yard game play season team touchdown	\\
\hline
RP & N/A \\
RecL2 & zzz\textunderscore kobe\textunderscore bryant	zzz\textunderscore super\textunderscore bowl	police	shot	family	election	\\
\hline
\end{tabular}
\label{topicword:table2}
\end{table}		

%% file: table3.tex
%
Tables 4, 5, 6, and 7 show the most frequent words in topics extracted by various algorithms on \textit{NIPS} dataset. The words are listed in the descending order. There are $M=1,700$ documents. Average words per document is $N\approx 900$. Vocabulary size is $W=2,500$.

It is difficult and confusing to group four sets of topics. We simply show topics extracted by each algorithm individually. 

\begin{table*}[!htb]
\centering

\caption{Examples of extracted topics on \textit{NIPS} by(Gibbs) }

\begin{tabular}{|>{\small}m{0.1\linewidth}|>{\footnotesize}m{0.8\linewidth}|}
\hline
Gibbs &analog	circuit	chip	output	figure	current	 vlsi\\
\hline
Gibbs &cells	cortex	visual	activity	orientation	cortical	receptive\\
\hline
Gibbs &training	error	set	generalization	examples	test	learning\\
\hline
Gibbs &speech 	recognition	word	training	hmm	speaker	mlp	acoustic\\
\hline
Gibbs & function	theorem	bound	threshold	number	proof	dimension	\\
\hline
Gibbs &model	modeling	observed	neural	parameter	proposed	similar\\
\hline
Gibbs & node	tree	graph	path	number	decision	structure\\
\hline
Gibbs &features	set	figure	based	extraction	resolution	line\\
\hline
Gibbs & prediction	regression	linear	training	nonlinear	input	experts\\
\hline
Gibbs &performance	problem	number	results	search	time	table\\
\hline
Gibbs & motion	direction	eye	visual	position	velocity	head\\
\hline
Gibbs & function	basis	approximation	rbf	kernel	linear	radial	gaussian\\
\hline
Gibbs & network	neural	output	recurrent	net	architecture	feedforward\\
\hline
Gibbs & local	energy	problem	points	global	region	optimization\\
\hline
Gibbs &units	inputs	hidden	layer	network	weights	training\\
\hline
Gibbs &representation	connectionist	activation	distributed	processing	language	sequence\\
\hline
Gibbs & time	frequency	phase	temporal	delay	sound	amplitude\\
\hline
Gibbs & learning	rule	based	task	examples	weight	knowledge\\
\hline
Gibbs & state	time	sequence	transition	markov	finite	dynamic\\
\hline
Gibbs & algorithm	function	convergence	learning	loss	step	gradient\\
\hline
Gibbs & image	object	recognition	visual	face	pixel	vision\\
\hline
Gibbs & neurons	synaptic	firing	spike	potential	rate	activity\\
\hline
Gibbs & memory	patterns	capacity	associative	number	stored	storage\\
\hline
Gibbs & classification	classifier	training	set	decision	data	pattern\\
\hline
Gibbs & level	matching	match	block	instance	hierarchical	part\\
\hline
Gibbs & control	motor	trajectory	feedback	system	controller	robot\\
\hline
Gibbs & information	code	entropy	vector	bits	probability	encoding\\
\hline
Gibbs & system	parallel	elements	processing	computer	approach	implementation\\
\hline
Gibbs & target	task	performance	human	response	subjects	attention\\
\hline
Gibbs & signal	filter	noise	source	independent	channel	filters	processing\\
\hline
Gibbs & recognition	task	architecture	network	character	module	neural\\
\hline
Gibbs & data	set	method	clustering	selection	number	methods\\
\hline
Gibbs & space	distance	vectors	map	dimensional	points	transformation\\
\hline
Gibbs & likelihood	gaussian	parameters	mixture	bayesian	data	prior\\
\hline
Gibbs & weight	error	gradient	learning	propagation	term	back\\
\hline
Gibbs & order	structure	natural	scale	properties	similarity	analysis\\
\hline
Gibbs & distribution	probability	variance	sample	random	estimate\\
\hline
Gibbs & dynamics	equations	point	fixed	case	limit	function\\
\hline
Gibbs & matrix	linear	vector	eq	solution	problem	nonlinear\\
\hline
Gibbs & learning	action	reinforcement	policy	state	optimal	actions	control	function	goal	environment\\
\hline
\end{tabular}
\end{table*}

\begin{table*}[!htb]
\centering
\caption{Examples of extracted topics on \textit{NIPS} by DDP(Data Dependent Projections) }
\begin{tabular}{|>{\small}m{0.1\linewidth}|>{\footnotesize}m{0.8\linewidth}|}			
\hline
DDP& 
loss	function	minima	smoothing	plasticity	logistic	site\\
\hline
DDP& spike	neurons	firing	time	neuron	amplitude	modulation \\
\hline
DDP& clustering	data	teacher	learning	level	hidden	model	error	\\
\hline
DDP& distance	principal	image	loop	flow	tangent	matrix	vectors	\\
\hline
DDP& network	experts	user	set	model	importance	data	\\
\hline
DDP& separation	independent	sources	signals	predictor	mixing	component\\
\hline
DDP& concept	learning	examples	tracking	hypothesis	incremental	greedy \\
\hline
DDP& learning	error	training	weight	network	function	neural	\\
\hline
DDP& visual	cells	model	cortex	orientation	cortical	response\\
\hline
DDP& population	tuning	sparse	codes	implicit	encoding	cybern\\
\hline
DDP& attention	selective	mass	coarse	gradients	switching	occurred\\
\hline
DDP& temperature	annealing	graph	matching	assignment	relaxation	correspondence\\
\hline
DDP& role	representation	connectionist	working	symbolic	distributed	expressions	\\
\hline
DDP& auditory	frequency	sound	time	signal	spectral	spectrum	filter\\
\hline
DDP& language	state	string	recurrent	noise	giles	order	\\
\hline
DDP& family	symbol	coded	parameterized	labelled	discovery\\
\hline
DDP&memory	input	capacity	patterns	number	associative	layer\\
\hline
DDP&model	data	models	distribution	algorithm	probability	gaussian\\
\hline
DDP&risk	return	optimal	history	learning	costs	benchmark	\\
\hline
DDP& kernel	data	weighting	estimators	divergence	case	linear\\
\hline
DDP& channel	information	noise	membrane	input	mutual	signal\\
\hline
DDP& image	surface	filters	function	scene	neural regions\\
\hline
DDP& delays	window	receiving	time	delay	adjusting	network\\
\hline
DDP& training	speech	recognition	network	word	neural		hmm\\
\hline
DDP& information	code	entropy	vector	bits	probability	encoding\\
\hline
DDP& figure	learning	model	set	training	segment	labeled\\
\hline
DDP& tree	set	neighbor	trees	number	decision	split	\\
\hline
DDP& control	motor	model	trajectory	controller	learning	arm	\\
\hline
DDP& chip	circuit	analog	voltage	current	pulse	vlsi\\
\hline
DDP& recognition	object	rotation	digit	image	letters	translation	\\
\hline
DDP& processor	parallel	list	dependencies	serial	target	displays\\
\hline
DDP& network	ensemble	training	networks	monte-carlo	input	neural	\\
\hline
DDP& block	building	terminal	experiment	construction	basic	oriented \\
\hline
DDP& input	vector	lateral	competitive	algorithm	vectors		topology\\
\hline
DDP& direction	velocity	cells	head	system	model	place	behavior \\
\hline
DDP& recursive	structured	formal	regime	analytic	realization	rigorous\\
\hline
DDP& similarity	subjects	structural	dot	psychological	structure	product \\
\hline
DDP& character	words	recognition	system	characters	text	neural	\\
\hline
DDP& learning	state	time	action	reinforcement	policy	robot	path\\
\hline
DDP& function	bounds	threshold	set	algorithm	networks dept	polynomial \\
\hline

\end{tabular}
\end{table*}

\begin{table*}[!htb]
\centering
\caption{Examples of extracted topics on \textit{NIPS} by RP (Random Projections) }
\begin{tabular}{|>{\small}m{0.1\linewidth}|>{\footnotesize}m{0.8\linewidth}|}			
\hline
RP& 
data	learning	set	pitch	space	exemplars	note	music\\
\hline
RP& images	object	face	image	recognition	model	objects	network\\
\hline
RP& synaptic	neurons	network	input	spike	time	cortical	timing\\
\hline
RP& hand	video	wavelet	recognition	system	sensor	gesture	time\\
\hline
RP& neural	function	networks	functions	set	data	network	number\\
\hline
RP& template	network	input	contributions	neural	component	output	transient\\
\hline
RP& learning	state	model	function	system	cart	failure	time\\
\hline
RP& cell	membrane	cells	potential	light	response	ganglion	retina\\
\hline
RP& tree	model	data	models	algorithm	leaves	learning	node\\
\hline
RP& state	network	learning	grammar	game	networks	training	finite\\
\hline
RP& visual	cells	spatial	ocular	cortical	model	dominance	orientation\\
\hline
RP& input	neuron	conductance	conductances	current	firing	synaptic	rate\\
\hline
RP& set	error	algorithm	learning	training	margin	functions	function\\
\hline
RP& items	item	signature	handwriting	verification	proximity	signatures	recognition\\
\hline
RP& separation	ica	time	eeg	blind	independent	data	components\\
\hline
RP& control	model	network	system	feedback	neural	learning	controller\\
\hline
RP& cells	cell	firing	model	cue	cues	layer	neurons\\
\hline
RP& stress	human	bengio	chain	region	syllable	profile	song\\
\hline
RP& genetic	fibers	learning	population	implicit	model	algorithms	algorithm\\
\hline
RP& chip	circuit	noise	analog	current	voltage	time	input\\
\hline
RP& hidden	input	data	states	units	training	set	error\\
\hline
RP& network	delay	phase	time	routing	load	neural	networks\\
\hline
RP& query	examples	learning	data	algorithm	dependencies	queries	loss\\
\hline
RP& sound	auditory	localization	sounds	owl	optic	knudsen	barn\\
\hline
RP& head	eye	direction	cells	position	velocity	model	rat\\
\hline
RP& learning	tangent	distance	time	call	batch	rate	data\\
\hline
RP& binding	role	representation	tree	product	structure	structures	completion\\
\hline
RP& learning	training	error	vector	parameters	svm	teacher	data\\
\hline
RP& problem	function	algorithm	data	penalty	constraints	model	graph\\
\hline
RP& speech	training	recognition	performance	hmm	mlp	input	network\\
\hline
RP& learning	schedule	time	execution	instruction	scheduling	counter	schedules\\
\hline
RP& boltzmann	learning	variables	state	variational	approximation	algorithm	function\\
\hline
RP& state	learning	policy	action	states	optimal	time	actions\\
\hline
RP& decoding	frequency	output	figure	set	message	languages	spin\\
\hline
RP& network	input	figure	image	contour	texture	road	task\\
\hline
RP& receptor	structure	disparity	image	function	network	learning	vector\\
\hline
RP& visual	model	color	image	surround	response	center	orientation\\
\hline
RP& pruning	weights	weight	obs	error	network	obd	elimination\\
\hline
RP& module	units	damage	semantic	sharing	network	clause	phrase\\
\hline
RP& character	characters	recognition	processor	system	processors	neural	words\\
\hline				
\end{tabular}
\end{table*}

\begin{table*}[!htb]
\centering
\caption{Examples of extracted topics on \textit{NIPS} by RecL2 }
\begin{tabular}{|>{\small}m{0.1\linewidth}|>{\footnotesize}m{0.8\linewidth}|}			
\hline
RecL2& 
network	networks	supported	rbf	function	neural	data	training\\
\hline
RecL2& asymptotic	distance	tangent	algorithm	vectors	set	vector	learning\\
\hline
RecL2& learning	state	negative	policy	algorithm	time	function	complex\\
\hline
RecL2& speech	recognition	speaker	network	positions	training	performance	networks\\
\hline
RecL2& cells	head	operation	direction	model	cell	system	neural\\
\hline
RecL2& object	model	active	recognition	image	views	trajectory	strings\\
\hline
RecL2& spike	conditions	time	neurons	neuron	model	type	input\\
\hline
RecL2& network	input	neural	recognition	training	output	layer	networks\\
\hline
RecL2& maximum	motion	direction	visual	figure	finally	order	time\\
\hline
RecL2& learning	training	error	input	generalization	output	studies	teacher\\
\hline
RecL2& fact	properties	neural	output	neuron	input	current	system\\
\hline
RecL2& sensitive	chain	length	model	respect	cell	distribution	class\\
\hline
RecL2& easily	face	images	image	recognition	set	based	examples\\
\hline
RecL2& model	time	system	sound	proportional	figure	dynamical	frequency\\
\hline
RecL2& lower	training	free	classifiers	classification	error	class	performance\\
\hline
RecL2& network	networks	units	input	training	neural	output	unit\\
\hline
RecL2& figure	image	contour	partially	images	point	points	local\\
\hline
RecL2& control	network	learning	neural	system	model	time	processes\\
\hline
RecL2& learning	algorithm	time	rate	error	density	gradient	figure\\
\hline
RecL2& state	model	distribution	probability	models	variables	versus	gaussian\\
\hline
RecL2& input	network	output	estimation	figure	winner	units	unit\\
\hline
RecL2& learning	model	data	training	models	figure	set	neural\\
\hline
RecL2& function	algorithm	loss	internal	learning	vector	functions	linear\\
\hline
RecL2& system	model	state	stable	speech	models	recognition	hmm\\
\hline
RecL2& image	algorithm	images	system	color	black	feature	problem\\
\hline
RecL2& orientation	knowledge	model	cells	visual	good	cell	mit\\
\hline
RecL2& network	memory	neural	networks	neurons	input	time	state\\
\hline
RecL2& neural	weight	network	networks	learning	neuron	gradient	weights\\
\hline
RecL2& data	model	set	algorithm	learning	neural	models	input\\
\hline
RecL2& training	error	set	data	function	test	generalization	optimal\\
\hline
RecL2& model	learning	power	deviation	control	arm	detection	circuit\\
\hline
RecL2& tree	expected	data	node	algorithm	set	varying	nodes\\
\hline
RecL2& data	kernel	model	final	function	space	linear	set\\
\hline
RecL2& target	visual	set	task	tion	cost	feature	figure\\
\hline
RecL2& model	posterior	map	visual	figure	cells	activity	neurons\\
\hline
RecL2& function	neural	networks	functions	network	threshold	number	input\\
\hline
RecL2& neural	time	pulse	estimation	scene	figure	contrast	neuron\\
\hline
RecL2& network	networks	training	neural	set	error	period	ensemble\\
\hline
RecL2& information	data	distribution	mutual	yield	probability	input	backpropagation\\
\hline
RecL2& units	hidden	unit	learning	network	layer	input	weights\\
\hline
\end{tabular}
\end{table*}

%% file: table4.tex
%
Tables 8 to 11  show the most frequent words in topics extracts by algorithms on \textit{NY Times} dataset. There are $M=300,000$ documents. Average words per document is $N\approx 300$. Vocabulary size is $W=15,000$. 

\begin{table*}[!htb]
\centering
\caption{Extracted topics on \textit{NY} Times by (RP) }
\begin{tabular}{|>{\small}m{0.1\linewidth}|>{\footnotesize}m{0.8\linewidth}|}			
\hline
RP& 
com	daily	question	beach	palm	statesman	american
\\
\hline
RP&building	house	center	home	space	floor	room
\\
\hline
RP&
cup	minutes	add	tablespoon	oil	food	pepper
\\
\hline
RP&
article	fax	information	com	syndicate	contact	separate
\\
\hline
RP&
history	american	flag	war	zzz\textunderscore america	country	zzz\textunderscore american
\\
\hline
RP&
room	restaurant	hotel	tour	trip	night	dinner
\\
\hline
RP&
meeting	official	agreement	talk	deal	plan	negotiation
\\
\hline
RP&
plane	pilot	flight	crash	jet	accident	crew
\\
\hline
RP&
fire	attack	dead	victim	zzz\textunderscore world\textunderscore trade\textunderscore center	died	firefighter
\\
\hline
RP&
team	game	zzz\textunderscore laker	season	player	play	zzz\textunderscore nba
\\
\hline
RP&
food	dog	animal	bird	drink	eat	cat
\\
\hline
RP&
job	office	chief	manager	executive	president	director
\\
\hline
RP&
family	father	son	home	wife	mother	daughter
\\
\hline
RP&
point	half	lead	shot	left	minutes	quarter
\\
\hline
RP&
game	team	season	coach	player	play	games
\\
\hline
RP&
military	ship	zzz\textunderscore army	mission	officer	boat	games
\\
\hline
RP&
need	help	important	problem	goal	process	approach
\\
\hline
RP&
scientist	human	science	research	researcher	zzz\textunderscore university	called
\\
\hline
RP&
computer	system	zzz\textunderscore microsoft	software	window	program	technology
\\
\hline
RP&
zzz\textunderscore china	zzz\textunderscore russia	chinese	zzz\textunderscore russian	russian	zzz\textunderscore united\textunderscore states	official
\\
\hline
RP&
body	hand	head	leg	face	arm	pound
\\
\hline
RP&
money	big	buy	worth	pay	business	find
\\
\hline
RP&
weather	water	wind	air	storm	rain	cold
\\
\hline
RP&
million	money	fund	contribution	dollar	raising	campaign
\\
\hline
RP&
police	officer	gun	crime	shooting	shot	violence
\\
\hline
RP&
night	told	asked	room	morning	thought	knew
\\
\hline
RP&
school	student	teacher	program	education	college	high
\\
\hline
RP&
palestinian	zzz\textunderscore israel	zzz\textunderscore israeli	peace	israeli	zzz\textunderscore yasser\textunderscore arafat	israelis
\\
\hline
RP&
race	won	track	racing	run	car	driver
\\
\hline
RP&
case	investigation	charges	prosecutor	lawyer	trial	evidence
\\
\hline
RP&
percent	market	stock	economy	quarter	growth	economic
\\
\hline
RP&
team	sport	player	games	fan	zzz\textunderscore olympic	gold
\\
\hline
RP&
company	zzz\textunderscore enron	companies	stock	firm	million	billion
\\
\hline
RP&
percent	number	million	according	rate	average	survey
\\
\hline
RP&
zzz\textunderscore american	zzz\textunderscore america	culture	today	century	history	social
\\
\hline
RP&
book	author	writer	writing	published	read	reader
\\
\hline
RP&
bill	zzz\textunderscore senate	zzz\textunderscore congress	zzz\textunderscore house	legislation	lawmaker	vote
\\
\hline
RP&
anthrax	disease	zzz\textunderscore aid	virus	official	mail	cases
\\
\hline
RP&
election	zzz\textunderscore florida	ballot	vote	votes	voter	zzz\textunderscore al\textunderscore gore
\\
\hline
RP&
look	fashion	wear	shirt	hair	designer	clothes
\\
\hline
RP&
lawyer	lawsuit	claim	case	suit	legal	law
\\
\hline
RP&
study	found	risk	level	studies	effect	expert
\\
\hline
RP&
light	look	image	images	eye	sound	camera
\\
\hline
RP&
cell	research	human	stem	scientist	organ	body
\\
\hline
RP&
found	century	river	ago	rock	ancient	village
\\
\hline
RP&
fight	ring	fighting	round	right	won	title
\\
\hline
RP&
energy	power	oil	gas	plant	prices	zzz\textunderscore california
\\
\hline
RP&
care	problem	help	brain	need	mental	pain
\\
\hline
RP&
word	letter	question	mail	read	wrote	paper
\\
\hline
RP&
play	show	stage	theater	musical	production	zzz\textunderscore broadway
\\
\hline
RP&
show	television	network	series	zzz\textunderscore nbc	broadcast	viewer
\\
\hline
RP&
run	hit	game	inning	yankees	home	games
\\
\hline
\end{tabular}
\end{table*}

\begin{table*}[!htb]
\centering
\caption{Extracted topics on \textit{NY} Times by (RP, continued) }
\begin{tabular}{|>{\small}m{0.1\linewidth}|>{\footnotesize}m{0.8\linewidth}|}			
\hline
RP&
religious	zzz\textunderscore god	church	jewish	faith	religion	jew
\\
\hline
RP&
zzz\textunderscore new\textunderscore york	zzz\textunderscore san\textunderscore francisco	gay	zzz\textunderscore manhattan	zzz\textunderscore new\textunderscore york\textunderscore city	zzz\textunderscore los\textunderscore angeles	zzz\textunderscore chicago
\\
\hline
RP&
season	zzz\textunderscore dodger	agent	player	manager	team	contract
\\
\hline
RP&
attack	terrorist	terrorism	official	bin	laden	zzz\textunderscore united\textunderscore states
\\
\hline
RP&
reporter	media	newspaper	public	interview	press	mayor
\\
\hline
RP&
black	zzz\textunderscore texas	white	hispanic	zzz\textunderscore georgia	racial	american
\\
\hline
RP&
zzz\textunderscore bush	administration	president	zzz\textunderscore white\textunderscore house	policy	zzz\textunderscore washington	zzz\textunderscore dick\textunderscore cheney
\\
\hline
RP&
hour	road	car	driver	truck	bus	train
\\
\hline
RP&
drug	patient	doctor	medical	cancer	hospital	treatment
\\
\hline
RP&
president	zzz\textunderscore clinton	zzz\textunderscore bill\textunderscore clinton	zzz\textunderscore white\textunderscore house	office	presidential	zzz\textunderscore washington
\\
\hline
RP&
company	product	sales	market	customer	business	consumer
\\
\hline
RP&
problem	fear	protest	situation	action	threat	crisis
\\
\hline
RP&
airport	flight	security	passenger	travel	airline	airlines
\\
\hline
RP&
water	plant	fish	trees	flower	tree	garden
\\
\hline
RP&
com	web	site	www	mail	online	sites
\\
\hline
RP&
goal	game	play	team	king	games	season
\\
\hline
RP&
death	prison	penalty	case	trial	murder	execution
\\
\hline
RP&
government	political	leader	power	election	country	party
\\
\hline
RP&
tax	cut	plan	billion	cost	taxes	program
\\
\hline
RP&
zzz\textunderscore george\textunderscore bush	campaign	zzz\textunderscore al\textunderscore gore	republican	democratic	voter	political
\\
\hline
RP&
weapon	nuclear	defense	zzz\textunderscore india	missile	zzz\textunderscore united\textunderscore states	system
\\
\hline
RP&
zzz\textunderscore internet	companies	company	internet	technology	access	network
\\
\hline
RP&
zzz\textunderscore taliban	zzz\textunderscore afghanistan	zzz\textunderscore pakistan	forces	war	afghan	military
\\
\hline
RP&
official	agency	information	rules	government	agencies	problem
\\
\hline
RP&
question	fact	point	view	reason	term	matter
\\
\hline
RP&
wanted	friend	knew	thought	worked	took	told
\\
\hline
RP&
film	movie	character	actor	movies	director	zzz\textunderscore hollywood
\\
\hline
RP&
remain	early	past	despite	ago	irish	failed
\\
\hline
RP&
art	artist	collection	show	painting	museum	century
\\
\hline
RP&
worker	job	employees	union	company	labor	companies
\\
\hline
RP&
land	local	area	resident	town	project	areas
\\
\hline
RP&
feel	sense	moment	love	feeling	character	heart
\\
\hline
RP&
zzz\textunderscore united\textunderscore states	zzz\textunderscore u\textunderscore s	zzz\textunderscore mexico	countries	country	zzz\textunderscore japan	trade
\\
\hline
RP&
yard	game	team	season	play	quarterback	zzz\textunderscore nfl
\\
\hline
RP&
special	gift	holiday	zzz\textunderscore christmas	give	home	giving
\\
\hline
RP&
tour	round	shot	zzz\textunderscore tiger\textunderscore wood	golf	course	player
\\
\hline
RP&
car	seat	vehicle	model	vehicles	wheel	zzz\textunderscore ford
\\
\hline
RP&
war	zzz\textunderscore iraq	zzz\textunderscore united\textunderscore states	military	international	zzz\textunderscore iran	zzz\textunderscore u\textunderscore s
\\
\hline
RP&
group	member	program	organization	director	board	support
\\
\hline
RP&
set	won	match	final	win	point	lost
\\
\hline
RP&
court	law	decision	right	case	federal	ruling
\\
\hline
RP&
feel	right	need	look	hard	kind	today
\\
\hline
RP&
pay	card	money	credit	account	bank	loan
\\
\hline
RP&
music	song	band	album	record	pop	rock
\\
\hline
RP&
priest	zzz\textunderscore boston	abuse	sexual	church	bishop	zzz\textunderscore massachusett
\\
\hline
RP&
women	children	child	girl	parent	young	woman

\\
\hline
RP&
guy	bad	tell	look	talk	ask	right
\\
\hline
RP&
european	french	zzz\textunderscore europe	german	zzz\textunderscore france	zzz\textunderscore germany	zzz\textunderscore united\textunderscore states
\\
\hline
\end{tabular}
\end{table*}

\begin{table*}[!htb]
\centering
\caption{Extracted topics on \textit{NY} Times by RecL2 }
\begin{tabular}{|>{\small}m{0.1\linewidth}|>{\footnotesize}m{0.8\linewidth}|}			
\hline
RecL2
& 
charges	zzz\textunderscore al\textunderscore gore	taking	open	party	million	full
\\
\hline
RecL2
&
file	filmed	season	embarrassed	attack	need	young
\\
\hline
RecL2
&
human	music	sexual	sold	required	launched	articulo
\\
\hline
RecL2
&
pass	financial	por	named	music	handle	task
\\
\hline
RecL2
&
zzz\textunderscore n\textunderscore y	zzz\textunderscore south	zzz\textunderscore mariner	convicted	book	big	zzz\textunderscore washington
\\ \hline
RecL2
&
zzz\textunderscore u\textunderscore s	ages	worker	zzz\textunderscore kansas	expected	season	sugar
\\ \hline
RecL2
&
team	official	group	panelist	night	cool	limited
\\ \hline
RecL2
&
corp	business	program	financial	left	corrected	professor
\\ \hline
RecL2
&
zzz\textunderscore london	commercial	zzz\textunderscore laker	services	took	beach	american
\\ \hline
RecL2
&
home	percent	screen	question	today	zzz\textunderscore federal	kind
\\ \hline
RecL2
&
important	mass	emerging	spokesman	threat	program	television
\\ \hline
RecL2
&
reported	zzz\textunderscore israel	lost	received	benefit	separate	zzz\textunderscore internet
\\ \hline
RecL2
&
article	night	mixture	independence	misstated	need	line
\\ \hline
RecL2
&
pay	home	join	book	zzz\textunderscore bush	zzz\textunderscore bill\textunderscore parcell	kind
\\ \hline
RecL2
&
boy	zzz\textunderscore mike\textunderscore tyson	property	helicopter	championship	limit	unfortunately
\\ \hline
RecL2
&
question	public	stock	yard	zzz\textunderscore calif	zzz\textunderscore jeff\textunderscore gordon	dropped
\\ \hline
RecL2
&
zzz\textunderscore red\textunderscore sox	matter	student	question	zzz\textunderscore pete\textunderscore sampras	home	game
run	called	zzz\textunderscore napster	places	season	need	tell
\\ \hline
RecL2
&
defense	player	job	version	zzz\textunderscore giant	movie	company
\\ \hline
RecL2
&
game	official	right	com	season	school	show
\\ \hline
RecL2
&
million	support	room	try	zzz\textunderscore new\textunderscore york	club	air
\\ \hline
RecL2
&
zzz\textunderscore arthur\textunderscore andersen	word	occurred	accounting	percent	zzz\textunderscore rudolph\textunderscore giuliani	dog
\\ \hline
RecL2
&
plan	zzz\textunderscore bush	zzz\textunderscore anaheim\textunderscore angel	learn	site	rate	room
\\ \hline
RecL2
&
place	zzz\textunderscore phoenix	program	gay	player	open	point
\\ \hline
RecL2
&
student	zzz\textunderscore republican	zzz\textunderscore tiger\textunderscore wood	birth	falling	homes	birthday
\\ \hline
RecL2
&
question	meeting	standard	home	zzz\textunderscore lance\textunderscore armstrong	ring	lead
\\ \hline
RecL2
&
order	point	called	analyst	player	children	zzz\textunderscore washington
\\ \hline
RecL2
&
father	zzz\textunderscore bill\textunderscore clinton	network	public	return	job	wrote
\\ \hline
RecL2
&
police	zzz\textunderscore clipper	worker	policies	home	screen	zzz\textunderscore white\textunderscore house
\\ \hline
RecL2
&
home	zzz\textunderscore georgia	zzz\textunderscore bush	security	zzz\textunderscore white\textunderscore house	zzz\textunderscore philadelphia	understanding
\\ \hline
RecL2
&
zzz\textunderscore bill\textunderscore bradley	case	prison	pretty	found	zzz\textunderscore state\textunderscore department	zzz\textunderscore internet
\\ \hline
RecL2
&
zzz\textunderscore democrat	zzz\textunderscore elian	turn	raised	leader	problem	show
\\ \hline
RecL2
&
named	music	una	pass	financial	sold	task
\\ \hline
RecL2
&
cost	company	companies	zzz\textunderscore america	show	left	official
\\ \hline
RecL2
&
plan	election	room	site	zzz\textunderscore bush	learn	list
\\ \hline
RecL2
&
percent	zzz\textunderscore l\textunderscore a	leader	zzz\textunderscore john\textunderscore ashcroft	general	lost	doctor
\\ \hline
RecL2
&
home	worker	zzz\textunderscore fbi	zzz\textunderscore louisiana	zzz\textunderscore patrick\textunderscore ewing	police	zzz\textunderscore bush
\\ \hline
RecL2
&
chairman	red	deal	case	public	www	electronic
\\ \hline
RecL2
&
kind	book	home	security	member	zzz\textunderscore troy\textunderscore aikman	zzz\textunderscore bush
\\ \hline
RecL2
&
estate	spend	beach	season	home	zzz\textunderscore black	nurse
\\ \hline
RecL2
&
test	theme	career	important	site	company	official
\\ \hline
RecL2
&
los	music	required	sold	task	human	topic
\\ \hline
RecL2
&
taking	open	zzz\textunderscore al\textunderscore gore	party	full	telephone	team
\\ \hline
RecL2
&
percent	word	zzz\textunderscore ray\textunderscore lewis	kind	home	stake	involved
\\ \hline
RecL2
&
point	called	analyst	zzz\textunderscore english	zzz\textunderscore washington	zzz\textunderscore england	project
\\ \hline
RecL2
&
lead	zzz\textunderscore u\textunderscore s	business	giant	quickly	game	zzz\textunderscore taliban
\\ \hline
RecL2
&
zzz\textunderscore bush	plan	zzz\textunderscore brazil	learn	rate	zzz\textunderscore latin\textunderscore america	fighting
\\ \hline
RecL2
&
mind	zzz\textunderscore united\textunderscore states	bill	hour	looking	land	zzz\textunderscore jerusalem
\\ \hline
RecL2
&
team	vision	right	official	wines	government	com
\\ \hline
RecL2
&
zzz\textunderscore america	airport	night	place	leader	lost	start
\\ \hline
RecL2
&
zzz\textunderscore los\textunderscore angeles	right	sales	journalist	level	question	combat
\\ \hline
RecL2
&
home	zzz\textunderscore maverick	police	worker	shot	screen	half
\\ \hline
RecL2
&
bill	zzz\textunderscore taiwan	country	moment	administration	staff	found
\\ \hline
RecL2
&
living	technology	company	changed	night	debate	school
\\ \hline
\hline
\end{tabular}
\end{table*}

\begin{table*}[!htb]
\centering
\caption{Extracted topics on \textit{NY} Times by RecL2, continued.}
\begin{tabular}{|>{\small}m{0.1\linewidth}|>{\footnotesize}m{0.8\linewidth}|}			
\hline
RecL2
&
zzz\textunderscore john\textunderscore mccain	case	prison	pretty	recent	separate	zzz\textunderscore clinton
\\ \hline
RecL2
&
plan	zzz\textunderscore bush	home	rate	zzz\textunderscore john\textunderscore rocker	election	half
\\ \hline
RecL2
&
zzz\textunderscore kobe\textunderscore bryant	zzz\textunderscore super\textunderscore bowl	police	shot	family	election	basketball
\\ \hline
RecL2
&
pay	kind	book	home	half	zzz\textunderscore drew\textunderscore bledsoe	safe
\\ \hline
RecL2
&
anthrax	bad	official	makes	product	zzz\textunderscore dodger	million
\\ \hline
RecL2
&
right	result	group	team	need	official	game
\\ \hline
RecL2
&
called	order	group	zzz\textunderscore washington	left	big	point
\\ \hline
RecL2
&
percent	problem	word	zzz\textunderscore timothy\textunderscore mcveigh	season	company	person
\\ \hline
RecL2
&
public	bill	zzz\textunderscore pri	include	player	point	case
\\ \hline
RecL2
&
zzz\textunderscore microsoft	son	money	season	attack	zzz\textunderscore olympic	zzz\textunderscore mexico
\\ \hline
RecL2
&
plan	zzz\textunderscore bush	room	learn	list	battle	zzz\textunderscore mike\textunderscore piazza
\\ \hline
RecL2
&
group	point	called	court	left	children	school
\\ \hline
RecL2
&
zzz\textunderscore united\textunderscore states	problem	public	land	looking	watched	school
\\ \hline
RecL2
&
home	zzz\textunderscore fbi	police	half	zzz\textunderscore jason\textunderscore kidd	percent	worker
\\ \hline
RecL2
&
question	public	company	zzz\textunderscore dale\textunderscore earnhardt	job	yard	dropped
\\ \hline
RecL2
&
zzz\textunderscore texas	big	zzz\textunderscore george\textunderscore bush	season	court	market	left
\\ \hline
RecL2
&
game	final	right	won	law	saying	finally
\\ \hline
RecL2
&
show	home	percent	official	office	shark	game
\\ \hline
RecL2
&
case	zzz\textunderscore kennedy	zzz\textunderscore jeb\textunderscore bush	electronic	red	www	show
\\ \hline
RecL2
&
official	bad	player	games	money	season	need
\\ \hline
RecL2
&
case	zzz\textunderscore bradley	zzz\textunderscore state\textunderscore department	prison	found	general	pretty
\\ \hline
RecL2
&
percent	returning	problem	leader	word	companies	serve
\\ \hline
RecL2
&
official	player	place	zzz\textunderscore new\textunderscore york	left	show	visit
\\ \hline
RecL2
&
country	zzz\textunderscore russia	start	public	hour	lost	called
\\ \hline
RecL2
&
zzz\textunderscore pakistan	newspaper	group	game	company	official	head
\\ \hline
RecL2
&
kind	pay	percent	safe	earned	zone	talking
\\ \hline
RecL2
&
beginning	game	right	com	season	won	games
\\ \hline
RecL2
&
zzz\textunderscore governor\textunderscore bush	case	percent	zzz\textunderscore clinton	found	zzz\textunderscore internet	zzz\textunderscore heisman
\\ \hline
RecL2
&
zzz\textunderscore manhattan	game	zzz\textunderscore laura\textunderscore bush	school	company	zzz\textunderscore clinton	right
\\ \hline
RecL2
&
big	zzz\textunderscore at	called	order	zzz\textunderscore boston	left	point
\\ \hline
RecL2
&
zzz\textunderscore america	zzz\textunderscore delta	company	court	airline	play	left
\\ \hline
RecL2
&
kind	pages	zzz\textunderscore trojan	reflect	percent	home	police
\\ \hline
RecL2
&
zzz\textunderscore house	zzz\textunderscore slobodan\textunderscore milosevic	problem	public	crisis	feet	word
\\ \hline
RecL2
&
left	securities	big	zzz\textunderscore south	book	zzz\textunderscore washington	received
\\ \hline
RecL2
&
part	percent	pardon	companies	administration	zzz\textunderscore clinton	number
\\ \hline
RecL2
&
zzz\textunderscore congress	left	company	play	business	zzz\textunderscore nashville	zzz\textunderscore michael\textunderscore bloomberg
\\ \hline
RecL2
&
zzz\textunderscore mccain	case	prison	lost	zzz\textunderscore clinton	zzz\textunderscore israel	administration
\\ \hline
RecL2
&
zzz\textunderscore san\textunderscore francisco	hour	problem	recent	job	information	reason
\\ \hline
RecL2
&
game	right	com	final	won	season	school
\\ \hline
RecL2
&
company	zzz\textunderscore cia	night	zzz\textunderscore washington	american	companies	zzz\textunderscore new\textunderscore york
\\ \hline
RecL2
&
point	left	lost	play	country	money	billion
\\ \hline
RecL2
&
father	wrote	mind	return	job	research	zzz\textunderscore palestinian
\\ \hline
RecL2
&
caught	bishop	general	seen	abuse	right	prior
\\ \hline
RecL2
&
kind	zzz\textunderscore white\textunderscore house	home	security	help	question	zzz\textunderscore new\textunderscore york
\\ \hline
RecL2
&
closer	threat	important	closely	official	local	cloning
\\ \hline
RecL2
&
zzz\textunderscore enron	place	league	remain	point	big	performance \\
\hline

\end{tabular}
\end{table*}